\documentclass[a4paper,fleqn]{cas-dc}
\usepackage[numbers]{natbib}

\def\tsc#1{\csdef{#1}{\textsc{\lowercase{#1}}\xspace}}
\tsc{WGM}
\tsc{QE}
\usepackage{stmaryrd}
\usepackage{longtable}
\usepackage{microtype}
\usepackage[english]{babel}
\usepackage{floatrow}
\usepackage{subcaption}
\usepackage{url}
\usepackage{colortbl}
\usepackage{amsmath}
\usepackage{mathrsfs}
\usepackage{amssymb}
\usepackage{amsthm}
\newtheorem{definition}{Definition}[section]
\newtheorem{theorem}{Theorem}[section]
\newtheorem{prop}{Proposition}[section]
\newtheorem{remark}{Remark}[section]

\newtheorem{corollary}{Corollary}[section]

\usepackage[left]{lineno}
\usepackage[textsize=scriptsize,backgroundcolor=yellow!40,obeyDraft]{todonotes}
\usepackage[nonumberlist,acronym,sanitize=none]{glossaries}
\glsdisablehyper
\usepackage{comment}
\usepackage[capitalise,nameinlink]{cleveref}
\Crefname{algocf}{Algorithm}{Algorithms}
\Crefname{lemma}{Lemma}{Lemmata}
\crefname{definition}{Def.}{Def.s}
\crefname{example}{Ex.}{Ex.s}
\Crefname{prop}{Proposition}{Propositions}
\crefname{prop}{Prop.}{Prop.s}
\crefalias{AlgoLine}{line}
\usepackage{tkz-base}
\usetikzlibrary{decorations.pathmorphing,trees,snakes,arrows,shapes,automata,arrows.meta}
\usepackage{paralist}
\usepackage{multicol}
\usepackage{booktabs}
\usepackage{enumitem}
\usepackage{multirow}
\usepackage{rotating}
\usepackage{etaremune}
\usepackage[ruled,linesnumbered,algo2e,vlined]{algorithm2e}
\usepackage{mathtools}
\usepackage{ifthen}
\usepackage[normalem]{ulem}
\usepackage{lineno}
\usepackage{soul}
\usepackage{floatrow}
\usepackage{lipsum} 
\usepackage{microtype}
\usepackage{xfrac}
\usepackage{tabto}

\usepackage[binary-units]{siunitx}

\usepackage{etoolbox}
\robustify\bfseries
\usepackage{wrapfig}
\usepackage{nicefrac}
\usepackage{footmisc}
\usepackage{xspace}

\usepackage{adjustbox}

\usepackage{graphicx}

\usepackage{makecell}

\usepackage[usestackEOL]{stackengine}

\renewcommand{\arraystretch}{1.5}

\newcolumntype{d}{>{\columncolor{gray!10}}c}
\newcolumntype{m}{>{\columncolor{gray!10}}l}
\newfloatcommand{capbtabbox}{table}[][\FBwidth]

\newenvironment{iiilist}{\begin{inparaenum}[\itshape(i)\upshape]}{\end{inparaenum}}

\def\BeginSmEq{\par\vspace{-0.7\baselineskip}\small}
\def\EndSmEq{\normalsize} 
\graphicspath{{figures/}}
\usepackage{amsmath}

\newacronym[\glslongpluralkey={Business Processes}]{bp}{BP}{Business Process}
\newacronym{wf}{WF}{workflow}
\newacronym{bpi}{BPI}{Business Process Intelligence}
\newacronym{bpm}{BPM}{Business Process Management}
\newacronym{bpms}{BPMS}{Business Process Management System}
\newacronym{bpmn}{BPMN}{Business Process Model and Notation}
\newacronym{soa}{SOA}{Service-Oriented Architecture}
\newacronym{kpi}{KPI}{Key Performance Indicator}
\newacronym{wfms}{WfMS}{Workflow Management System}
\newacronym{pn}{PN}{Petri net}
\newacronym{cpn}{CPN}{colored Petri net}
\newacronym{xes}{XES}{eXtensible Event Stream}
\newacronym{po}{PO}{Partial Order}
\newacronym{tl}{TL}{Temporal Logic}  \newacronym{ltl}{LTL}{Linear Temporal Logic}
\newacronym{fol}{FOL}{First Order Logic}
\def\ltlf {\ensuremath{\mathrm{LTL}_f}}
\newacronym{ltlf}{\ltlf}{Linear Temporal Logic on Finite Traces}
\newacronym{mso}{MSO}{Monadic Second Order Logic}
\newacronym{rex}{RE}{regular expression}
\def\Autom {\ensuremath{A}}
\newacronym[symbol=\Autom,longplural={finite state automata}]{fsa}{FSA}{finite state automaton}
\newacronym[symbol=\Autom,longplural={deterministic finite automata}]{dfa}{DFA}{deterministic finite automaton}
\newacronym[symbol=\Autom,longplural={deterministic finite state automata}]{dfs}{DFS}{deterministic finite state automaton}
\newacronym[symbol=\Autom,longplural={nondeterministic finite automata}]{nfa}{NFA}{nondeterministic finite automaton}
\def\Declare {\textsc{Declare}}
\newglossaryentry{declare}{name={\Declare},description={a declarative process modelling language and notation}}
\newglossaryentry{task}{name={task},description={the non-divisible, elementary activity}}

\def\lettera {\ensuremath{a}}

\newcommand{\taskize}[1] {\ensuremath{\scalebox{0.85}{\normalfont\textsf{#1}}}}
\def\taska {\taskize{a}}
\def\taskb {\taskize{b}}
\def\taskc {\taskize{c}}
\def\taskd {\taskize{d}}
\def\taske {\taskize{e}}
\def\taskf {\taskize{f}}

\def\tasky {\taskize{y}}

\newglossaryentry{promod}{name={process specification},description={the specification of a process}
}
\def\DeclaModel {\ensuremath{\mathcal{S}}}
\newglossaryentry{declamodel}{name={declarative \glsentrytext{promod}},description={\glsentrydesc{promod}, expressed by means of constraints},
	symbol={\DeclaModel}
}

\newglossaryentry{mindeclamodel}{name={discovered \glsentrytext{declamodel}},description={\glsentrydesc{declamodel}, discovered from an \glsentrytext{evtlog}},
	symbol={\DeclaModel}
}
\newglossaryentry{minerful}{name={MINERful},description={the declarative process discovery algorithm \glsentrytext{minerful}}}
\newacronym{mf}{Mf}{\gls{minerful}}
\newglossaryentry{minerfulVac}{name={MINERful Vacuity Checker},description={\glsentrytext{minerful}} algorithm with semantical vacuity detection}
\newacronym{mfv}{Mf-Vchk}{\gls{minerfulVac}}
\newacronym{dmm}{DMM}{Declare Maps Miner}
\newglossaryentry{decmapmin}{name={Declare Maps Miner},description={the declarative process discovery algorithm \glsentrytext{decmapmin}}}
\newacronym{dmm2}{DM2}{Declare Miner 2}
\newglossaryentry{decmapmin2}{name={Declare Miner 2},description={improvement of \glsentrytext{decmapmin} algorithm}}
\newglossaryentry{janus}{name={Janus},description={the declarative process discovery algorithm \glsentrytext{janus}}}
\def\LogAlph {\ensuremath{\Sigma}}

\newglossaryentry{logalph}{
	name={log alphabet},description={the process alphabet, as reflected in a log},symbol={\LogAlph}}
\def\Evt {\ensuremath{e}}
\newglossaryentry{evt}{
	name={event},description={a record of an instantaneous fact during the process enactment},symbol={\Evt}}
\def\EvtTrace {\ensuremath{t}}

\newglossaryentry{evttrace}{
	name={trace},description={a sequence of \glsplural{evt}},symbol={\EvtTrace}}
\def\EvtLog {\ensuremath{L}}

\newglossaryentry{evtlog}{
	name={event log},description={a collection of \glstext{evttrace}s},symbol={\EvtLog}}
\def\Subsum {\ensuremath{\sqsubseteq}}

\newglossaryentry{subsum}{name={subsumption},description={is subsumed by},symbol={\Subsum}}

\newglossaryentry{relaxop}{name={relaxation},description={relaxation operator, climbing the \glsentrytext{subsum} hierarchy}}

\newglossaryentry{acti}{name={activation},description={the activation of a constraint}}
\newglossaryentry{target}{name={target},description={target}}
\def\Cns {\ensuremath{C}}
\newglossaryentry{con}{name={constraint},description={a temporal business rule},
	symbol={\Cns}
}

\newglossaryentry{welldef}{name={well-defined},description={of \glsentrytext{con}s for which a finite non-empty trace exists that complies with them}
}
\newglossaryentry{cnspar}{name={parameter},description={a parameter of a \glsentrytext{con}},
}
\newglossaryentry{cnsarity}{name={arity},description={number of parameters of a \glsentrytext{con}},
}
\newglossaryentry{exi}{
	name={existence},
	description={constrains single activities}
}
\newglossaryentry{exicon}{
	name={\glsentrytext{exi} \glsentrytext{con}},
	description={constrains single activities}
}
\newglossaryentry{posicon}{
	name={position \glsentrytext{con}},
	description={constrains the position of activities}
}
\newglossaryentry{cardicon}{
	name={cardinality \glsentrytext{con}},
	description={limits the number of activities}
}
\newglossaryentry{rela}{
	name={relation},
	description={constraint on pairs of activities}
}
\newglossaryentry{relacon}{
	name={\glsentrytext{rela} \glsentrytext{con}},
	description={constraint on pairs of activities}
}
\newglossaryentry{unirelacon}{
	name={unidirectional \glsentrytext{relacon}},
	description={constraint on pairs of activities, out of which one is the activation, as the other is the target}
}
\newglossaryentry{unifwrelacon}{
	name={\glsentrytext{fw}-\glsentrytext{unirelacon}},
	description={constraint on pairs of activities, having the first parameter as the activation, and the second one as the target}
}
\def\FwCns {\ensuremath{\mathit{fw}}}
\newglossaryentry{fw}{
	name={forward},
	description={forward constraint},
	symbol={\FwCns}
}

\newglossaryentry{unibwrelacon}{
	name={\glsentrytext{bw}-\glsentrytext{unirelacon}},
	description={constraint on pairs of activities, having the second parameter as the activation, and the first one as the target}
}
\def\BwCns {\ensuremath{\mathit{bw}}}
\newglossaryentry{bw}{
	name={backward},
	description={backward constraint},
	symbol={\BwCns}
}

\newglossaryentry{corelacon}{
	name={coupling \glsentrytext{con}},
	description={constraint based on pairs of relation constraints}
}
\newglossaryentry{nega}{
	name={negative},
	description={of a constraint, that negates a coupling relation constraint}
}
\newglossaryentry{negacon}{
	name={\glsentrytext{nega} \glsentrytext{con}},
	description={constraint negating a coupling relation constraint}
}
\def\CnsTemp {\ensuremath{\mathcal{C}}}
\newglossaryentry{cnstemp}{name={template},description={the template of a \glsentrydesc{con}},
	symbol={\CnsTemp}}
\def\CnsTempPrime {\ensuremath{\CnsTemp'}}
\def\CnsTempSecond {\ensuremath{\CnsTemp''}}
\newcommand{\CnsTempFunc}[2] {\ensuremath{\CnsTemp(#1\ifthenelse{\equal{#2}{}}{}{,#2})}}
\newcommand{\CnsTempFuncPrime}[2] {\ensuremath{\CnsTempPrime(#1\ifthenelse{\equal{#2}{}}{}{,#2})}}
\newcommand{\CnsTempFuncSecond}[2] {\ensuremath{\CnsTempSecond(#1\ifthenelse{\equal{#2}{}}{}{,#2})}}

\newglossaryentry{cnstype}{name={type},description={the type of a \glsentrydesc{cnstemp}}}
\def\CnsTempRep {\ensuremath{\mathfrak{C}}}
\newglossaryentry{cnsrep}{name={repertoire},description={the repertoire of \glsentrytext{declare} \glsentrytext{temp}s},
	symbol={\CnsTempRep}}

\newglossaryentry{cnsuniv}{name={\glsentrytext{con}s universe},description={the set of \glsentrytext{declare} \glsentrytext{temp}s over the process alphabet reflected in the log}}
\def\CnsInstRelation {\ensuremath{\Gamma}}

\newglossaryentry{cnsinst}{name={\glsentrytext{cnstemp} instantiation relation},description={the assignment relation instantiating \glsentrytext{cnstemp}s into \glsentrytext{con}s, namely assigning \glsentrytext{task}s to \glsentrytext{cnspar}s.},
	symbol={\CnsInstRelation}}

\newcommand{\CnsInterpFun} {\ensuremath{\mathscr{I}}}
\newglossaryentry{cnsinterp}{
	name={interpretation function},description={function interpreting a \glsentrytext{declamodel}},
	symbol={\CnsInterpFun}}

\def\RelaConTemp {\ensuremath{\mathcal{R}}}
\newglossaryentry{relacontemp}{name={relation template},description={the template of a relation \glsentrydesc{con}},
	symbol={\RelaConTemp}}

\def\ExiConTemp {\ensuremath{\mathcal{E}}}
\newglossaryentry{exicontemp}{name={existence template},description={the template of an existence \glsentrydesc{con}},
		symbol={\ExiConTemp}}

\def\Supp {\ensuremath{\sigma}}

\newglossaryentry{support}{name={support},description={the support of a \glsentrydesc{con}},
	symbol={\Supp}}

\def\Conf {\ensuremath{\kappa}}
\newglossaryentry{conf}{name={confidence},description={the confidence level of a \glsentrydesc{con}},
	symbol={\Conf}}

\def\IntF {\ensuremath{\iota}}
\newglossaryentry{intf}{name={interest factor},description={the interest factor of a \glsentrydesc{con}},
	symbol={\IntF}}

\def\EvaluationFunctor {\ensuremath{\eta}}
\newglossaryentry{evaluation}{
	name={evaluation},description={evaluation of a \glsentrytext{con} or a \glsentrytext{declamodel} over a \glsentrytext{evttrace} or an \glsentrytext{evtlog}},
	symbol={\EvaluationFunctor}}

\def\RespTxt {Response}

\def\PrecTxt {Precedence}

\newcommand{\Resp}[2] {\ensuremath{\textsc{\RespTxt}(#1,#2)}}

\newcommand{\Preced}[2] {\ensuremath{{\textsc{\PrecTxt}}(#1,#2)}}

\def\MultiSetFunctor {\ensuremath{\mathbb{M}}}
\newglossaryentry{multiset}{
	name={multi-set},description={a collection possibly containing multiple units of the same element},
	symbol={\MultiSetFunctor}}

\def\PowerSetFunctor {\ensuremath{\mathbb{P}}}
\newglossaryentry{powerset}{
	name={power-set},description={the collection of sets generated by all combinations without repetition of elements in a set},
	symbol={\PowerSetFunctor}}

\def\LTLuntil {\ensuremath{\;\mathop{\mathrm{\mathbf{U}}}\;}}

\def\AutomInitState {\ensuremath{s_0}}

\newglossaryentry{fsainit}{name={initial state},description={initial state of the automaton},
	symbol=\AutomInitState}

\newacronym{rnf}{RNF}{Reactive Normal Form}
\def\RNFnext {\ensuremath{\bigcirc}}
\def\RNFprev {\ensuremath{\ominus}}
\def\RNFfut {\ensuremath{\Diamond}}

\newcommand*{\BarDiamond}{\tikz[
  inner sep=0pt,
  shorten >=.5ex,
  shorten <=.5ex,
  line cap=round,
  baseline=(c.base),
  ]\draw
  (0,0) node (c) {$\Diamond$}
  ($(c.west)+(0,0)$) -- ($(c.east)+(0,0)$);}
\def\RNFpast {\ensuremath{\BarDiamond}}
\def\RNFglobFut {\ensuremath{\Box}}
\def\RNFglobPast {\ensuremath{\boxminus}}
\def\RNFuntil {\LTLuntil}
\def\RNFsince {\ensuremath{\;\mathop{\mathrm{\mathbf{S}}}\;}}

\def\RNFstart {\ensuremath{\EvtTrace_{\textrm{Start}}}}
\def\RNFend {\ensuremath{\EvtTrace_{\textrm{End}}}}
\def\RNFtrue {\ensuremath{\textsl{True}}}
\def\RNFfalse {\ensuremath{\textsl{False}}}

\newacronym{rf}{RCon}{Reactive Constraint}
\newacronym{sepautset}{sep.aut.set}{separated automata set}
\def\RfImplication {\ensuremath{\,\begin{tikzpicture}[x=0.2ex,y=0.2ex,baseline={([yshift=-0.5ex]current bounding box.center)}]\draw (0,0) rectangle (4,4);\draw[-{Stealth[scale=0.75]}] (4,2) -- (14,2);\end{tikzpicture}\:}}
 \def\TruthDegree {\ensuremath{\zeta}}

\newglossaryentry{truthdegree}{name={interestingness degree},description={the degree of interestingness of a \glsentrydesc{evttrace} to a \glsentrydesc{con}.},
	symbol={\TruthDegree}}

\def\RfActivation {\ensuremath{\alpha}}
\def\RfActivationFormula {\ensuremath{\varphi_{\alpha}}}
\newcommand{\RfActivationFormulaVar}[1] {\ensuremath{\varphi_{\alpha_#1}}}
\def\RfObject {\ensuremath{\tau}}
\def\RfObjectFormula {\ensuremath{\varphi_\tau}}
\newcommand{\RfObjectFormulaVar}[1] {\ensuremath{\varphi_{\tau_#1}}}

\newacronym{ltlp}{LTLp}{Linear-time Temporal Logic with Past}
\def\ltlpf {\ensuremath{\mathrm{LTLp}_f}}
\newacronym{ltlpf}{\ltlpf}{Linear-time Temporal Logic with Past on Finite Traces}
\newacronym{ldl}{LDL}{Linear Dynamic Logic}
\newacronym{ldlf}{LDL$_f$}{Linear Dynamic Logic over Finite Traces}

\def\RfFormula {\ensuremath{\Psi}}

\def\separationDegree {\textit{D}}
\newglossaryentry{sepdegree}{name={separation degree},description={number of triples of a \glsentrytext{sepautset}},
	symbol={\separationDegree}}

\newglossaryentry{sat}{name={satisfaction},description={verification of a formula on a structure}}
\newglossaryentry{fulfilment}{name={fulfilment},description={satisfaction of a constraint on a trace in which the activation occurs}}
\newglossaryentry{activator}{name={activator},description={the event that signals the occurrence of the activation in the trace}}

\def\instant {\ensuremath{i}}
\def\EvtTraceLength {\ensuremath{n}}
\def\LTLfpFormula {\ensuremath{\varphi}}

\newglossaryentry{mirrorimage}{name={mirror image},description={temporal formula obtained by replacing all its operators with their mirror images}}

\def\dnk {\ensuremath{\textsf{{\color{gray}x}}}}

\begin{document}
\let\WriteBookmarks\relax
\def\floatpagepagefraction{1}
\def\textpagefraction{.001}

\shorttitle{Measurement of Rule-based LTLf Declarative Process Specifications}    

\shortauthors{Cecconi, Barbaro, Di Ciccio, Senderovich}  

\title [mode = title]{Measuring Rule-based LTLf Process Specifications: \\ A Probabilistic Data-driven Approach}

\author[1]{Alessio Cecconi}[orcid=0000-0001-5730-6332]
\ead{alessio.cecconi.phd@gmail.com}

\affiliation[1]{organization={Vienna University of Economics and Business},
	addressline={Welthandelsplatz 1}, 
	city={Vienna},
postcode={1020}, 
country={Austria}}

\author[2]{Luca Barbaro}[orcid=0000-0002-2975-5330]
\ead{luca.barbaro@uniroma1.it}
\author[2]{Claudio {Di Ciccio}}[orcid=0000-0001-5570-0475]
\ead{claudio.diciccio@uniroma1.it}
\cormark[1]

\affiliation[2]{organization={Sapienza University of Rome, Department of Computer Science},
	addressline={Viale Regina Elena 295}, 
	city={Rome},
postcode={00161}, 
country={Italy}}

\author[3]{Arik Senderovich}[orcid=0000-0003-4728-8024]
\ead{sariks@yorku.ca}

\affiliation[3]{organization={York University, School of Information Technology (ITEC)},
	addressline={4700 Keele St}, 
	city={Toronto},
postcode={M3J 1P3}, 
country={Canada}}

\cortext[1]{Corresponding author}

\begin{abstract}
	Declarative process specifications define the behavior of processes by means of rules based on \acrfull{ltlf}.
	In a mining context, these specifications are inferred from, and checked on, multi-sets of runs recorded by information systems (namely, event logs).
	To this end, being able to gauge the degree to which process data comply with a specification is key.
	However, existing mining and verification techniques analyze the rules in isolation, thereby disregarding their interplay.
	In this paper, we introduce a framework to devise probabilistic measures for declarative process specifications.
	Thereupon, we propose a technique that measures the degree of satisfaction of specifications over event logs.
	To assess our approach, we conduct an evaluation with real-world data, evidencing its applicability for diverse process mining tasks, including discovery, checking, and drift detection.
\end{abstract}

\begin{highlights}
	\item Interestingness measures for rule-based process specifications are introduced.
	\item A probabilistic estimation framework of single rules and entire specifications is proposed.
	\item Measures of specifications differ from the aggregation of the measures for their individual rules.
	\item Measures exhibit different sensitivity to process behavior changes over time.
\end{highlights}


\maketitle

\section{Introduction}
\label{sec:intro}
\todo[inline]{Reviews and response sheet here: \url{https://docs.google.com/spreadsheets/d/1b200_wsfQbuS8iAAEL5VfCeDZO8B60zq8up17V5khY4/}}
The declarative specification of a process allows users and designers to norm and control its behavior through rules.
These rules consist of temporal logic formulae (such as \ltlf) that are verified against recorded runs of the process-aware systems in an event log, to check their compliance with the behavioral properties it must guarantee.
This automated checking task shows wide adoption in multiple areas of computer science, including process mining~\cite{DBLP:conf/spin/PesicBA10,DBLP:conf/ijcai/GiacomoFMP21}, planning~\cite{DBLP:conf/aaai/BacchusK96,DBLP:conf/aaai/CamachoTMBM17},  and software engineering~\cite{DBLP:conf/kbse/LemieuxPB15,DBLP:journals/ase/CaoTLL18}.
Rule-based specifications allow for the definition of system and process behavior focusing on core constraints that must be satisfied, leaving other execution details to knowledge- and context-based decisions of the actors~\cite{DBLP:books/sp/22/CiccioM22} with a so-called ``open-ended'' paradigm~\cite{DBLP:conf/cidm/MaggiMA11}. This aspect makes them particularly suitable to scenarios in which workflow flexibility is key~\cite{DiCiccio.etal/JoDS2015:KiPs,Aalst.etal/CSRD09:DeclarativeWFsBalancing}, as in the healthcare domain~\cite{DBLP:journals/eswa/RovaniMLA15,process_mining_for_healthcare} with the specification of clinical practice guidelines~\cite{DBLP:journals/widm/GuzzoRV22,Mannhardt2016Sepsis}.

Despite the increasing interest in this challenge, we observe that
a fundamental problem remains unaddressed.
Measuring the extent to which traces adhere to the admissible behavior in terms of \emph{specifications}, or sets of rules, is still a problem that leaves ample margins for investigation.

Let us consider an example inspired by~\cite{anomaly_detection_clinical_pathways} on the clinical pathway for the treatment of unstable angina. A declarative process specification $\DeclaModel$, e.g., consists of two rules, $\Psi_1$ and $\Psi_2$. Rule $\Psi_1$ states that ``Whenever a specific combination of medications is prescribed, it should be preceded by an electrocardiogram (ECG) and a cardiac enzyme test.'' Rule $\Psi_2$ indicates that ``If a patient experiences severe chest pain, an immediate administration of sublingual nitroglycerin will follow.'' Confidence is the ratio of events in a log that satisfy the consequent (i.e., the \emph{target}, like the ECG and cardiac enzyme test in $\Psi_1$), given the satisfaction of the antecedent (i.e., the \emph{activator}, like the occurrence of severe chest pain in $\Psi_2$).

Suppose the confidence of $\Psi_1$ and $\Psi_2$ are \SI{100}{\percent} and \SI{50}{\percent}, respectively. What is the confidence of $\DeclaModel=\{\Psi_1, \Psi_2\}$?
Considering the confidence of a single formula consisting of the conjunction of all rules (e.g., $\Psi_1 \wedge \Psi_2$) may be too coarse grained, since 
violating a single rule in one event or in multiple ones would lead to the same result of violating the whole specification for that trace.\todo{changed the wrong example with a proper one.} Similarly, aggregating measures over multiple rules (as in \cite{Cecconi2021Measuring}) may be misleading. Let the activator of $\Psi_1$, e.g., occur \num{100} times in the log, always leading to the satisfaction of the rule (Confidence \SI{100}{\percent}), and the activator of $\Psi_2$ occur twice, leading to the violation of the rule once (Confidence \SI{50}{\percent}). Consequently, there will be one violation out of \num{102} occurrences of the activators. Yet, the average confidence of the two rules will be \SI{75}{\percent}.
A lack of a proper framework to gauge the degree to which specifications are satisfied by, or emerging from, recorded data, hinders
their adoption for the discovery, checking and drift-detection of system and process behavior in rule-based settings. In turn, this issue slows down the evolution of dedicated techniques even where those tasks may turn out to be pivotal, like in the highly dynamic and flexible context of clinical pathways.

To overcome this issue, we propose a new approach adapting and extending the concept of \acrfull{rf}, originally proposed in~\cite{Cecconi2018Interestingness}, and the measurement framework for single declarative rules expressed as \acrshortpl{rf}~\cite{Cecconi2021Measuring}. \acrshortpl{rf} are rules expressed in an if--then fashion (like $\Psi_1$ and $\Psi_2$), namely a pair of {\ltlf} formulae one of which is the activator (if) and the other is the target (then).
\acrshortpl{rf} cover the full spectrum of declarative process specification languages such as \Declare~\cite{DBLP:conf/spin/PesicBA10} as any {\ltlf} formula can be translated into an \acrshort{rf}~\cite{Cecconi2018Interestingness}.
Equipped with this notion, we propose a measurement framework that takes inspiration from classical association rule mining~\cite{Geng2006Interestingness} to assess whether, and in how far, process specifications consisting of {{\ltlf}-based} rules expressed in an ``if--then'' fashion are satisfied by a trace. Our approach is rooted in probability theory and statistical inference. Specifically, in order to provide a non-binary interpretation for specification measurements, we model events of satisfaction and violation of formulae by traces (and logs of traces) using probability theory, and derive corresponding maximum-likelihood estimators for these probabilistic models. Moreover, we show that these estimators can be computed in polynomial time. 

To the best of our knowledge, this work is the first to tackle and solve the problem of devising well-defined measures for \emph{entire} declarative specifications consisting of multiple rules. To tackle this problem, we move from an ad-hoc counting approach to a sound probabilistic theory based on maximum likelihood estimation. Finally, we conduct an evaluation of our approach on real-world data with its software prototype implementation. This article is an extended and revised version of our previous paper~\cite{Cecconi2022Measurement}. The advancements particularly the following aspects in particular:
\begin{itemize}
	\item We enhanced the theoretical part of our contribution providing a more comprehensive analysis of the probabilistic approach along with detailed proofs for each proposition and theorem (\cref{sec:estimation,sec:approach}); 
 in addition, we added an appendix for the basic concepts from probability theory used in the paper (\cref{appendix:probability});
	\item We introduced a discussion in which we state the general properties and behaviors of the probability estimators (\cref{sec:approach-discussion});
	\item We added a preliminary section about interestingness measures. It introduces these measures prior to their use thus clarifying the goals of the probabilistic approach and the value of the final outcome (\cref{sec:measures-intro});
	\item We extend the quantitative evaluation by using a larger collection of datasets (\cref{sec:evaluation-discovery-vs-specification-measurements});
	\item We present a new and extensive use case where we measure the impact of drifts on declarative specifications. This part demonstrates how measured changes in the entire specification can signal shifts and anomalies in an ongoing process. (\cref{sec:use-case-process-drift}).
\end{itemize}

In the remainder of this paper,
\cref{sec:ltlf-logs,sec:rf,sec:measures-intro} formalize the background notions our work is based upon: {\ltlf} and its interpretation on event logs, \acrshortpl{rf}, rule-based declarative process specifications, and interestingness measures.
\Cref{sec:estimation} lays the foundations of our probabilistic theory, upon which the evaluation and measurement of declarative process specifications are based upon as described in \cref{sec:approach}.
We report on the evaluation of our implemented prototype on real-world data in \cref{sec:evaluation}. \Cref{sec:related-works} analyzes the research in the literature that relates to our investigation. Finally, \cref{sec:conclusions} concludes the paper with remarks on future work.

\section{Event logs and \acrfull{ltlf}}
\label{sec:ltlf-logs}
\begin{table*}[tbp]
	\begin{adjustbox}{width=0.95\textwidth,center=\textwidth}
		\begin{tabular}{clccccccccc}
\toprule
\textbf{Log} & \textbf{Evaluation} & \makecell[c]{\textbf{\acrshort{rf}} \\ \textbf{/ Specification}} & \makecell[c]{\textbf{$ P $ of \acrshort{rf}} \\ / \textbf{$ P $ of \DeclaModel}}  & \textbf{$ P $ of act.} & \textbf{$ P $ of target} & \textbf{Support} & \textbf{Confidence} & \textbf{Recall} & \textbf{Specificity} & \textbf{Lift} \\ 
\midrule

$t_1 = \langle \taska,\taskb,\taskc,\taskd,\taskb,\taskc,\taske,\taskc,\taskb \rangle$ & \makecell[l]{$\langle$\dnk, \dnk, 1, \dnk, \dnk, 1, \dnk, 1, \dnk$\rangle$\\$\langle$\dnk, \dnk, \dnk, 1, \dnk, \dnk, \dnk, \dnk, \dnk$\rangle$\\$\langle$\dnk, \dnk, 1, 1, \dnk, 1, \dnk, 1, \dnk$\rangle$} & \makecell[c]{~$ \RfFormula_1 = \taskc\RfImplication\RNFpast\taska $\\~$ \RfFormula_2 = \taskd\RfImplication\RNFfut\taske $\\~$\DeclaModel=\{\RfFormula_1,\RfFormula_2\}$} & \makecell[c]{~1.00\\~1.00\\~1.00}  & \makecell[c]{~0.33\\~0.11\\~0.44}  & \makecell[c]{~1.00\\~0.78\\~0.89}  & \makecell[c]{~0.33\\~0.11\\~0.44}  & \makecell[c]{~1.00\\~1.00\\~1.00}  & \makecell[c]{~0.33\\~0.14\\~0.50}  & \makecell[c]{~0.00\\~0.25\\~0.20}  & \makecell[c]{~1.00\\~1.29\\~1.13}  \\[1.2em]

$t_2 = \langle \taskb,\taskd,\taska,\taskb,\taskb,\taskd,\taske,\taskd,\taskc \rangle$ & \makecell[l]{$\langle$\dnk, \dnk, \dnk, \dnk, \dnk, \dnk, \dnk, \dnk, 1$\rangle$\\$\langle$\dnk, 1, \dnk, \dnk, \dnk, 1, \dnk, 0, \dnk$\rangle$\\$\langle$\dnk, 1, \dnk, \dnk, \dnk, 1, \dnk, 0, 1$\rangle$} &  \makecell[c]{~$ \RfFormula_1 = \taskc\RfImplication\RNFpast\taska $\\~$ \RfFormula_2 = \taskd\RfImplication\RNFfut\taske $\\~$\DeclaModel=\{\RfFormula_1,\RfFormula_2\}$} &  \makecell[c]{~1.00\\~0.67\\~0.75} &  \makecell[c]{~0.11\\~0.33\\~0.44} &  \makecell[c]{~0.78\\~0.78\\~0.78} &  \makecell[c]{~0.11\\~0.22\\~0.33} &  \makecell[c]{~1.00\\~0.67\\~0.75} &  \makecell[c]{~0.14\\~0.29\\~0.43} &  \makecell[c]{~0.25\\~0.17\\~0.20} &  \makecell[c]{~1.29\\~0.86\\~0.96} \\[1.2em]

$t_3 = \langle \taskc,\taskd,\taska,\taskb,\taskc,\taske,\taskb,\taskc,\taskb,\taskc \rangle$ &  \makecell[l]{$\langle$0, \dnk, \dnk, \dnk, 1, \dnk, \dnk, 1, \dnk, 1$\rangle$\\$\langle$\dnk, 1, \dnk, \dnk, \dnk, \dnk, \dnk, \dnk, \dnk, \dnk$\rangle$\\$\langle$0, 1, \dnk, \dnk, 1, \dnk, \dnk, 1, \dnk, 1$\rangle$} &  \makecell[c]{~$ \RfFormula_1 = \taskc\RfImplication\RNFpast\taska $\\~$ \RfFormula_2 = \taskd\RfImplication\RNFfut\taske $ \\~$\DeclaModel=\{\RfFormula_1,\RfFormula_2\}$ } &  \makecell[c]{~0.75\\~1.00\\~0.80} &  \makecell[c]{~0.40\\~0.10\\~0.50} &  \makecell[c]{~0.80\\~0.60\\~0.70} &  \makecell[c]{~0.30\\~0.10\\~0.40} &  \makecell[c]{~0.75\\~1.00\\~0.80} &  \makecell[c]{~0.38\\~0.17\\~0.57} &  \makecell[c]{~0.17\\~0.44\\~0.40} &  \makecell[c]{~0.94\\~1.67\\~1.14} \\[1.2em]

$t_4 = \langle \taskb,\taskc,\taska,\taskc,\taske,\taska \rangle$ &  \makecell[l]{$\langle$\dnk, 0, \dnk, 1, \dnk, \dnk$\rangle$\\$\langle$\dnk, \dnk, \dnk, \dnk, \dnk, \dnk$\rangle$\\$\langle$\dnk, 0, \dnk, 1, \dnk, \dnk$\rangle$} &  \makecell[c]{~$ \RfFormula_1 = \taskc\RfImplication\RNFpast\taska $\\~$ \RfFormula_2 = \taskd\RfImplication\RNFfut\taske $ \\~$\DeclaModel=\{\RfFormula_1,\RfFormula_2\}$ } &  \makecell[c]{~0.50\\~NaN \\~0.50} &  \makecell[c]{~0.33\\~0.00\\~0.33} &  \makecell[c]{~0.67\\~0.83\\~0.50} &  \makecell[c]{~0.17\\~0.00\\~0.17} &  \makecell[c]{~0.50\\~NaN \\~0.50} &  \makecell[c]{~0.25\\~0.00\\~0.33} &  \makecell[c]{~0.25\\~0.17\\~0.50} &  \makecell[c]{~0.75\\~NaN \\~1.00} \\[1.2em]

$t_5 = \langle \taskb,\taskb,\taskb \rangle$ &  \makecell[l]{$\langle$\dnk, \dnk, \dnk$\rangle$\\$\langle$\dnk, \dnk, \dnk$\rangle$\\$\langle$\dnk, \dnk, \dnk$\rangle$} &  \makecell[c]{~$ \RfFormula_1 = \taskc\RfImplication\RNFpast\taska $\\~$ \RfFormula_2 = \taskd\RfImplication\RNFfut\taske $ \\~$\DeclaModel=\{\RfFormula_1,\RfFormula_2\}$ } &  \makecell[c]{~NaN\\~NaN\\~NaN} &  \makecell[c]{~0.00\\~0.00\\~0.00} &  \makecell[c]{~0.00\\~0.00\\~0.00} &  \makecell[c]{~0.00\\~0.00\\~0.00} &  \makecell[c]{~NaN\\~NaN\\~NaN} &  \makecell[c]{~NaN\\~NaN\\~NaN} &  \makecell[c]{~1.00\\~1.00\\~1.00} &  \makecell[c]{~NaN\\~NaN\\~NaN} \\[2.0em]

\makecell[c]{~$L = \{t_1^{17}, t_2^{6}, t_3^{5}, t_4^{12}, t_5^{5}\}$\\~$ |L|=45 $}   &   &  \makecell[c]{~$ \RfFormula_1 = \taskc\RfImplication\RNFpast\taska $\\~$ \RfFormula_2 = \taskd\RfImplication\RNFfut\taske $ \\~$\DeclaModel=\{\RfFormula_1,\RfFormula_2\}$ } &  \makecell[c]{~0.80\\~0.85\\~0.81} &  \makecell[c]{~0.27\\~0.10\\~0.37} &  \makecell[c]{~0.75\\~0.69\\~0.65} &  \makecell[c]{~0.22\\~0.08\\~0.30} &  \makecell[c]{~0.80\\~0.85\\~0.81} &  \makecell[c]{~0.29\\~0.12\\~0.46} &  \makecell[c]{~0.27\\~0.33\\~0.44} &  \makecell[c]{~1.07\\~1.24\\~1.25} \\[1.2em]
   
\bottomrule
	NaN values denote a division by \num{0}.  & & & & & & & & & & \\
\end{tabular}
 	\end{adjustbox}
	\caption[Checking of a specification against a log.]{Measurements of \acrshortpl{rf} $\RfFormula_1 = \taskc\RfImplication\RNFpast\taska$, and $\RfFormula_2 = \taskd\RfImplication\RNFfut\taske$, and of specification $\DeclaModel = \{ \RfFormula_1, \RfFormula_2 \}$ on a log \vspace{-1ex}}
	\label{tab:running:example}
	\vspace{-5ex}
\end{table*}
In this paper, we are interested in the checking of specifications against collections of traces reporting on multiple executions of the process. As runs can recur, we formalize such structure
as a multi-set of traces, namely an event log.
\begin{definition}[Log]\label{def:eventlog}
	Given a finite alphabet of propositional symbols $\LogAlph$, we name as \emph{event} an assignment for the symbols in $\LogAlph$ and as \emph{trace} a finite sequence of events. An \emph{event log} (or \emph{log} for short) is a finite multi-set of traces $L = \{t_1^{j_1},\ldots,t_m^{j_m}\}$ of cardinality $|L| = \sum_{i=1}^m{j_i}$.
\end{definition}
\noindent
For example, \cref{tab:running:example} presents a log $ L = \{t_1^{17}, t_2^{6}, t_3^{5}, t_4^{12}, t_5^{5}\} $ defined over alphabet $ \LogAlph = \{\taska,\taskb,\taskc,\taskd,\taske\} $. Its cardinality is $45$.

\medskip
\acrfull{ltlf}~\cite{DeGiacomo2013Linear} expresses propositions over linear discrete-time structures of finite length -- namely, traces as per \cref{def:eventlog}. It shares its syntax with \acrfull{ltl}~\cite{Pnueli/FOCS1977:LTL} and is at the basis of declarative process specification languages such as \Declare~\cite{DBLP:books/sp/22/CiccioM22}. 
Here, we endow {\ltlf} with past modalities as in~\cite{DBLP:conf/lop/LichtensteinPZ85}. In the remainder of this section (and, as a compendium, in \cref{appendix:ltlf}), we outline the core concepts of {\ltlf} to which our approach resorts.

\begin{definition}[Syntax of {\ltlf}]
Well-formed \gls{ltlf} formulae are built from an alphabet $\LogAlph \supseteq \{ \lettera \}$ of propositional symbols, auxiliary symbols `$($' and `$)$', propositional constants $\RNFtrue$ and $\RNFfalse$, the logical connectives $\lnot$ and $\land$,
the unary temporal operators {\RNFnext} (\emph{next}) and {\RNFprev} (\emph{yesterday}), and the binary temporal operators \RNFuntil (\emph{until}) and \RNFsince (\emph{since}) as follows:
\par\vspace{-0.75\baselineskip}
\small
\begin{flalign*}
\LTLfpFormula \Coloneqq &
\RNFtrue | \RNFfalse | \lettera |
(\lnot\LTLfpFormula) |
(\LTLfpFormula_1 \land \LTLfpFormula_2) |  \\
&
( \RNFnext \LTLfpFormula) | 
(\LTLfpFormula_1 \RNFuntil \LTLfpFormula_2) | 
(\RNFprev\LTLfpFormula) | 
(\LTLfpFormula_1 \RNFsince \LTLfpFormula_2).
\end{flalign*}
\normalfont
\end{definition}
\noindent
We may omit parentheses when the operator precedence intuitively follows from the expression.
Given $\{\taske, \taskd\} \subseteq \LogAlph$, e.g., the following is an \gls{ltlf} formula:
${(\RNFnext\lnot\taske)\RNFuntil\taskd}$. 

Semantics of \gls{ltlf} is given in terms of finite \glspl{evttrace}, i.e., finite words over the alphabet $2^\LogAlph$.
We name the index of the element in the trace as \emph{instant}.
Intuitively, $\RNFnext \LTLfpFormula$ and $\RNFprev \LTLfpFormula$ indicate that $\LTLfpFormula$ holds true at the next and previous instant, respectively;
$\LTLfpFormula_1 \RNFuntil \LTLfpFormula_2$ states that $\LTLfpFormula_2$ will eventually hold and, until then, $\LTLfpFormula_1$ holds too;
dually, $\LTLfpFormula_1 \RNFsince \LTLfpFormula_2$ signifies that $\LTLfpFormula_2$ holds at some point and, from that instant on, $\LTLfpFormula_1$ holds too.
We formalize the above as follows.
\begin{definition}[Semantics of {\ltlf}]
Given a finite \gls{evttrace} {\EvtTrace} of length $n \in \mathbb{N}$, an \gls{ltlf} formula {\LTLfpFormula} is satisfied at a given instant {\instant} ($ 1\leq\instant\leq\EvtTraceLength $) by induction of the following:
\begin{compactenum}
	\item[$ (\EvtTrace, \instant) $] $  \models \RNFtrue $; 
$ (\EvtTrace, \instant) \nvDash \RNFfalse $; \item[$ (\EvtTrace, \instant) $] $   \models \lettera $ iff $ \lettera $ is $ \RNFtrue $ in $ \EvtTrace(\instant) $; 
	\item[$ (\EvtTrace, \instant) $] $  \models \lnot\LTLfpFormula $ iff $ (\EvtTrace, \instant) \nvDash \LTLfpFormula $; \item[$ (\EvtTrace, \instant) $] $   \models \LTLfpFormula_1\land\LTLfpFormula_2 $ iff $ (\EvtTrace, \instant) \models \LTLfpFormula_1 $ and $ (\EvtTrace, \instant) \models \LTLfpFormula_2 $; \item[$ (\EvtTrace, \instant) $] $  \models \RNFnext\LTLfpFormula $ iff $ i<\EvtTraceLength $ and $ (\EvtTrace, \instant+1) \models \LTLfpFormula $; \item[$ (\EvtTrace, \instant) $] $  \models \RNFprev\LTLfpFormula $ iff $ \instant>1 $ and $ (\EvtTrace, \instant-1) \models \LTLfpFormula $; \item[$ (\EvtTrace, \instant) $] $  \models \LTLfpFormula_1\RNFuntil\LTLfpFormula_2 $ iff $ (\EvtTrace, j) \models\LTLfpFormula_2 $  with $ \instant\leq j\leq\EvtTraceLength $, and $ (\EvtTrace, k) \models \LTLfpFormula_1 $ for all $k$ s.t.\ $ {\instant\leq k<j} $; \item[$ (\EvtTrace, \instant) $] $  \models \LTLfpFormula_1\RNFsince\LTLfpFormula_2 $ iff $ (\EvtTrace, j) \models\LTLfpFormula_2 $ with $ 1\leq j\leq \instant $, and $ (\EvtTrace, k) \models \LTLfpFormula_1 $ for all $k$ s.t.\ $ {j<k\leq \instant} $. \end{compactenum}
\normalfont
\end{definition}
\noindent
Without loss of generality, we consider here the non-strict semantics of $\RNFuntil$ and $\RNFsince$~\cite{DBLP:conf/birthday/HodkinsonR05}.
Also, notice that each event in \cref{tab:running:example} satisfies only one proposition (thus applying the ``Declare assumption''~\cite{DBLP:conf/aaai/GiacomoMM14}) for the sake of simplicity. 
In the following, we might directly refer to the sequence of events $\langle e_1,\ldots,e_n \rangle$ of a trace $t$ of length $n$ to indicate the sequence of assignments at instants $1,\ldots,n$.
For example, $t_1$, $t_2$, and $t_4$ in \cref{tab:running:example} are written as $\langle \taska,\taskb,\taskc,\taskd,\taskb,\taskc,\taske,\taskc,\taskb \rangle$, $\langle \taskb,\taskd,\taska,\taskb,\taskb,\taskd,\taske,\taskd,\taskc \rangle$, and $\langle \taskb,\taskc,\taska,\taskc,\taske,\taska \rangle$, respectively. We thus indicate, e.g., that $(t_1,1) \models \taska$,  $(t_2,4) \models \taskb$, and $(t_4,2) \models \taskc$.
Considering again the formula $(\RNFnext\lnot\taske)\RNFuntil\taskd$, we have that $(t_1,1) $ satisfies it
(i.e., the formula is satisfied at the first instant of $\EvtTrace_1$),
whereas $(t_2,6)$ does not. 

\smallskip
From the above operators, the following can be derived:
\begin{compactitem}
	\item Classical boolean abbreviations $ \lor, \rightarrow $;
	\item Constant $\RNFend \equiv \lnot\RNFnext\RNFtrue $, the last instant of a trace;
	\item Constant $\RNFstart \equiv \lnot\RNFprev\RNFtrue $, the first instant of a trace;
	\item $ \RNFfut\LTLfpFormula \equiv \RNFtrue\RNFuntil\LTLfpFormula $, indicating that $ \LTLfpFormula $ holds true at an instant between the current one (included) and {\RNFend} (we name this operator \emph{eventually});
\item $ \RNFpast\LTLfpFormula \equiv \RNFtrue\RNFsince\LTLfpFormula $, indicating that $ \LTLfpFormula $ holds true at an instant in the closed interval from {\RNFstart} to the current one (\emph{once});
	\item $ \RNFglobFut\LTLfpFormula \equiv \lnot\RNFfut\lnot\LTLfpFormula $, indicating that $ \LTLfpFormula $ holds true at every instant from the current one till {\RNFend} (\emph{always}); \item $ \RNFglobPast\LTLfpFormula \equiv \lnot\RNFpast\lnot\LTLfpFormula $, indicating that $ \LTLfpFormula $ holds true at every instant from {\RNFstart} to the current one (\emph{historically}). \end{compactitem}
\noindent
For example, $\taskd\wedge\RNFfut\taske$ is satisfied in a trace when the propositional atom $\taskd$ holds true and $\taske$ holds true at a later instant in the same trace.
Considering the log in \cref{tab:running:example}, we have that $(t_2, 6) \models \taskd\wedge\RNFfut\taske$ whereas $(t_1,1) \nvDash \taskd\wedge\RNFfut\taske$.

Notice that the semantics of the \emph{past operators} $\RNFprev$, $\RNFsince$, $\RNFpast$, and $\RNFglobPast$ correspond to the semantics of the \emph{future operators} $\RNFnext$, $\RNFuntil$, $\RNFfut$, and $\RNFglobFut$, respectively, if we evaluate them on finite traces read in reverse~\cite{Cecconi2018Interestingness}. For example, the evaluation of 
$(\RNFprev\lnot\taske)\RNFsince\taskd$ on 
$\langle \taska,\taske,\taskc,\taska,\taskc,\taskb \rangle$
is equivalent to the evaluation of 
$(\RNFnext\lnot\taske)\RNFuntil\taskd$ on 
$\langle \taskb,\taskc,\taska,\taskc,\taske,\taska \rangle$.

\medskip
Let $\|\varphi\|$ denote the size of the {\ltlf} formula $\varphi$ in terms of propositional symbols and connectives excluding parentheses. For example, $\|(\RNFnext\lnot\taske)\RNFuntil\taskd\|$ is \num{5} and $\|\taskd\wedge\RNFfut\taske\|$ is \num{4}.

\begin{theorem}[\!\!\cite{DBLP:conf/aaai/FiondaG16}]\label{thm:checking:trace}
	Let $t$ be a finite trace of length $n \in \mathbb{N}$.
	Checking whether $(t,i)$ (with $1 \leq i \leq n$) satisfies an {\ltlf} formula $\varphi$, $(t,i) \models \varphi$, is feasible in $O(n^2 \times \|\varphi\|)$.
\end{theorem}
\begin{proof}\itshape
It follows from the proof in~\cite{DBLP:conf/aaai/FiondaG16} elaborated for future operators (\RNFnext, \RNFfut, \RNFglobFut, and $\!\!\RNFuntil\!\!$). Notice that the use of past modalities $\RNFprev$, $\RNFglobPast$, $\RNFpast$ and $\!\!\RNFsince\!\!$ do not alter the complexity. Indeed, they can be included in the parse tree of the constructive proof in~\cite{DBLP:conf/aaai/FiondaG16} as the respective future counterparts and checked against the trace read in reverse (i.e., from end to start~\cite{Cecconi2018Interestingness}).
\end{proof}

\begin{corollary}\label{thm:checking:log}
	Let $L$ be an event log as per \cref{def:eventlog} consisting of $m \in \mathbb{N}$ distinct traces of length up to $n \in \mathbb{N}$ and cardinality $|L| \geq m$.
	Labeling the events in $L$ that satisfy an {\ltlf} formula $\varphi$ is feasible in $O(n^3 \times \|\varphi\| \times m)$.
\end{corollary}
\begin{proof}\itshape
	The proof follows from \cref{thm:checking:trace}: the checking is done for the $O(n)$ events of all $m$ distinct traces in $L$.
\end{proof}

\section{\acrfullpl{rf} and process specifications}
\label{sec:rf}
In this section, we illustrate how we adopt {\ltlf} to express rules specifying the behavior of a process in the form of \glspl{rf}.

\subsection{\acrfullpl{rf}}\label{sec:rcon:syntax:semantics}
\Glspl{rf} express \gls{ltlf}-based rules as antecedent-consequent pairs in an ``if-then'' fashion.
Next, we formalize their definition.
\begin{definition}[\acrfull{rf}]\label{def:reactiveconstraint}
	Given an alphabet of propositional symbols
	$\LogAlph$, let {\RfActivationFormula} and {\RfObjectFormula} be \gls{ltlf} formulae over {\LogAlph}. 
	A \emph{\acrfull{rf}} $\RfFormula$ is a pair $ ( \RfActivationFormula, \RfObjectFormula ) $ hereafter denoted as 
	$\RfFormula \triangleq \RfActivationFormula \RfImplication \RfObjectFormula$.
We name {\RfActivationFormula} as \emph{activator} and {\RfObjectFormula} as \emph{target}.
\end{definition}

We define the semantics of an \gls{rf} $\RfFormula = \RfActivationFormula \RfImplication \RfObjectFormula$ as follows: given a trace $t$ of length $n$ and an instant $i$ with $1 \leq i \leq n$, we say that
\begin{compactdesc}
	\item[{\RfFormula} is satisfied by $\EvtTrace$ in $\instant$\normalfont,] i.e., $ \EvtTrace, \instant \models \RfFormula $, iff $ \EvtTrace, \instant \models \RfActivationFormula$ and $ \EvtTrace, \instant \models \RfObjectFormula$
	\item[{\RfFormula} is violated by $\EvtTrace$ in $\instant$\normalfont,] $ \EvtTrace, \instant \nvDash \RfFormula $, iff $ \EvtTrace, \instant \models \RfActivationFormula$ and $ \EvtTrace, \instant \nvDash \RfObjectFormula$
	\item[{\RfFormula} is unaffected by $\EvtTrace$ in $\instant$\normalfont,] iff $ \EvtTrace, \instant \nvDash \RfActivationFormula$.
\end{compactdesc}
We also say that {\RfFormula} \textbf{is activated by} $t$ if there exists an instant $i$ s.t.\ $1 \leq i \leq n$ and $t,i \models \RfActivationFormula$.
For example, take $ \RfFormula_1 = \taskc\RfImplication\RNFpast\taska$ in \cref{tab:running:example}. At every occurrence of {\taskc} (the activator),  $\RfFormula_1$ is either \emph{satisfied} (if {\taskc} is eventually preceded by {\taska} as $\RNFpast\taska$ is the target), or \emph{violated}. $\RfFormula_1$ is instead \emph{unaffected} by those events in which {\taskc} does not occur.
The activator of $ \RfFormula_2 = \taskd\RfImplication\RNFfut\taske$ in \cref{tab:running:example} is {\taskd} and the target is $\RNFfut\taske$.
Whenever {\taskd} occurs, it is either satisfied (if eventually followed by {\taske}) or violated (otherwise). It is unaffected by events wherein {\taskd} does not hold.
Notice that by declaring that the activator of  $ \RfFormula_1 $ is {\taskc}, the user makes the ``trigger'' of the rule explicit. $ \RfFormula_1 $ and $ \RfFormula_2 $ are the \gls{rf} representation of what are known as \Preced{\taskc}{\taska} and \Resp{\taskd}{\taske} in the declarative process specification language \Declare~\cite{DBLP:conf/spin/PesicBA10}, respectively.

\paragraph{A Note on the Role of the Activators in \glspl{rf}.}
We remark that any {\ltlf} formula $\phi$ can be expressed by means of an \gls{rf}. The expressiveness of \glspl{rf} fully covers that of {\ltlf} and, a fortiori, of {\Declare}. The examination of the expressiveness of \glspl{rf} is discussed in detail in~\cite{Cecconi2018Interestingness,Cecconi2021Measuring}.
On the other hand, the temporal conditions that an \ltlf formula such as $\RNFglobFut (\phi_1 \to \phi_2)$ exerts do not substantially differ from those of $\phi_1 \RfImplication \phi_2$.
However, we adopt \glspl{rf} to express rules in a mining context because the explicit definition of the activator makes it possible to distinguish an interesting satisfaction from a vacuous one~\cite{DBLP:journals/sttt/KupfermanV03}.
The {\ltlf} formulae
$\psi_1 = \RNFglobFut (\taskc \to \RNFpast\taska)$
and
$\psi_2 = \RNFglobFut \left( ( \lnot \taskc \RNFuntil \taska ) \lor \lnot \taskc \right)$
both express the \Preced{\taskc}{\taska} constraint.
The rule dictates that, for every event, \emph{if} {\taskc} occurs, then it must be preceded by {\taska}.
This explanation closely resembles $\psi_1$.
Therefore, if {\taskc} does not occur, the constraint is also satisfied
(regardless of whether {\taska} occurs or not: \emph{ex falso sequitur quodlibet}), though vacuously.
Notice that the latter condition is explicit in the second conjunct under $\RNFglobFut$ in $\psi_2$: ``$\ldots \vee \lnot \taskc$''.
Indicating that a rule is satisfied though being unable to distinguish whether it is actually triggered or not adds limited information~\cite{Maggi.etal/CIDM2011:UserGuidedDiscovery}.
The adoption of \glspl{rf}, though not solving the vacuity problem (activator and target could still be vacuously satisfiable formulae per se), lets \emph{the user} explicitly define the conditions that make a satisfaction interesting.
With $ \taskc \RfImplication \RNFpast\taska $, e.g., we state that the rule is unaffected (neither violated nor satisfied) by events in which $\taskc$ does not occur.
Notice, however, that alternative expressions can be used to customize the interpretation of the rules.
To adopt the classical binary interpretation of {\ltlf} formulae, \Preced{\taskc}{\taska} can be expressed, e.g., as
$ {\Psi_3 = \RNFtrue \RfImplication \left( ( \lnot \taskc \RNFuntil \taska ) \lor \lnot \taskc \right)} $,
or
$ {\Psi_4 =  \RNFstart \RfImplication \RNFglobFut (\taskc \to \RNFpast\taska)} $.
In the first case, $\Psi_3$ indicates that every event activates the \gls{rf} (because $\RNFtrue$ holds in every event).
Therefore, the satisfaction of the target determines the satisfaction of the constraint.
In the second case, the activator of $\Psi_4$ is $ \RNFstart $.
Thus, the first event acts as a representative for the whole trace as either satisfying (if no events in it violates the constraint) or violating (if at least one event violates it). The constraint is unaffected by the trace in every other event.
The definition of well-formed \glspl{rf} and guidelines for their formulation in a mining context go beyond the scope of this paper, and suggests interesting theoretical investigation for future work.

\subsection{Process specifications}
Equipped with the above notion of \glspl{rf} and the rationale behind their use in this context, we can define a process specification as follows.
\begin{definition}[Rule-based {\ltlf} process specification]\label{def:specification}
	A rule-based {\ltlf} process specification (henceforth, \emph{specification} for short) is a finite non-empty set of \glspl{rf} $ \DeclaModel \triangleq \{ \RfFormula_1, \ldots, \RfFormula_s \}$, with $s \in \mathbb{N}$. \end{definition}
For example,~\cref{tab:running:example} presents a specification $ \mathcal{S}=\{\RfFormula_1, \RfFormula_2\} $ composed by the $ \RfFormula_1$ \gls{rf} above and $ \RfFormula_2 = \taskd\RfImplication\RNFfut\taske$.

\begin{corollary}\label{thm:specification:check}
	Let $\DeclaModel = \{ \RfFormula_1, \ldots, \RfFormula_s \}$ be a specification consisting of $s$ \glspl{rf}, the activator and target of which are of size up to $\|\varphi\|$.
	Labeling the events in $L$ with the satisfaction of activator and target of every \gls{rf} in $\DeclaModel$ is feasible in $O(n^3 \times \|\varphi\| \times |L| \times s)$.
\end{corollary}
\begin{proof}\itshape The proof follows from \cref{thm:checking:log}, as a pair of labels is sufficient for all \glspl{rf} in the specification. \end{proof}
Take, e.g., trace $t_4 = \langle \taskb,\taskc,\taska,\taskc,\taske,\taska \rangle$ from \cref{tab:running:example} and the aforementioned \gls{rf} $ \RfFormula_1 = \taskc\RfImplication\RNFpast\taska$.
The activator ($\taskc$) is satisfied in $(t_4,2)$ and $(t_4,4)$. We can thus label every event in $t_4$ thereby creating a new sequence as follows: $\langle 0,1,0,1,0,0 \rangle$ where $1$ and $0$ indicate a satisfaction and a violation of the formula in the corresponding event, respectively. Similarly, we can create a sequence of labels denoting whether the target ($\RNFpast\taska$) is satisfied: $\langle 0,0,1,1,1,1 \rangle$.

A trivial approach to classify traces as compliant with an \glspl{rf} or not is to check whether no event violates it. Nevertheless, especially in checking contexts, understanding the \emph{extent} to which a trace and a log satisfy a specification is key~\cite{DBLP:conf/aips/GiacomoMMS16}.
Next, we introduce the interestingness measures, i.e., the quantitative device we employ to quantify the extent of specification satisfaction. 
Subsequently, we lay the foundations rooted in probabilistic theory to reach this goal.

\section{Interestingness Measures for \glspl{rf}}
\label{sec:measures-intro}

\todo{added section o briefly introduce interestingness measures. It makes the story of the paper better and places more clearly the contribution}

The association rule mining field has a long history of measures development for association rules, also called interestingness measures~\cite{Adamo/2001:Dataminingassociationrules,Geng2006Interestingness}.
These rules have a standard ``if-$A$-then-$B$'' form, where $A$ and $B$ are elements that may occur in a phenomenon, e.g., instructions in a set of database transactions, or events in a set of process traces.
These measures are based on the (joint) probabilities of the occurrences (and co-occurrences) of $A$ and $B$ with the general goal of understanding whether there is a significant (directional) relation between the two elements or, alternatively, whether their co-occurrence is uncorrelated. Historically, the development of these measures began in market basket analysis with the introduction of the A-Priori algorithm~\cite{Agrawal1994Fast}. The problem that the algorithm addressed was to discover, given a set of transactions, association rules between sets of frequently co-occurring elements. Yet, the simple co-occurrence does not necessarily imply a correlation between the elements. Therefore, more refined measures that distinguish authentic associations from spurious ones were developed~\cite{Geng2006Interestingness}.

\begin{table}[tbp] 
\begin{tabular}{cc}
\toprule
\textbf{Measure}     & \textbf{Definition}   \\
\midrule
Support     & $ P(A \cap B) $ \\
Confidence  & $ P(B|A) $ \\
Recall      & $ P(A|B) $ \\
Specificity & $ P(\lnot B | \lnot A) $ \\
Lift        & $ \dfrac{P(A \cap B)}{P(A)P(B)} $ \\
\bottomrule
\end{tabular} \caption[Selection of interestingness measures]{A selection of interestingness measures}
        \label{tab:measures-def}
\end{table}
Let us consider a few examples of such measures, presented in~\cref{tab:measures-def}: 
Given an ``if-$A$-then-$B$'' rule, 
the Support measure, $ P(A \cap B) $, quantifies the joint occurrences of the two elements;
Confidence, $ P(B|A) $, considers the occurrence of $B$ only when we know that $A$ occurred;
Specificity, $P(\lnot B|\lnot A) $, measures the non-occurrence of $B$ given the non-occurrence of $A$;
Recall, $ P(A|B) $, measures the conditional occurrence of $A$ given $B$;
finally, Lift, $ \frac{P(A \cap B)}{P(A)P(B)} $, is the ratio of the probability of co-occurrence of $A$ and $B$ over the product of the individual probabilities of $A$ and $B$ to occur. 
Intuitively, 
Support is high when $A$ and $B$ occur frequently together;
Confidence is high when $B$ occurs every time $A$ occurs, while
Recall is high in the opposite situation, i.e., if $A$ occurs every time $B$ occurs;
Specificity is high when $B$ does not occur if $A$ does not occur;
Lift is high when the separate probability of $A$ and $B$ occurring is lower than the probability of their joint occurrence.

Clearly, these measures quantify different aspects of the rule and, depending on the needs of the user, one measure may turn out to be more useful than another. A plethora of other interestingness measures have been defined in the literature. We refer the reader to existing surveys for an extensive overview~\cite{Geng2006Interestingness,Lenca2008OnSelecting,Tan2004Selecting}.

In the context of process mining, it has been already shown how to apply these interestingness measures to single \gls{rf} rules~\cite{Cecconi2021Measuring}. By definition, \glspl{rf} are ``if-then'' rules too, which makes these measures suitable for interestingness measurement based on the probabilities of the satisfaction of the activator and target conditions in a trace or a log.

The main goal of this paper is to extend these results to specifications of \glspl{rf}. To this end, we must assume a probabilistic model and derive sound statistical estimators for specifications satisfied by traces and, subsequently, logs.

\section{Estimators for \acrshort{ltlf} formulae}
\label{sec:estimation}

The interestingness measures for \glspl{rf} are based on the probabilities of their activator ({\RfActivationFormula}) and target ({\RfObjectFormula}) \gls{ltlf} formulae.
In this section,
we propose estimators for the probabilities of traces and logs satisfying \gls{ltlf} formulae and show that these estimators are computable in polynomial time.

\subsection{Trace estimators}
\label{sec:trace-estimators}
We start by defining probabilistic models 
for the evaluation of formulae over traces. One can consider the probability of an event in a given trace $t = \langle e_1, \ldots, e_n \rangle$ to satisfy an \gls{ltlf} formula $\varphi$ as the degree to which $\varphi$ is satisfied in that trace, which we denote as $P(\varphi(t))$. 

Throughout this work, we assume the existence of a 
labeling mechanism $\Lambda$ that, 
when given an event $e$ in a trace $t$ and a formula $\varphi$, 
marks the event with $1$ if the event satisfies $\varphi$ or with $0$ otherwise, i.e., $\Lambda(e,\varphi)\in \{0,1\}$.
This procedure can be achieved in polynomial time as shown in Corollary~\ref{thm:checking:log} through automata-based techniques for \gls{ltlf} formulae verification~\cite{Cecconi2021Measuring}.

Therefore, every trace $t$ can be associated with a binary sequence $x_{\varphi,t}$ as follows: \begin{equation} x_{\varphi,t} = \left\langle\Lambda\!\left(e_1, \varphi\right), \ldots ,\Lambda\!\left(e_n, \varphi\right)\right\rangle. \nonumber \end{equation}
Take again, e.g.,
$t_2 = \langle e_{2,1}, \ldots, e_{2,9} \rangle = \langle \taskb,\taskd,\taska,\taskb,\taskb,\taskd,\taske,\taskd,\taskc \rangle$ from \cref{tab:running:example}, and formula $\varphi=\taskd\wedge\RNFfut\taske$.
As only $(t_2, 2)$ and $(t_2, 6)$ satisfy $\varphi$, $\Lambda\!\left(e_{2,1}, \varphi\right)$ and $\Lambda\!\left(e_{2,6}, \varphi\right)$ return $1$ while $\Lambda\!\left(e_{2,i}, \varphi\right)$ is $0$ for every $i\in\{1,\ldots,9\}\setminus\{2,6\}$, i.e., \begin{equation}
x_{\varphi,t_2} = \langle 0,1,0,0,0,1,0,0,0 \rangle. \nonumber \end{equation} \noindent In what follows, we assume that $P(\varphi(t))$ (the probability of $t$ to satisfy $\varphi$),
is independent of the position of the event in the trace. This is an uninformative prior assumption, i.e., we assume that 
we are unaware of the values of other events and of the event location within the trace
when evaluating a specific event. Thus, the sequence $x_{\varphi,t}$ can be viewed as an independent and identically distributed (i.id.)
draw from a Bernoulli random variable $X_{\varphi,t}$, which takes
the value of $1$ with probability $P(\varphi(t))$ and $0$ otherwise, i.e., 
\begin{equation}
X_{\varphi,t} =
\begin{cases}
\begin{aligned}
	1, &  &  & \textrm{w.p.} \ \ P(\varphi(t)), \\
	0, &  &  & \textrm{otherwise}.
\end{aligned}
\end{cases} 
\label{eq:trace:eval:rand:var}
\end{equation}
This leads us to our first estimator, namely that of $P(\varphi(t))$.\footnote{
We write $P(\varphi(t))$ for the probability of specific events such as the satisfaction of a formula, and $P(X=x)$
for the probability that a random variable $X$ is set to a specific value $x$. See Appendix~\ref{appendix:probability} for 
basic probability notation.}

\begin{prop} \label{prop:uni}
The maximum likelihood estimator (MLE)\footnote{MLE estimators exhibit important statistical properties such as unbiasedness and consistency (see~\cref{appendix:stats}).} for $P(\varphi(t))$ where $t = \langle e_1, \ldots, e_n \rangle$ is
\begin{equation} \label{eq:unibernoulli}
	\widehat{P(\varphi(t))} = \frac{1}{n} \sum_{i=1}^n \Lambda(e_i, \varphi).
\end{equation}
\end{prop}
\begin{proof}\itshape
Since $X_{\varphi,t}$ 
is a 
univariate Bernoulli 
random variable its MLE is well-established in the 
literature (see, e.g., \cite{bickel2015mathematical}). 
Specifically, it equals to the ratio of `successful trials' over the $n$ trials. \end{proof} 
Returning to our running example (\cref{tab:running:example}), the MLE estimator is used to compute the trace probabilities of the {\RfActivationFormula} formula ($P$ of act.) and of the {\RfObjectFormula} formula ($P$ of target) for each \gls{rf} $\RfActivationFormula \RfImplication \RfObjectFormula$ in $\mathcal{S}$ and for every trace $ t\in L $.
Let us consider again trace
$t_4 = \langle \taskb,\taskc,\taska,\taskc,\taske,\taska \rangle$ and the \gls{rf} $ \RfFormula_1 =  \varphi_{\alpha_1}\RfImplication\varphi_{\tau_1}= \taskc\RfImplication\RNFpast\taska$.
As we have previously discussed in \cref{sec:rf}, the evaluation of $\varphi_{\alpha_1}$ and $\varphi_{\tau_1}$ on $t_4$ leads to the following sequences of labels, respectively:
$\langle 0,1,0,1,0,0 \rangle$ and
$\langle 0,0,1,1,1,1 \rangle$.
Therefore, we conclude that
$\widehat{P\!\left(\varphi_{\alpha_1}\left(t_4\right)\right)}$ is $\frac{2}{6}$
and
$\widehat{P\!\left(\varphi_{\tau_1}\left(t_4\right)\right)}$ is $\frac{4}{6}$.

In order to obtain measures of interest for formulae and specifications
we must, in addition, obtain estimators for 
the intersection of two \gls{ltlf} formulae $\varphi_1$ and $\varphi_2$
being satisfied by a trace,
e.g.,
$P(\varphi_1(t) \cap \varphi_2(t))$,
and for the conditional distribution of $\varphi_1$ to be satisfied by trace $t$ 
conditional on $\varphi_2$ being satisfied by the trace, e.g.,
$P(\varphi_1(t) | \varphi_2(t))$. 

The latter will be particularly useful to extend
the estimators to entire process specifications. Notice that we provide results for the satisfaction of formulae, 
yet similar results can be derived for violations 
by quantifying, e.g., $P\!\left(\neg \varphi_1(t) \cap \varphi_2(t)\right)$ and $P\!\left(\neg \varphi_1(t) | \varphi_2(t)\right)$.
Formalizing the above, we wish to estimate the quantities 
of interest from a labeled sequence,  $$x_{\left(\varphi_1, \varphi_2\right), t} = \left\langle\left(\Lambda\!\left(e_i, \varphi_1\right),\Lambda\!\left(e_i, \varphi_2\right)\right)\right\rangle_{i=1}^n,$$ for $t = \langle e_1, \ldots, e_n \rangle$.

Take, for example, $t_4 = \langle \taskb,\taskc,\taska,\taskc,\taske,\taska \rangle$ from \cref{tab:running:example} and the pair of activator and target of $\Psi_1=\taskc\RfImplication\RNFpast\taska$, i.e., $\varphi_{\alpha_1}=\taskc$ and $\varphi_{\tau_1}=\RNFpast\taska$, respectively.
We have that
$x_{(\varphi_{\alpha_1},\varphi_{\tau_1}),t_4}$ is $\langle (0,0),(1,0),(0,1),(1,1),(0,1),(0,1) \rangle$.

The resulting joint sequence $x_{\left(\varphi_1, \varphi_2\right), t}$ is again assumed to be an i.id.\ draw from a \emph{bivariate} Bernoulli random variable
$X_{\left(\varphi_1, \varphi_2\right), t}$. The bivariate Bernoulli corresponds to four parameters related to the four possible outcomes, namely:
\begin{equation} \label{eq:bivariate}
X_{\left(\varphi_1, \varphi_2\right), t} = 
\begin{cases}
\begin{aligned}
	(0,0), &  &  & \textrm{w.p} \ \ p_{00}, \\
	(0,1), &  &  & \textrm{w.p} \ \ p_{01}, \\
	(1,0), &  &  & \textrm{w.p} \ \ p_{10}, \\
	(1,1), &  &  & \textrm{w.p} \ \ p_{11},
\end{aligned}
\end{cases}
\end{equation}such that $\sum_{i,j} p_{ij} = 1$.
A more detailed definition of \emph{bivariate} Bernoulli random variables is given in~\cite{dai2013multivariate}.
The MLE for each $p_{ij}$ is proposed in~\cite{ip2015multivariate}. 
When estimating $P\!\left(\varphi_1(t) \cap \varphi_2(t)\right)$ 
we are essentially interested in an estimator for $p_{11}$,
which is the probability that both formulae are satisfied by the trace. 
\begin{prop}\label{prop:joint}
	The MLE for $P\!\left(\varphi_1(t) \cap \varphi_2(t)\right)$ given a trace $t = \langle e_1, \ldots, e_n \rangle$ is
\begin{equation} \widehat{P\!\left(\varphi_1(t) \cap \varphi_2(t)\right)} = \hat{p}_{11} = \frac{1}{n}\sum_{i=1}^n \Lambda\!\left(e_i, \varphi_1\right) \cdot \Lambda\!\left(e_i, \varphi_2\right).\end{equation}
\end{prop} 
\begin{proof}\itshape The proof follows from the MLE estimators for Bivariate
Bernoulli random variables with corresponding success probabilities, $p_{ij}$, with $ i,j \in \{0,1\}$ (see~\cite{ip2015multivariate}). \end{proof} \noindent From the example above, we conclude that \begin{equation}
\widehat{P\!\left(\varphi_{\alpha_1}\left(t_4\right) \cap \varphi_{\tau_1}\left(t_4\right)\right)} = \frac{1}{6}, \end{equation} as only one element in the sequence is $(1,1)$. Similarly to $P(\varphi_1(t) \cap \varphi_2(t))$, we can estimate the other combinations of satisfaction and violation of the two formulae (using $\hat{p}_{00}$, $\hat{p}_{01}$, and $\hat{p}_{10}$). 
\begin{remark}
One may be tempted to assume independence between the two sequences $\Lambda\!\left(e_i, \varphi_1\right)_{i=1}^{n}$ and $\Lambda\!\left(e_i, \varphi_2\right)_{i=1}^{n}$. However, this is seldomly the case. Consider two formulae: one that requires activity $\lettera$ to happen at least once in $t$, and the other that requires that activity $\lettera$ never occurs. Clearly, the evaluation of these two formulae would not be independent. 
\end{remark} Having modeled the joint probability of two \gls{ltlf} formulae satisfied by a trace, we can now define the probability of one formula being satisfied (or violated) by $t$ conditioned on another formula being satisfied (or violated) by the same trace $t$.
The conditional distribution of the random variable,
$X_{\varphi_1,t} \ | \ X_{\varphi_2, t} = x_2$,
is a \emph{univariate} Bernoulli that depends on the sequence $x_2$ (which results from applying $\Lambda$ to $\varphi_2$ and $t$) and on the four parameters $p_{ij}$ of the joint bivariate Bernoulli distribution~(see the proof in~\cite{dai2013multivariate}). This result 
leads to the following estimator of the conditional probability. 
\begin{prop} \label{prop:conditional}
The MLE for $P\!\left(\varphi_1(t) | \varphi_2(t)\right)$  given a trace $t = \langle e_1, \ldots, e_n \rangle$ 
is 
\begin{equation}\widehat{P\!\left(\varphi_1(t) | \varphi_2(t)\right)} = \frac{\sum_{i=1}^n \Lambda\!\left(e_i, \varphi_1\right)\cdot \Lambda\!\left(e_i, \varphi_2\right)}{\sum_{i=1}^n \Lambda\!\left(e_i, \varphi_2\right)}.
\end{equation}
\end{prop}
\begin{proof}\itshape The conditional distribution for bivariate Bernoulli is a univariate Bernoulli with success probability of $\frac{\hat{p}_{11}}{\hat{p}_{01}+\hat{p}_{11}}$~\cite{dai2013multivariate}. Therefore, the MLE estimation of $\widehat{P\!\left(\varphi_1(t) | \varphi_2(t)\right)}$ is that of a univariate Bernoulli with that success probability~\cite{bickel2015mathematical}.

\noindent
This leads to the following estimator:
\begin{equation*}\widehat{P\!\left(\varphi_1(t) | \varphi_2(t)\right)} = \frac{\hat{p}_{11}}{\hat{p}_{01}+\hat{p}_{11}},
\end{equation*}
with $\hat{p}_{11}$ as derived in \cref{prop:joint} and $\hat{p}_{01}$ being
\begin{equation*}\hat{p}_{01} = \frac{1}{n} \sum_{i=1}^n \left(1-\Lambda\!\left(e_i, \varphi_1\right)\right)\cdot \Lambda\!\left(e_i, \varphi_2\right).\end{equation*} 
 
\noindent
Therefore, \BeginSmEq
\begin{equation*} \widehat{P\!\left(\varphi_1(t) | \varphi_2(t)\right)} = \frac{\hat{p}_{11}}{\hat{p}_{01}+\hat{p}_{11}} 
                         = \frac{\sum_{i=1}^n \Lambda\!\left(e_i, \varphi_1\right)\cdot \Lambda\!\left(e_i, \varphi_2\right)}{\sum_{i=1}^n \Lambda\!\left(e_i, \varphi_2\right)}.
\end{equation*}
\EndSmEq

 \end{proof}
\begin{remark}
    When estimating $\widehat{P\!\left(\varphi_1(t) | \varphi_2(t)\right)}$, the denominator of the estimator may be equal to $0$. In such a case, the conditional probability is ill-defined and the trace is ignored for log-level computations; the value is denoted as $\textrm{NaN}$.
\end{remark}

Take, e.g., $t_4 = \langle \taskb,\taskc,\taska,\taskc,\taske,\taska \rangle$ from \cref{tab:running:example} as above, $\varphi_{\alpha_1}=\taskc$ and $\varphi_{\tau_1}=\RNFpast\taska$, namely the activator and target of $\Psi_1$.
We have that $\widehat{P\!\left(\varphi_{\tau_1}\left(t_4\right) | \varphi_{\alpha_1}\left(t_4\right)\right)}$ is $\frac{1}{2}$ as $\varphi_{\alpha_1}$ is satisfied by two events, only one of which satisfies $\varphi_{\tau_1}$ too.

Similarly, we can estimate the remaining that correspond to the random variable, $X_{\varphi_1,t} \ | \ X_{\varphi_2, t} = x_2$, i.e., $P\!\left(\neg \varphi_1(t) | \neg \varphi_2(t)\right)$, $P\!\left(\varphi_1(t) | \neg \varphi_2(t)\right)$, and $P\!\left(\neg \varphi_1(t) | \varphi_2(t)\right)$, since all of them are derived from the random variable $X_{\varphi_1,t} \ | \ X_{\varphi_2, t} = x_2$ and the fact that it is a univariate Bernoulli.

\subsection{Log estimators}
We lift our results from traces to logs by estimating
$P\!\left(\varphi(L)\right)$, i.e., the probability that the log $L$ satisfies a formula $\varphi$. Recall that an event log $L = \left\{t_1^{j_1}, \ldots, t_m^{j_m}\right\}$ is a bag of traces with trace $t_i^{j_i}$ occurring $j_i$ times in the log. The size of the log $|L| = \sum_{i=1}^m j_i$ is the total number of traces.

Let us denote with $\bar{L} = \{t_1, \ldots, t_m\}$ the set of unique traces in $L$.
We assume that the traces in $L$ are independently generated by a 
trace generator $T$, which is, in turn, associated with a discrete probability function $P(T=t)$.\footnote{In practice, the trace can be generated via a random walk over, e.g., a finite-state automaton~\cite{DBLP:conf/caise/CiccioBCM15}.} Let $\mathcal{T}$ be the support of the probability distribution of $T$, i.e., \begin{equation} \mathcal{T} = \{t \ | \ P(T=t)>0\}.\end{equation} First, we generalize our definitions from a given trace $t$ to a \emph{random} trace $T$.
To this end, we 
assume log completeness: the log contains all possible traces that can be generated from $T$, i.e., $\bar{L} = \mathcal{T}$.
We plan to lift this assumption in future work.

Next, we shall define $X_{\varphi, T}$ as a sequence of binary evaluations -- similarly to the approach we adopted to define $X_{\varphi, t}$ in Eq.~\eqref{eq:trace:eval:rand:var}, yet over a random trace $T$.
Note that $X_{\varphi, T}$ is essentially a doubly stochastic random variable, 
as its success probability, $P\!\left(X_{\varphi, T} = 1\right)$, changes for traces randomly sampled from the distribution of $T$ (which we assume to be generalized Bernoulli, i.e., Bernoulli with multiple outcomes). 
We shall use
$X_{\varphi, T}$ together with our log completeness assumption to derive 
an estimator for $P(\varphi(L))$. 
\begin{prop}\label{MLE:bayes1}
The MLE for $P(\varphi(L))$ for a log $L$ with a trace set $\bar{L} = \{ t_i = \langle e_{i,1}, \ldots, e_{i,n_i} \rangle\}_{i=1}^m$ is
\begin{align}\label{eq:p:phi:log:mle}
	\widehat{P(\varphi(L))} = \frac{1}{|L|} \sum_{i=1}^{m} \frac{j_i}{n_i} \sum_{k=1}^{n_i} \Lambda\!\left(e_{i,k}, \varphi\right).
\end{align} with $m$ being the number of unique traces in $L$ and $|L|$ denoting the number of traces in the log. \end{prop}\begin{proof}\itshape 
From the law of total probability (see Appendix~\ref{appendix:probability}) we get that \BeginSmEq
\begin{align}\label{eq:ltp_log}
	P\!\left(X_{\varphi, T} = x\right) &= \sum_{t\in \mathcal{T}} P(T=t) P\!\left(X_{\varphi, T}= x \mid T=t\right) = \nonumber \\
                          & =      \sum_{t\in \mathcal{T}} P(T=t) P\!\left(X_{\varphi, t}= x\right),
	\end{align} \EndSmEq which provides a link between log-based evaluation of formulae and trace-based evaluation. 
 Since we assume log completeness, we may replace $\mathcal{T}$ with $\bar{L}$ in Eq.~\eqref{eq:ltp_log}, and plug in $P(\varphi(L))$ for $P\!\left(X_{\varphi, T}= 1\right)$ and $P(\varphi(t))$ for 
$P(X_{\varphi, t}= 1)$, thus yielding the following: 
\begin{align}\label{eq:p:phi:log:mle2}
	\widehat{P(\varphi(L))} =  \sum_{i=1}^{m} \widehat{P(T=t_i)} \widehat{P(\varphi(t_i))}.                            
\end{align} The term $\widehat{P(\varphi(t_i))}$ is estimated using \cref{prop:uni}, and $T$ is assumed 
to be a generalized Bernoulli random variable which has a known MLE estimator as follows (it is a special case of the Multinomial random variable),
\begin{equation}
    \widehat{P(T=t_i)} = \frac{j_i}{\sum_{k=1}^m j_k} = \frac{j_i}{|L|}.
\end{equation} The proof for the latter can be found in~\cite{rao1957maximum}. 

Finally, to show that the proposed estimator for $\widehat{P(\varphi(L))}$ is indeed an MLE we use
a well-known result in statistics is that any function of multiple MLE estimators yields an MLE estimator~\cite[Ch.~7]{casella2021statistical},
which completes our proof. \end{proof}
For example, $t_4$ has a multiplicity of \num{12} considering the example log in~\cref{tab:running:example}. The cardinality of the example log $L$ is \num{45}, so $\widehat{P(t=t_4)}$ is $\frac{12}{45}$.
We saw in \cref{sec:trace-estimators} that 
$\widehat{P\!\left(\varphi_{\alpha_1}(t_4)\right)}$ is $\frac{2}{6}$
for 
$ \varphi_{\alpha_1} = \taskc
$.
Therefore, the term for $i=4$ in the summation of Eq.~\eqref{eq:p:phi:log:mle} is $\frac{12}{45}\cdot\frac{2}{6}\approxeq0.09$ for $\varphi_{\alpha_1}$.
Extending the sum to all the traces in $\bar{L}$, we have that the value of $\widehat{P\!\left(\varphi_{\alpha_1}(L)\right)}$ is approximately $0.274$.

To lift the estimators of intersection and  
conditional probabilities from traces to logs,
we can again apply the law of total probability and derive the following.
\begin{theorem}\label{MLE:intersection}
	The MLE of $P\!\left(\varphi_1(L) \cap \varphi_2(L)\right)$ for a log $L$ with a trace set $\bar{L} = \{ t_i = \langle e_{i,1}, \ldots, e_{i,n_i} \rangle\}_{i=1}^m$ is
	\BeginSmEq \begin{equation}
		\widehat{P\!\left(\varphi_1(L) \cap \varphi_2(L)\right)}= \frac{1}{|L|}\sum_{i=1}^m \frac{j_i}{n_i} \sum_{k=1}^{n_i} \Lambda\!\left(e_{i,k}, \varphi_1\right) \Lambda\!\left(e_{i,k}, \varphi_2\right).\label{eq:joint}
	\end{equation}
	\EndSmEq 
\end{theorem} \begin{proof}\itshape
   From the law of total probability, we get that
   \footnotesize
   \begin{align*}
		P\!\left(\varphi_1(L) \cap \varphi_2(L)\right) &= \sum_{i=1}^m P(T=t_i)
        P\!\left(\varphi_1(T_i) \cap \varphi_2(T_i) \mid T_i = t_i\right) = \\
            &=\sum_{i=1}^m P(T=t_i) P\!\left(\varphi_1(t_i) \cap \varphi_2(t_i)\right),
            \end{align*}
            \EndSmEq
    which leads to the following estimate: \BeginSmEq
    \begin{align} \widehat{P\!\left(\varphi_1(L) \cap \varphi_2(L)\right)} & = \sum_{i=1}^m \widehat{P(T=t_i)} \widehat{P\!\left(\varphi_1(t_i) \cap \varphi_2(t_i)\right)} = \nonumber \\
             &=\frac{1}{|L|}\sum_{i=1}^m \frac{j_i}{n_i} \sum_{k=1}^{n_i} \Lambda\!\left(e_{i,k}, \varphi_1\right) \Lambda\!\left(e_{i,k}, \varphi_2\right).\label{eq:joint2}
	\end{align} \EndSmEq To obtain the result in Eq.~\eqref{eq:joint2} we estimate $P\!\left(T=t_i\right)$ as in~\cref{MLE:bayes1} (Multinomial MLE) and  
$P\!\left(\varphi_1(t) \cap \varphi_2(t)\right)$ as in~\cref{prop:joint} (bivariate Bernoulli).
Since both are MLEs, we get that $\widehat{P\!\left(\varphi_1(L) \cap \varphi_2(L)\right)}$ is an MLE~\cite[Ch.~7]{casella2021statistical}. 
\end{proof}

In \cref{sec:trace-estimators}, e.g., we showed that
$\widehat{P\!\left(\varphi_{\alpha_1}(t_4) \cap \varphi_{\tau_1}(t_4)\right)}$ is $\frac{1}{6}$
considering the example log, $\varphi_{\alpha_1}$ as above, and $\varphi_{\tau_1}=\RNFpast\taska$.
Recalling that $\widehat{P\!\left(t=t_4\right)}$ is $\frac{12}{45}$, we have that the term of the summation in Eq.~\eqref{eq:joint} for $t=t_4$ is $\frac{12}{45}\cdot\frac{1}{6}\approxeq0.04$.
The value of $\widehat{P\!\left(\varphi_{\alpha_1}(L) \cap \varphi_{\tau_1}(L)\right)}$ is computed by summing up the elements obtained for every $t\in\bar{L}$, thus obtaining $0.219$.

Lastly, we show an estimator for the conditional distribution 
$P\!\left(\varphi_1(L) \mid \varphi_2(L)\right)$ -- similarly, one can
derive estimators for the other conditional probabilities as for traces.
\begin{theorem}\label{MLE:conditional}
	The MLE for $P\!\left(\varphi_1(L) \mid \varphi_2(L)\right)$ 
	for a log $L$ with a trace set $\bar{L} = \{ t_i = \langle e_{i,1}, \ldots, e_{i,n_i} \rangle\}_{i=1}^m$  
	is \BeginSmEq \begin{align}\label{eq:log_conditional}
    	\widehat{P\!\left(\varphi_1(L) \ | \ \varphi_2(L)\right)} = \frac{\sum_{i=1}^m \frac{j_i}{n_i} \sum_{k=1}^{n_i} \Lambda\!\left(e_{i,k}, \varphi_1\right) \cdot \Lambda\!\left(e_{i,k}, \varphi_2\right)}{\sum_{i=1}^{m} \frac{j_i}{n_i} \sum_{k=1}^{n_i} \Lambda\!\left(e_{i,k}, \varphi_2\right)}.
	\end{align} \EndSmEq
\end{theorem} \begin{proof}\itshape  
  	By definition of conditional probability we get that \begin{align*}
    	\widehat{P\!\left(\varphi_1(L) \ | \ \varphi_2(L)\right)} = \frac{\widehat{P\!\left(\varphi_1(L) \cap \varphi_2(L)\right)}}{\widehat{P\!\left(\varphi_2(L)\right)}},
	\end{align*}  
with $\widehat{P\!\left(\varphi_1(L) \cap \varphi_2(L)\right)}$ 
estimated as in \cref{MLE:intersection}, and $\widehat{P\!\left(\varphi_2(L)\right)}$ as in \cref{MLE:bayes1}. Plugging the two expressions gets us to the estimator in Eq.~\eqref{eq:log_conditional}.
        We once more use the result in \cite[Chapter 7]{casella2021statistical} to prove that the estimator is an MLE.
 \end{proof} 
For instance, we showed above that for the example log $L$ in \cref{tab:running:example} the estimations
$\widehat{P\!\left(\varphi_{\alpha_1}(L) \cap \varphi_{\tau_1}(L)\right)}$ and
$\widehat{P\!\left(\varphi_{\alpha_1}(L)\right)}$ are $0.219$ and $0.274$, respectively.
By applying the computation above, we have that
$\widehat{P\!\left(\varphi_{\tau_1}(L) \ | \ \varphi_{\alpha_1}(L)\right)}$ is $\frac{0.219}{0.274}\approxeq0.80$.

\medskip
We conclude this section by highlighting that the computation of the estimators described thus far is tractable.
\begin{theorem}\label{thm:estimation_ltl}
	The estimators for \gls{ltlf} formulae being satisfied by a trace or a log, and the intersection and conditional probabilities thereof
	are computable in polynomial time.
\end{theorem}
\begin{proof}\itshape The proof relies on Theorem~\ref{thm:checking:trace} and Corollary~\ref{thm:checking:log}: once we have checked and labeled the traces, the computation of estimators requires only queries over the resulting labels, which can be performed in $O(|L| \times n)$ considering $n$ as the length of the longest trace in the log. \end{proof}

\section{Evaluation and measurement of specifications}
\label{sec:approach}

In the previous section, we estimated the probabilities of any \gls{ltlf} formula.
Yet, as the evaluation of an \gls{rf} differs from that of an \gls{ltlf} formula (see \cref{sec:rf}), 
the evaluation of a specification consisting of \glspl{rf} should take into account the interplay of activators and targets. The key point is in how to evaluate an \emph{intersection} of \glspl{rf} (i.e., a specification) on events.
The rationale is that once we have the evaluation of the specification on every event, we can estimate the probabilities and consequently its measures like for any other \gls{rf}.
In the remainder of this section, we start formalizing the semantics of \gls{rf} intersections in \cref{sec:specs}, continue proposing estimators of probabilities of traces and log satisfying such specifications (\cref{sec:logestimators}), describe the computation of interestingness measures (\cref{sec:computing:measures}), and finally highlight advantages and practical implications of our approach (\cref{sec:approach-discussion}).

\subsection{Evaluating specifications}
\label{sec:specs}
Formally, we define the semantics of a specification {\DeclaModel} as follows: given a trace $ t $ of length $ n $, an instant $ i $ with $ 1\leq i\leq n $, and a specification $ \DeclaModel \triangleq \{ \RfFormula_1, \ldots, \RfFormula_s \}$, with $s \in \mathbb{N}$ and $\RfFormula_j = \RfActivationFormulaVar{j} \RfImplication \RfObjectFormulaVar{j}$ for every $j$ s.t.\ $1 \leq j \leq s$, we say that 
\begin{compactitem}
	\item{\DeclaModel} is activated by $\EvtTrace$ in $\instant$\normalfont, i.e., $ (\EvtTrace, \instant) \models \DeclaModel_\RfActivation $, iff there exists a $ \RfFormula_j \in \DeclaModel $ s.t.\ $ (\EvtTrace, \instant) \vDash \RfActivationFormulaVar{j}$; \item{\DeclaModel} is satisfied by $\EvtTrace$ in $\instant$\normalfont, $ (\EvtTrace, \instant) \models \DeclaModel$, iff $ (\EvtTrace, \instant) \models \DeclaModel_\RfActivation $ and there does not exist any $ \RfFormula_j\in\DeclaModel$ s.t.\ $(t,i) \nvDash \RfFormula_j$; \item{\DeclaModel} is violated by $\EvtTrace$ in $\instant$\normalfont, $ (\EvtTrace, \instant) \nvDash \DeclaModel $, iff there exists a $ \RfFormula_j\in\DeclaModel$ s.t.\ $(t,i) \nvDash \RfFormula_j$;
\item{\DeclaModel} is unaffected by $\EvtTrace$ in $\instant$ iff $ (\EvtTrace, \instant) \nvDash \DeclaModel_\RfActivation $.
\end{compactitem}
In other words, {\DeclaModel}  is activated if at least one of its \glspl{rf} is activated, satisfied if all and only its activated \glspl{rf} are satisfied, violated if at least one activated \gls{rf} is violated, and unaffected if it is not activated.

In light of the above, we can express a specification $\DeclaModel = \{ \RfFormula_1, \ldots, \RfFormula_s \}$ as an \gls{rf}, $ \DeclaModel = \DeclaModel_\RfActivation \RfImplication \DeclaModel_\RfObject$, where $\DeclaModel_\RfActivation$ and $\DeclaModel_\RfObject$ are {\ltlf} formulae expressed as follows:
\begin{equation}
	\DeclaModel_\RfActivation =
	\bigvee\limits_{j=1}^{s}{\varphi_{\alpha_j}};\qquad
	\DeclaModel_\RfObject =
	\bigwedge\limits_{j=1}^{s}{\neg\left(\varphi_{\alpha_j} \wedge \neg\varphi_{\tau_j}\right)}.
\end{equation}
\noindent
For example, 
{\DeclaModel} from \cref{tab:running:example} is activated when either $\RfFormula_1$ or $\RfFormula_2$ are (i.e., $\DeclaModel_\RfActivation = \taskc \lor \taskd $) and it is satisfied when all the activated constraints are satisfied, i.e., $\DeclaModel_\RfObject = (\lnot\taskc\lor\RNFpast\taska)\land(\lnot\taskd\lor\RNFfut\taske)$. Hence, {\DeclaModel} is violated in $(\EvtTrace_3,1)$, e.g, because $\RfFormula_1$ is violated and satisfied in $(\EvtTrace_3,2)$, because $\RfFormula_2$ is satisfied and $\RfFormula_1$ is unaffected. 

The possibility to reduce a specification to a single \gls{rf} is key to defining the corresponding estimators and thereby computing the probability a trace and a log satisfy it.

\subsection{Estimators for Specifications}\label{sec:estimators}
In this part, we define what trace and log estimators
for single \gls{rf}s, as well as for \gls{rf} specifications. 
\subsubsection{Trace Estimators} We first start by estimating the interestingness degree of an \gls{rf}. 
Let $\Lambda_R$ be an \gls{rf} interpreter that takes an event and an \gls{rf} $\RfFormula = \RfActivationFormula \RfImplication \RfObjectFormula$ and returns a label $\Lambda_R(e,\RfFormula) \in \{0,\dnk,1\}$ corresponding to $\RfFormula$ being violated, unaffected, and satisfied by $e$, respectively. The labeling of an event given an \gls{rf} $\Psi$ resorts to $\Lambda$ (explained in \cref{sec:trace-estimators})
as follows:
\begin{equation}
\Lambda_R(e, \Psi) = \begin{cases}
	\begin{aligned}
		0, &  &  & \textrm{if} \ \ \Lambda\!\left(e, \varphi_\alpha\right)=1 \ \ \textrm{ and } \ \ \Lambda\!\left(e, \varphi_\tau\right)=0, \\
		1, &  &  & \textrm{if} \ \ \Lambda\!\left(e, \varphi_\alpha\right)=1 \ \ \textrm{ and } \ \ \Lambda\!\left(e, \varphi_\tau\right)=1,\\
		\dnk, &  &  & \textrm{otherwise}.
	\end{aligned}
\end{cases}\end{equation}
Notice that {\dnk} is a new outcome of the labeling function $\Lambda_R$ that solely applies to \glspl{rf} but not to {\ltlf} formulae (see the definition of $\Lambda$ in \cref{sec:trace-estimators}).
The second column of~\cref{tab:running:example} lists the labels assigned by $\Lambda_R$ to every event in the traces and all \glspl{rf} in a sequence. For example, the evaluation of $\Psi_1=\varphi_{\alpha_1}\RfImplication\varphi_{\tau_1}=\taskc\RfImplication\RNFpast\taska$ on $t_4=\langle e_{4,1},\ldots,e_{4,6} \rangle = \langle \taskb,\taskc,\taska,\taskc,\taske,\taska \rangle$ is such that
$(t_4,1) \nvDash \taskc$, thus $\Lambda_R(e_{4,1},\Psi_1)=\dnk$; also, we observe that
$(t_4,2) \vDash \taskc$ but 
$(t_4,2) \nvDash \RNFpast\taska$, therefore $\Lambda_R(e_{4,2},\Psi_1)=0$; finally, we have that
$(t_4,4) \vDash \taskc$ and 
$(t_4,4) \vDash \RNFpast\taska$, so $\Lambda_R(e_{4,4},\Psi_1)=1$.
Following this approach, we attain the following sequence of labels from $t_4$ via $\Lambda_R$ on $\Psi_1$: $\langle \dnk,0, \dnk, 1, \dnk, \dnk \rangle$.

We aim at estimating the probabilities of 
\emph{interesting} satisfaction and violation of a given \gls{rf} in a trace,
which correspond to cases in which the \gls{rf} is activated by it (see \cref{sec:rcon:syntax:semantics}).
Therefore, the corresponding random variable that describes
an `interesting' \gls{rf} (i.e., an \gls{rf} that is activated by a trace and satisfied in it), is $$X_{\varphi_\tau,t} \mid X_{\varphi_\alpha,t} = 1.$$
As before, the latter is a univariate Bernoulli random variable with success probability of 
$$P(\Psi(t)) = P\!\left(\varphi_\tau(t) \ | \ \varphi_\alpha(t)\right).$$
In other words, $P(\Psi(t))$ is the probability that an \gls{rf} $\Psi$ is interestingly satisfied by 
some trace $t$. We are now ready to state our first estimation result. 

\begin{prop} \label{prop:rf:intersection}
The MLE of $P(\Psi(t))$ for a given trace $t = \langle e_1, \ldots, e_n \rangle$  is 
\BeginSmEq \begin{equation}
\widehat{P(\Psi(t))} = \frac{\sum_{i=1}^n \Lambda\!\left(e_i, \varphi_\tau(t)\right)\cdot \Lambda\!\left(e_i, \varphi_\alpha(t)\right)}{\sum_{i=1}^n \Lambda\!\left(e_i, \varphi_\alpha(t)\right)}.
\end{equation}
\EndSmEq 
\end{prop} 
\begin{proof}\itshape We need to estimate the probability that an \gls{rf} is interestingly satisfied by $t$, 
\begin{equation*}
P(\Psi(t)) = P\!\left(\varphi_\tau(t) \ | \ \varphi_\alpha(t)\right).
\end{equation*} To this end, we can employ~\cref{prop:conditional} to compute the corresponding MLE: \BeginSmEq
\begin{equation*}
\widehat{P\!\left(\varphi_\tau(t) \ | \ \varphi_\alpha(t)\right)} = \frac{\sum_{i=1}^n \Lambda\!\left(e_i, \varphi_\tau(t)\right)\cdot \Lambda\!\left(e_i, \varphi_\alpha(t)\right)}{\sum_{i=1}^n \Lambda\!\left(e_i, \varphi_\alpha(t)\right)}.
\end{equation*} \EndSmEq 
\end{proof}  \noindent For example, $\widehat{P\!\left(\Psi_1(t_4)\right)}=\frac{1}{2}$, as shown in 
\cref{sec:trace-estimators}. Moreover, note that for $t_4$, $\Psi_2$ is never activated,
which leads to $\widehat{P\!\left(\Psi_2(t_4)\right)} = \textrm{NaN}$. To obtain an MLE for $\widehat{P\!\left(\lnot \Psi(t)\right)}$ we can write
\begin{equation}
\widehat{P\!\left(\lnot \Psi(t)\right)} = 1-\widehat{P(\Psi(t))},
\end{equation} due to the result that a function of an MLE is an MLE~\cite{casella2021statistical}. 
\todo{Remark 5.1 was removed}

At this stage, we turn to estimate the probabilities of interest for a specification $\mathcal{S}$ (i.e., a set of \glspl{rf}). 
Specifically, since a specification is interpreted similarly to 
a single \gls{rf}, once we obtained the interpretation for $\mathcal{S}_\RfActivation$ and $\mathcal{S}_\RfObject$, 
we apply the labeling mechanism $\Lambda_R$ again to obtain one of the three possible outcomes:
\begin{equation}
\Lambda_R(e, \mathcal{S}) = \begin{cases}
	\begin{aligned}
		0, &  &  & \textrm{if } \Lambda\!\left(e, \mathcal{S}_\RfActivation\right)=1  \textrm{ and }  \Lambda\!\left(e, \mathcal{S}_\RfObject\right)=0, \\
		1, &  &  & \textrm{if } \Lambda\!\left(e, \mathcal{S}_\RfActivation\right)=1  \textrm{ and }  \Lambda\!\left(e, \mathcal{S}_\RfObject\right)=1, \\
		\dnk, &  &  & \textrm{otherwise}.
	\end{aligned}
\end{cases}\end{equation} Note that $\varphi_\alpha$ and $\varphi_\tau$ are replaced by $\mathcal{S}_\RfActivation$ and $\mathcal{S}_\RfObject$, respectively, but the event-based labeling mechanism $\Lambda_R$ remains unchanged. 
This allows for the re-use of \cref{prop:rf:intersection}
to obtain estimators for interesting satisfaction and violation of specifications denoted by $P(\mathcal{S}(t))$ and $P(\neg\mathcal{S}(t))$, respectively. 

For estimating interesting satisfaction recall that $P(\mathcal{S}(t)) = P\!\left(\mathcal{S}_\tau(t) \ | \ \mathcal{S}_\alpha(t)\right)$. Thus, we get the following result. \begin{prop} \label{prop:specification:trace}
	The MLE of $P(\mathcal{S}(t))$ for a given trace $t = \langle e_1, \ldots, e_n \rangle$ is
	\BeginSmEq \begin{equation}
	\widehat{P(\mathcal{S}(t))} = \widehat{P\!\left(\mathcal{S}_\tau(t) \ | \ \mathcal{S}_\alpha(t)\right)} = \frac{\sum_{i=1}^n \Lambda\!\left(e_i,\mathcal{S}_\RfObject\right)\cdot \Lambda\!\left(e_i, \mathcal{S}_\RfActivation\right)}{\sum_{i=1}^n \Lambda\!\left(e_i, \mathcal{S}_\RfActivation\right)}.
	\end{equation} 
	\EndSmEq \end{prop} \begin{proof}\itshape Since $\mathcal{S}$ is an \gls{rf}, the MLE for the probability of interestingly satisfying it is given in~\cref{prop:rf:intersection}. We plug-in the activator and target of the \gls{rf}, 
$\mathcal{S}_\RfActivation$ and $\mathcal{S}_\RfObject$ instead of $\varphi_\alpha$ and $\varphi_\tau$, respectively, to complete the proof. \end{proof}
Notice that we use a similar notation for single \glspl{rf}, replacing $\Psi(t)$ with $\mathcal{S}(t)$ to denote satisfaction of a specification. For example, $\widehat{P\!\left(\DeclaModel(t_4)\right)}=\frac{1}{2}$, as we consider only $0$ or $1$ outcomes applying the computation described in~\cref{prop:specification:trace}.

To obtain an MLE for $\widehat{P(\lnot \mathcal{S}(t))}$ we can write
\begin{equation}
\widehat{P(\lnot \mathcal{S}(t))} = 1-\widehat{P(\mathcal{S}(t))},
\end{equation} due to the result that a function of an MLE is an MLE~\cite{casella2021statistical}.

\subsubsection{Log Estimators} \label{sec:logestimators}
To derive log estimators we again assume that 
$L$ is complete, i.e., $\bar{L} = \mathcal{T}$ with $\mathcal{T}$ being the support of the probability distribution of $T$. 
In what follows, we 
provide a result for a specification of \glspl{rf} (including the special case that the specification is a single \gls{rf}).
Let $P(\mathcal{S}(L))$ denote the probability of a log $L$ to interestingly satisfy a specification $\mathcal{S}$.
\begin{theorem} \label{prop:rf:intersection:log}
 The MLE of $P(\mathcal{S}(L))$ for a log $L$ with a trace set $\bar{L} = \{ t_i = \langle e_{i,1}, \ldots, e_{i,n_i} \rangle\}_{i=1}^m$  
	is given by \BeginSmEq
	\begin{align}\label{eq:log_conditional_spec}
    	\widehat{P(\mathcal{S}(L))} & = \widehat{P\!\left(\mathcal{S}_\tau(L) \ | \ \mathcal{S}_\alpha(L)\right)} = \\ \nonumber
    	                            & = \frac{\sum_{i=1}^m \frac{j_i}{n_i} \sum_{k=1}^{n_i} \Lambda\!\left(e_{i,k}, \mathcal{S}_\RfObject\right) \cdot \Lambda\!\left(e_{i,k}, \mathcal{S}_\RfActivation\right)}{\sum_{i=1}^{m} \frac{j_i}{n_i} \sum_{k=1}^{n_i} \Lambda\!\left(e_{i,k}, \mathcal{S}_\RfActivation\right)}.
	\end{align} \EndSmEq
\end{theorem} \begin{proof}\itshape The proof follows from~\cref{MLE:conditional}. 
Specifically, we replace $\varphi_\alpha$ and $\varphi_\tau$ in Eq.~\eqref{eq:log_conditional} with $\mathcal{S}_\RfActivation$ and $\mathcal{S}_\RfObject$, respectively.
\end{proof}

\noindent
Returning to the running example in \cref{tab:running:example}, we have that $\widehat{P(\mathcal{S}(L))}$ is $\frac{(17\cdot0.44)+(6\cdot0.33)+(5\cdot0.40)+(12\cdot0.17)+(5\cdot 0.0)}{(17\cdot0.44)+(6\cdot0.44)+(5\cdot0.50)+(12\cdot0.33)+(5\cdot 0.0)} \approxeq 0.81$.
To obtain an MLE for $\widehat{P(\lnot \mathcal{S}(L))}$ we can write
\begin{equation}
\widehat{P(\lnot \mathcal{S}(L))} = 1-\widehat{P(\mathcal{S}(L))},
\end{equation} due to the result that a function of an MLE is an MLE~\cite{casella2021statistical}. 
Notice that to obtain an estimator for a single \gls{rf} with activator $\varphi_\alpha$ and target $\varphi_\tau$, we replace $\mathcal{S}_\RfActivation$ and $\mathcal{S}_\RfObject$ with $\varphi_\alpha$ and $\varphi_\tau$, respectively.

Similarly to \cref{sec:estimation}, we conclude by providing a result that states the computational complexity of our technique. 
\begin{theorem}
The estimation of the interestingness measures over specifications of \gls{rf} given an event log is polynomial.
\end{theorem} \begin{proof}\itshape The proof relies on a similar argument to the proof of Theorem~\ref{thm:estimation_ltl}: each of the estimators is a  query over the labeled sequences that can be computed in $O\left(|L| \times n\right)$, considering $n$ as the length of the longest trace in $L$.\end{proof}

\subsection{Computing Measures of Interestingness for Specifications}\label{sec:computing:measures}
\todo{extended the discussion regarding the interestingness measures and their meaning in process traces and logs for specifications.}
Having defined the estimators, we can now quantify the interestingness of rule-based \gls{ltlf} process specifications relying on association rule mining measures. 
Specifically, the estimates that we derive for specifications (i.e.,~\cref{prop:specification:trace,prop:rf:intersection:log}) allow us to compute a plethora of interestingness measures
while considering the joint effects of multiple rules at once.\todo{removed reference to strict boolean conjunction}
Thereby, we advance the state of the art as measures were previously restricted to a single-rule scope~\cite{Cecconi2021Measuring}.

\begin{table}[tbp] 
\begin{tabular}{ccc}

\toprule
\textbf{Measure}     & \textbf{Trace} &  \textbf{Log}  \\
\midrule
Support     & $ P\!\left(\DeclaModel_\RfActivation \cap \DeclaModel_\RfObject,t\right) $ & $ P\!\left(\DeclaModel_\RfActivation \cap \DeclaModel_\RfObject,L\right) $\\
Confidence  & $ P\!\left(\DeclaModel_\RfObject|\DeclaModel_\RfActivation, t\right) $ & $ P\!\left(\DeclaModel_\RfObject|\DeclaModel_\RfActivation, L\right) $\\
Recall      & $ P\!\left(\DeclaModel_\RfActivation|\DeclaModel_\RfObject, t\right) $ & $ P\!\left(\DeclaModel_\RfActivation|\DeclaModel_\RfObject, L\right) $\\
Specificity & $ P\!\left(\lnot\DeclaModel_\RfObject|\lnot\DeclaModel_\RfActivation, t\right) $ & $ P\!\left(\lnot\DeclaModel_\RfObject|\lnot\DeclaModel_\RfActivation, L\right) $\\
Lift        & $ \dfrac{P\!\left(\DeclaModel_\RfActivation \cap \DeclaModel_\RfObject,t\right)}{P\!\left(\DeclaModel_\RfActivation, t\right)P\!\left(\DeclaModel_\RfObject, t\right)} $ & $ \dfrac{P\!\left(\DeclaModel_\RfActivation \cap \DeclaModel_\RfObject,L\right)}{P\!\left(\DeclaModel_\RfActivation, L\right)P\!\left(\DeclaModel_\RfObject, L\right)} $\\
\bottomrule
\end{tabular} \caption[Selection of interestingness measures for specifications of RCons]{Interestingness measures from \cref{tab:measures-def} defined for specifications}
		\label{tab:measures}
\end{table}

\Cref{tab:measures} shows the measures presented in~\cref{tab:measures-def} computed for a specification $\mathcal{S}$.
In the following, we describe what these measures intuitively mean given a specification and an event log. As previously described, the event log serves as a means to estimate probabilities, which in turn are the building block of the interestingness measures.
We define these measures on two levels.
At a trace-level, they gauge interestingness of a specification based on the number of events that satisfy specific conditions.
At a log-level, they take into account events \emph{and} traces, which we shall collectively name as ``fraction of the log'' (or log fraction) for simplicity.
Without loss of generality, we will describe the measures at a log-level. As it can be seen in \cref{tab:measures}, the same concepts seamlessly apply by restricting the notion of ``log fraction'' to that of ``trace fraction'' as the part of events in a single trace for the trace-level, due to the regularity of the definitions.

To determine the interestingness of a specification, Support considers the fraction of the log satisfying both its activator and its target. Confidence assesses in how far the target is satisfied within the log fraction that satisfies the activator. Dually, Recall reports on the events and traces that satisfy the activator among those that belong to the log fraction satisfying the target. Specificity quantifies the violation of the target within the fraction of the log that violates the activator.
Finally, Lift scales Support by the product of the estimated probabilities of activator and target. 

In the running example (see \cref{tab:running:example}), we use~\cref{prop:specification:trace,prop:rf:intersection:log}
to compute the measures that appear in the last five columns (Support, Confidence, etc.) for $\mathcal{S} = \{ \Psi_1, \Psi_2\}$. Observing the 
result in the last line of~\cref{tab:running:example}, e.g., we have that $\widehat{P(\mathcal{S}(L))}$ is $0.81$, which together
with the probabilities of activator and target of $\mathcal{S}$ 
yield a Support of \num{0.3} and Confidence of \num{0.81}.

The measures used here are only an explanatory subset of the available ones. Thanks to the estimators presented throughout \cref{sec:approach}, it is possible to seamlessly compute all the ones presented in \cite{Cecconi2021Measuring} and define new ones, based on the probabilities of activators and targets.

\subsection{Discussion}
\label{sec:approach-discussion}
\todo{added sub-section to discuss, out of probability details, the outcome of our proposal}

Having laid down the theoretical foundations, 
we turn to discussing the proposed approach from a practical perspective.
Intuitively, the probability estimators of an entire specification can be viewed as a relaxed metric of satisfaction, i.e., the extent to which the specification is satisfied in the given trace or log. It is a relaxed notion of satisfaction because it does not employ a strict boolean evaluation for the entire trace (as commonly done when evaluating \gls{ltlf} formulae). The logical evaluation is performed at the level of single events, while for traces and logs we use probability theory leading to relaxed statistical estimators. These estimators provide details about the severity of the observed violations 
in traces and logs. 

Furthermore, log estimators are not biased by the length of the traces, but only by their frequency. Consider the following example: given a log of two traces, one composed of \num{100} events and the other of \num{10} events, suppose that every event in the first trace satisfies the specification, whereas every event of the second one violates it. Thus, \num{100} out of \num{110} events, namely about the \SI{91}{\percent} of them, satisfy the specification. However, only \SI{50}{\percent} of the traces satisfy it. 
A trace represents the execution of a process instance. Indeed, regardless of the number of steps required, the main information is the overall outcome from that instance, e.g., to which extent a specification holds true in that trace.
This aspect is reflected in our log estimators: Given their definition in \cref{sec:logestimators}, it can be intuitively observed that the information of the trace used for the log aggregation is a fraction of events in the trace (i.e., those that satisfy a formula) and not their total number.  
Notice that the frequency of the trace in the log plays a major role, as it signals the probability of running the process in the way the trace reports.

The combination of rules in a specification is also a pivotal point of the contribution of our approach.
Recall that a specification is not simply a container of independent rules, but a unique system composed of their simultaneous interactions.
The classical approach to reason about specifications is to build a unique \gls{ltlf} formula formed by the logical conjunction of all the rules~\cite{DBLP:conf/ecai/Li0PVH14}. The disadvantage of such an approach is that the specification may be considered as violated by an entire trace for the violation of a rule by a single event. 
Since the specification is a unique system, it is agreeable that a violated rule implies the violation of the entire \glspl{rf} specification, yet the finer granular level of the analysis confines this violation to a single event, and not to the entire trace. Therefore, an estimator for specifications should score low values in traces only if multiple events violate the rule.
Extending the focus on event logs, the estimator should assign low values if numerous events in considerably many traces violate it. 

\begin{table}[tbp] 
\begin{tabular}{clcc}
\toprule
\textbf{Log} & 
\textbf{Evaluation} & 
\makecell[c]{\textbf{\acrshort{rf}} \\ \textbf{/ Specification}} & 
\makecell[c]{\textbf{$ P $ of \acrshort{rf}} \\ / \textbf{$ P $ of \DeclaModel}}\\ 
\midrule

$t_1 = \langle \taska,\taskb,\taskc,\taskd,\taske,\taskf  \rangle$ & 
\makecell[l]{
	$\langle$0, 1, 1, 1, 1, 1$\rangle$\\
	$\langle$1, 0, 1, 1, 1, 1$\rangle$\\
	$\langle$1, 1, 0, 1, 1, 1$\rangle$\\
	$\langle$1, 1, 1, 0, 1, 1$\rangle$\\
	$\langle$1, 1, 1, 1, 0, 1$\rangle$\\
	$\langle$1, 1, 1, 1, 1, 0$\rangle$\\
	$\langle$0, 0, 0, 0, 0, 0$\rangle$
} & 
\makecell[c]{
	~$ \RfFormula_1 = \RNFtrue\RfImplication\lnot\taska $\\
	~$ \RfFormula_2 = \RNFtrue\RfImplication\lnot\taskb $\\
	~$ \RfFormula_3 = \RNFtrue\RfImplication\lnot\taskc $\\
	~$ \RfFormula_4 = \RNFtrue\RfImplication\lnot\taskd $\\
	~$ \RfFormula_5 = \RNFtrue\RfImplication\lnot\taske $\\
	~$ \RfFormula_6 = \RNFtrue\RfImplication\lnot\taskf $\\
	~$\DeclaModel=\{\RfFormula_1,\dots,\RfFormula_6\}$
} &
\makecell[c]{
	~0.83\\
	~0.83\\
	~0.83\\
	~0.83\\
	~0.83\\
	~0.83\\
	~0.00
}  \\[1.2em]
\midrule

$t_2 = \langle \taska,\tasky,\tasky,\tasky,\taske,\taskf  \rangle$ & 
\makecell[l]{
	$\langle$0, 1, 1, 1, 1, 1$\rangle$\\
	$\langle$1, 1, 1, 1, 1, 1$\rangle$\\
	$\langle$1, 1, 1, 1, 1, 1$\rangle$\\
	$\langle$1, 1, 1, 1, 1, 1$\rangle$\\
	$\langle$1, 1, 1, 1, 0, 1$\rangle$\\
	$\langle$1, 1, 1, 1, 1, 0$\rangle$\\
	$\langle$0, 1, 1, 1, 0, 0$\rangle$
} & 
\makecell[c]{
	~$ \RfFormula_1 = \RNFtrue\RfImplication\lnot\taska $\\
	~$ \RfFormula_2 = \RNFtrue\RfImplication\lnot\taskb $\\
	~$ \RfFormula_3 = \RNFtrue\RfImplication\lnot\taskc $\\
	~$ \RfFormula_4 = \RNFtrue\RfImplication\lnot\taskd $\\
	~$ \RfFormula_5 = \RNFtrue\RfImplication\lnot\taske $\\
	~$ \RfFormula_6 = \RNFtrue\RfImplication\lnot\taskf $\\
	~$\DeclaModel=\{\RfFormula_1,\dots,\RfFormula_6\}$
} &
\makecell[c]{
	~0.83\\
	~1.00\\
	~1.00\\
	~1.00\\
	~0.83\\
	~0.83\\
	~0.50
}  \\ [1.2em]

\midrule

\makecell[c]{
	~$L = \{t_1^{5}, t_2^{10}\}$\\[1ex]~$ |L|=15 $
	}   &   
&  
\makecell[c]{
	\makecell[c]{
	~$ \RfFormula_1 = \RNFtrue\RfImplication\lnot\taska $\\
	~$ \RfFormula_2 = \RNFtrue\RfImplication\lnot\taskb $\\
	~$ \RfFormula_3 = \RNFtrue\RfImplication\lnot\taskc $\\
	~$ \RfFormula_4 = \RNFtrue\RfImplication\lnot\taskd $\\
	~$ \RfFormula_5 = \RNFtrue\RfImplication\lnot\taske $\\
	~$ \RfFormula_6 = \RNFtrue\RfImplication\lnot\taskf $\\
	~$\DeclaModel=\{\RfFormula_1,\dots,\RfFormula_6\}$
}
} &  
\makecell[c]{
	~0.83\\
	~0.94\\
	~0.94\\
	~0.94\\
	~0.83\\
	~0.83\\
	~0.33
}  \\[1.2em]

\bottomrule
\end{tabular}

 \caption[Example of a specification violated by high probability constraints.]{Example of a specification violated by high probability constraints.}
	\label{tab:validation:example:full:violation}
\end{table}

\Cref{tab:validation:example:full:violation} exemplifies the above reasoning. The specification $\DeclaModel=\{\RfFormula_1,\ldots,\RfFormula_6\}$ is composed only of constraints formulated as $\RNFtrue\RfImplication\lnot z$, i.e., every event should be different from a given event $z$ (with $z \in \{\taska,\ldots,\taskf\} $). 
In the first trace, $\EvtTrace_1$, every single constraint is individually violated only by one event of the trace and satisfied by the others. Accordingly, their estimators assign high values ($1.00$). Nonetheless, the specification is violated in every single event by at least one constraint, thus the estimator's score for {\DeclaModel} is $0.00$. 
Instead, every event of the second trace, $\EvtTrace_2$, satisfies $\RfFormula_2$, $\RfFormula_3$ and $\RfFormula_4$ (scoring $1.00$), while $\RfFormula_1$, $\RfFormula_5$ and $\RfFormula_6$ are violated by one event therein each (scoring $0.83$). In this case, then, the estimator assigns $0.50$ to the specification.
Since the numerosity of $t_1$ in the event log is \num{5} and the numerosity of $t_2$ is twice as much (\num{10}), the specification on the entire log is assigned $0.33$ by the estimator.

This is a desirable behavior: A specification is evaluated for the combination of its constraints, not its single parts. The first trace shows clearly that, despite the single constraints being mostly satisfied, together they do not properly capture the execution of the trace. 
Thus, the average of the measures would be a descriptive statistic for the single rules but a misleading indicator for the specification.
Exploring extensions of weighted evaluations of specifications where specific constraints may be more relevant than others and treated unequally is an interesting future outlook.

\section{Evaluation}
\label{sec:evaluation}
The goal of this section is to demonstrate how gauging the interestingness of an entire process specification with multiple measures leads to novel results that are not achievable via the analysis of typical, single-rule based measures.
To this end, we conduct three experiments.
In 
\cref{sec:evaluation-discovery-vs-specification-measurements}, we show how the measure of a specification differs from the assemblage of the measures of single rules; 
in \cref{sec:evaluation-absolute-match-different-measures}, we provide experimental evidence of the fact that distinct measures show different aspects of a specification, even in the case of full compliance with a Confidence level of \num{1.0};   
in \cref{sec:use-case-process-drift}, we propose a use-case application of our approach, analyzing its support to process drift detection.
Finally, \cref{sec:exp:discussion} concludes the section with a discussion of our findings.

We implemented our technique in a proof-of-concept Java tool, publicly available at \url{https://github.com/Oneiroe/Janus}.
The implementation natively supports a relevant set of rule templates based on~\cite{Dwyer1999Patterns}, but we already showed in \cref{sec:approach} that the technique is seamlessly applicable to any~\gls{rf}. Furthermore, inspired by~\cite{Geng2006Interestingness}, it supports the computation of \num{37} interestingness measures. 

\begin{table}[tbp] 
\begin{tabular}{lS[table-format = 6.0]S[table-format = 2.0]S[table-format = 6.0]}
\toprule
\textbf{Event log}	&
\textbf{Traces} &
\textbf{Tasks} &
\textbf{Events}\\ 

\midrule
BPIC12~\citep{Dongen/BPIC2012:log}& 13087 & 36 & 262200 
\\
BPIC13\_cp~\citep{Steeman/BPIC2013:log}& 1487 & 7 & 6660 
\\
BPIC14\_f~\citep{Dongen/BPIC2014:log}& 41353 & 9 & 369485 
\\
BPIC15\_1f~\citep{Dongen2015BPIC15}& 902 & 70 & 21656  
\\
RTFMP~\citep{LeoniMannhardt/RoadTrafficFine2015:log}& 150370 & 11 & 561470 
\\
Sepsis~\citep{Mannhardt2016Sepsis}& 1050 & 16 & 15214 
\\
Help-Desk~\citep{Polato2017helpdesk}& 4580 & 14 & 21348
\\

\bottomrule
\end{tabular}

 \caption[Details of the real-world event logs used for the evaluation]{Details of the real-world event logs used for the evaluation}
	\label{tab:logs}
\end{table}

In our experiments, we used a set of openly available, real-world event logs. The details and references to the datasets are reported in \cref{tab:logs}.
Note that our tool scales linearly with the size of the event log, the size of the specification, and the number of measures to compute. For example, for the Sepsis dataset the measurement took around \SI{35}{\second} with the heaviest setup (i.e., to compute \num{37} measures for a specification containing \num{3202} rules and a log with \num{1050} traces) on an Intel Core i5-7300U CPU at \SI{2.60}{\GHz}, quad-core, \SI{16}{\giga\byte} of RAM, running Ubuntu 18.04.6.

\subsection{Measuring single rules and entire specifications}
\label{sec:evaluation-discovery-vs-specification-measurements}
In the following, we assess the usefulness of specifications measurement on real-life data, applying our technique to the results of various pattern-based \gls{ltlf} specification miners, i.e., Janus~\cite{Cecconi2018Interestingness}, MINERful~\cite{DiCiccio2015OnTheDiscovery}, and Perracotta~\cite{Yang2006Perracotta}.
The goal is to highlight the importance of checking an entire specification, as opposed to the analysis of individual rules. The rationale is that while many specification miners use a threshold to retrieve only rules satisfied above that value, the corresponding specification may present a satisfaction degree below the desired level, following Morin's principle that the whole may be less than the sum of its parts~\cite{morin2013methode}. 

The interested reader can find the scripts and input files to reproduce the tests alongside output reports and the full collection of specification rules at \url{https://oneiroe.github.io/DeclarativeSpecificationMeasurements-Journal-static/}.
We conducted the experiment as follows. We discovered a specification from the log with each miner. Then, we used our tool to compute the interestingness measures on the log. Here we focus on Confidence, as all miners implement a custom calculation for it.
We repeated the discovery step with increasing Confidence thresholds, from \num{0} to \num{1}, with step size of \num{0.05}. Finally, we compared the measures of the specifications to the input threshold. 

The results can be found in~\cref{fig:evaluation-miners-trend-comparison-sepsis,fig:evaluation-miners-trend-comparison-bpic12,fig:evaluation-miners-trend-comparison-bpic13,fig:evaluation-miners-trend-comparison-bpic14,fig:evaluation-miners-trend-comparison-bpic15,fig:evaluation-miners-trend-comparison-rtfm,fig:evaluation-miners-trend-comparison-helpdesk}. As highlighted in~\cref{fig:evaluation-miners-trend-comparison-sepsis,fig:evaluation-miners-trend-comparison-bpic14,fig:evaluation-miners-trend-comparison-bpic15}, every miner returned specifications whose overall Confidence is lower than the initially set threshold, even for thresholds greater that \num{0.7}. This means that although every rule in a specification has a Confidence value greater or equal to the threshold, the overall specification performs worse than the user-defined minimum accepted level.
\begin{figure}
	\centering
	\includegraphics[width=1.0\linewidth]{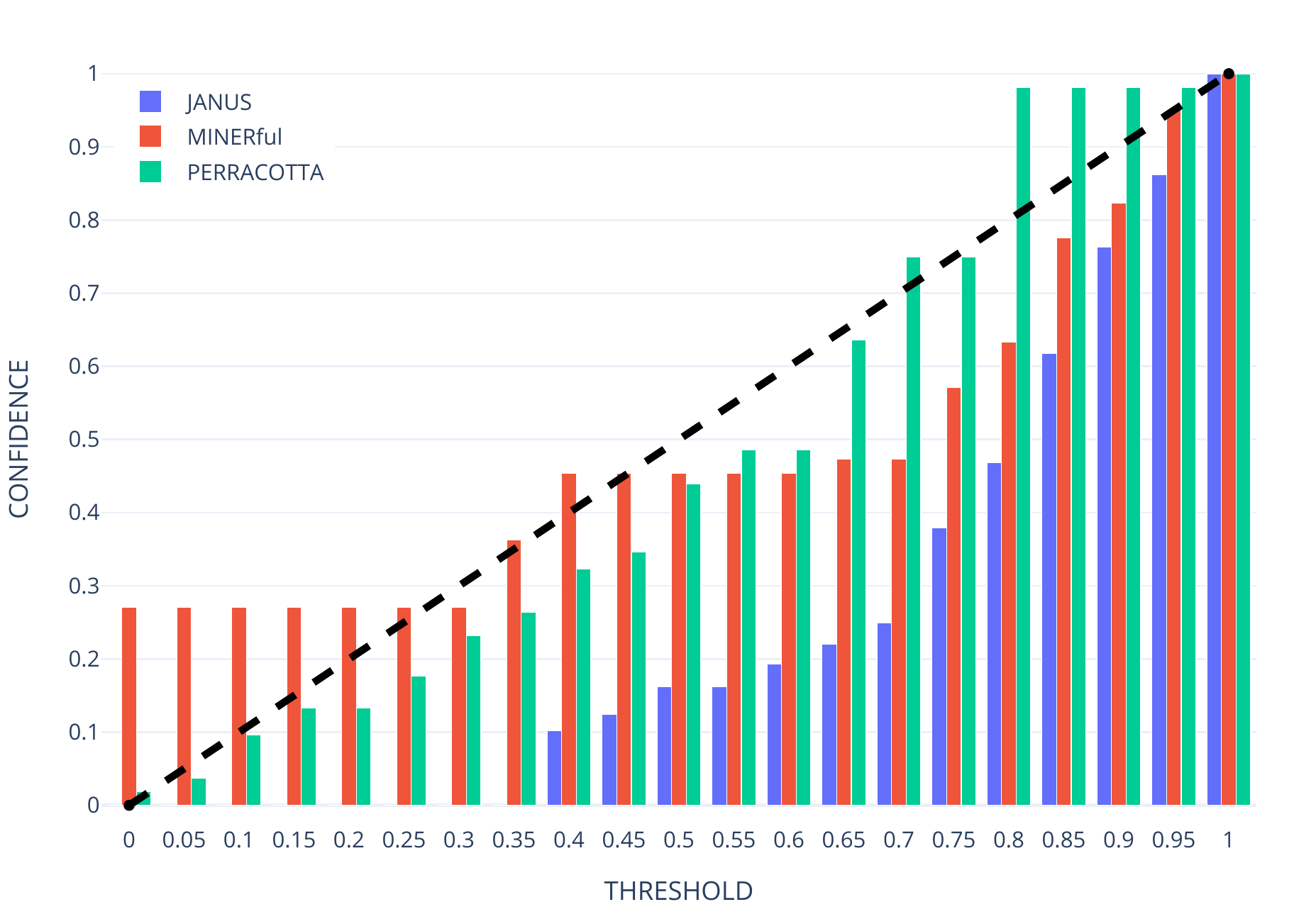}
	\caption[Discovered specification confidence: Sepsis]{Confidence of the mined specifications with respect to the threshold used for their discovery with the Sepsis dataset~\cite{Mannhardt2016Sepsis}.}
	\label{fig:evaluation-miners-trend-comparison-sepsis}
\end{figure}
\begin{figure}
	\centering
	\includegraphics[width=1.0\linewidth]{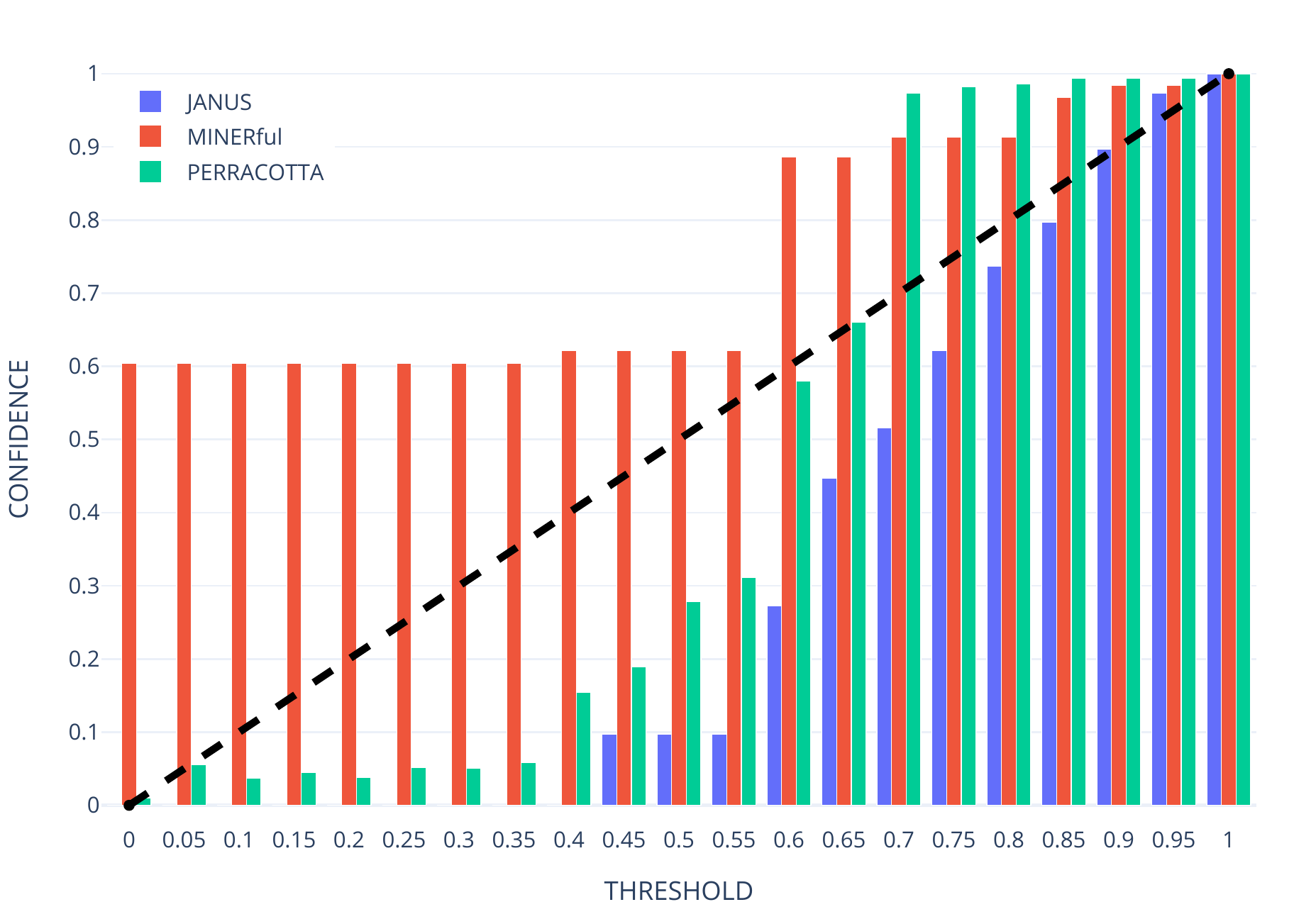}
	\caption[Discovered specification confidence: BPIC12]{Confidence of the mined specifications with respect to the threshold used for their discovery with the BPIC12 dataset~\cite{Dongen/BPIC2012:log}.}
	\label{fig:evaluation-miners-trend-comparison-bpic12}
\end{figure}
\begin{figure}
	\centering
	\includegraphics[width=1.0\linewidth]{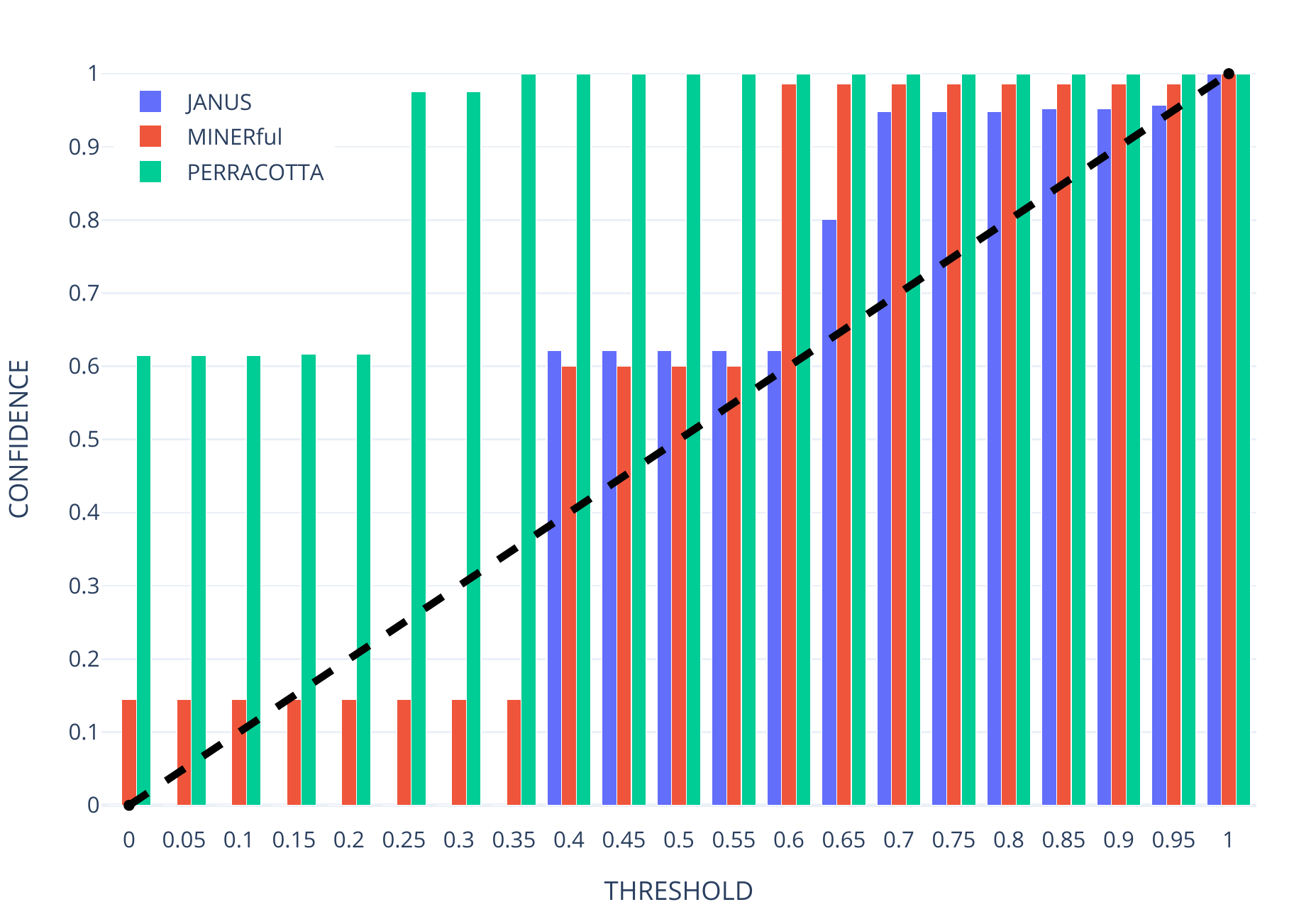}
	\caption[Discovered specification confidence: BPIC13]{Confidence of the mined specifications with respect to the threshold used for their discovery with the BPIC13 dataset~\cite{Steeman/BPIC2013:log}.}
	\label{fig:evaluation-miners-trend-comparison-bpic13}
\end{figure}
\begin{figure}
	\centering
	\includegraphics[width=1.0\linewidth]{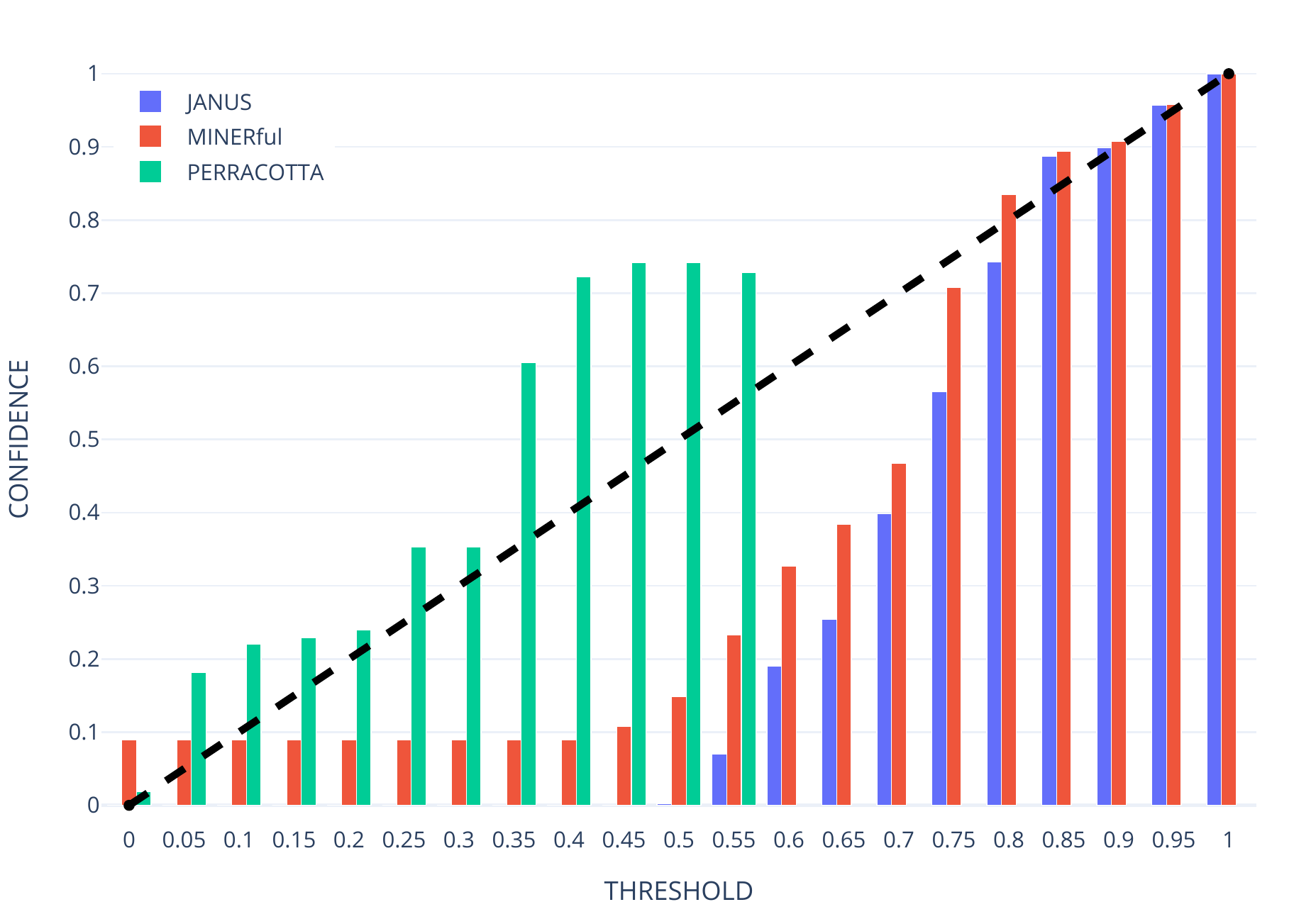}
	\caption[Discovered specification confidence: BPIC14\_f]{Confidence of the mined specifications with respect to the threshold used for their discovery with the BPIC14\_f dataset~\cite{Dongen/BPIC2014:log}. Note: Perracotta did not return any rules with thresholds above \num{0.55} with this dataset.}
	\label{fig:evaluation-miners-trend-comparison-bpic14}
\end{figure}
\begin{figure}
	\centering
	\includegraphics[width=1.0\linewidth]{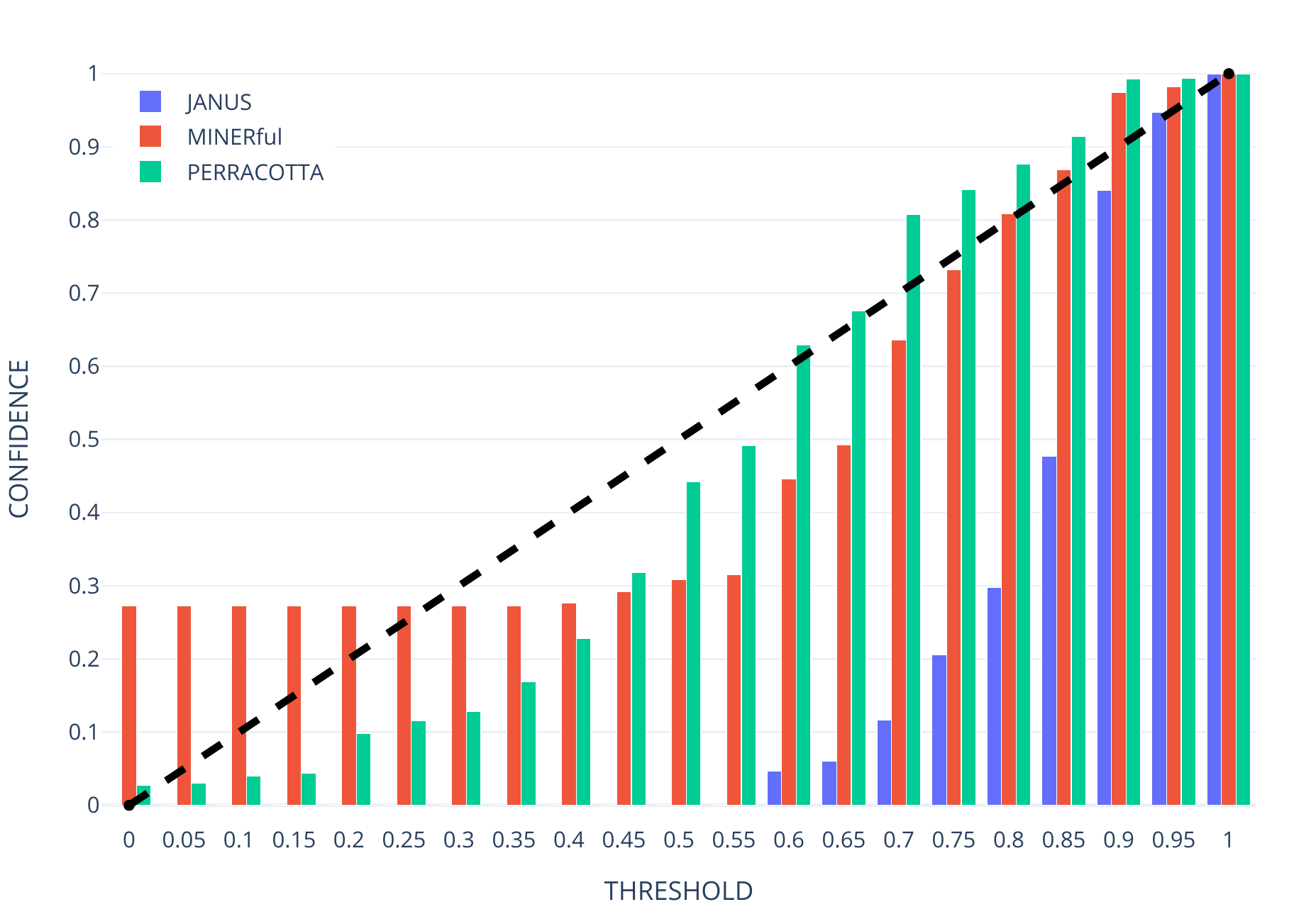}
	\caption[Discovered specification confidence: BPIC15\_1f]{Confidence of the mined specifications with respect to the threshold used for their discovery with the BPIC15\_1f dataset~\cite{Dongen2015BPIC15}.}
	\label{fig:evaluation-miners-trend-comparison-bpic15}
\end{figure}
\begin{figure}
	\centering
	\includegraphics[width=1.0\linewidth]{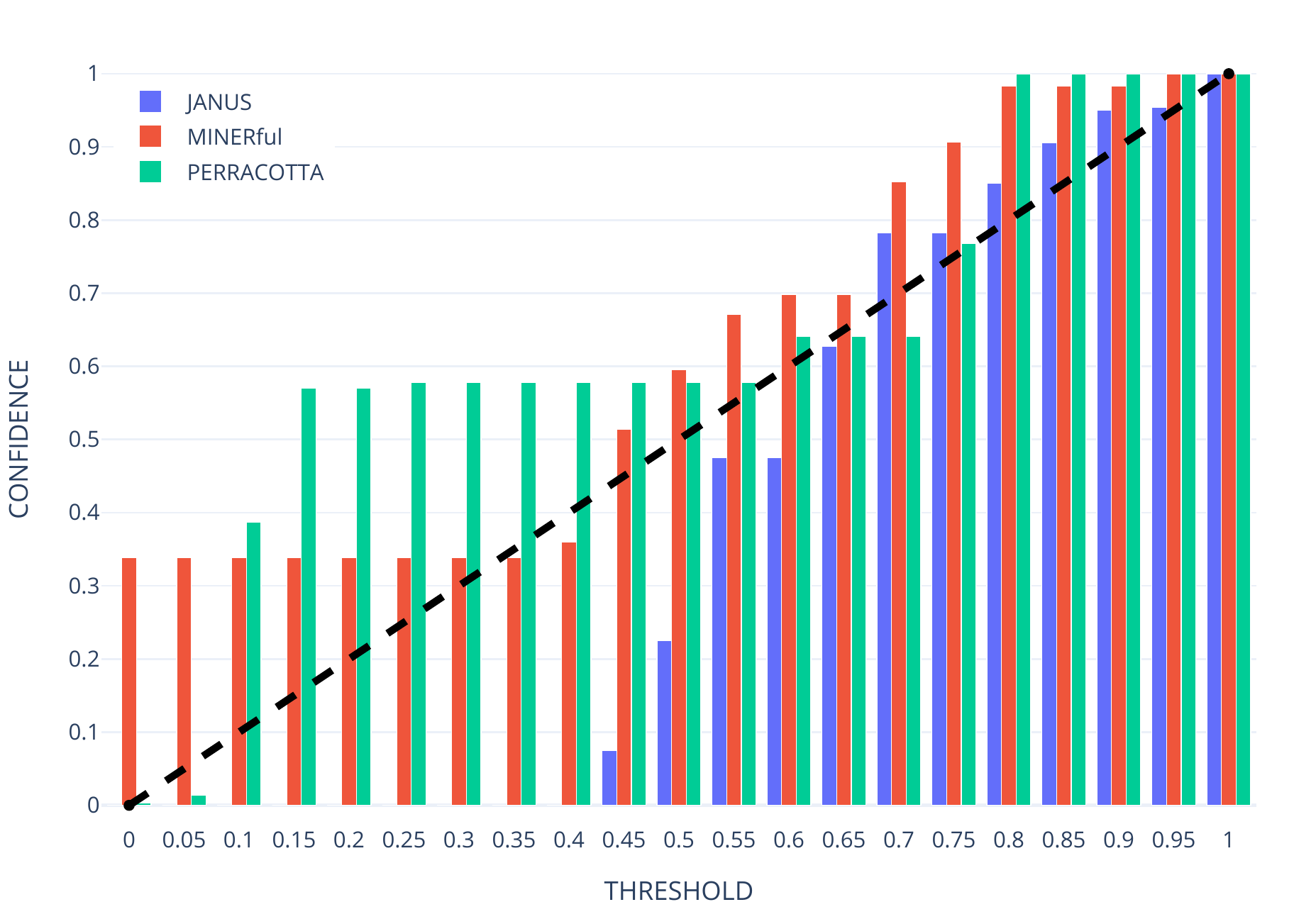}
	\caption[Discovered specification confidence: RTFMP]{Confidence of the mined specifications with respect to the threshold used for their discovery with the RTFMP dataset~\cite{LeoniMannhardt/RoadTrafficFine2015:log}.}
	\label{fig:evaluation-miners-trend-comparison-rtfm}
\end{figure}
\begin{figure}
	\centering
	\includegraphics[width=1.0\linewidth]{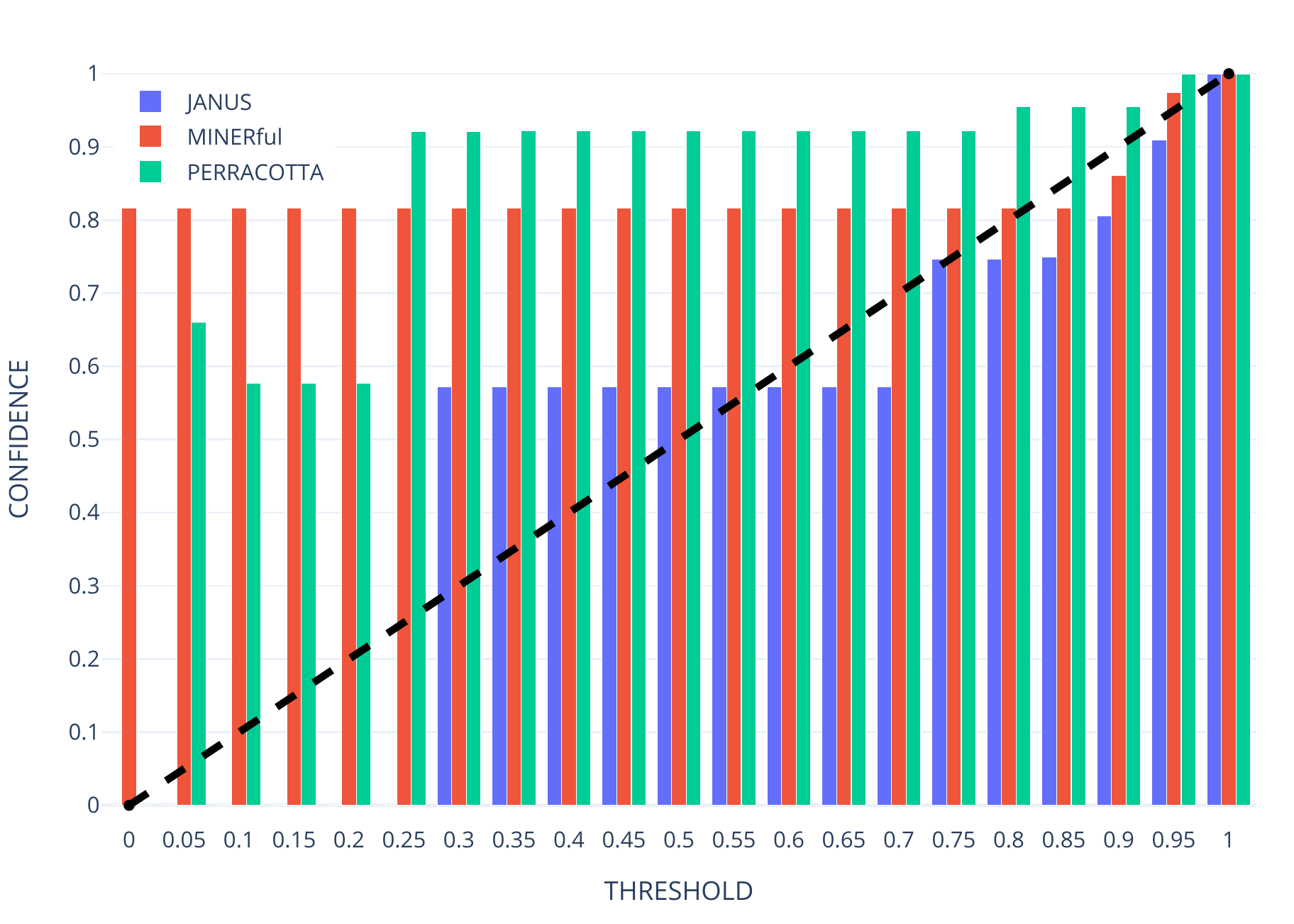}
	\caption[Discovered specification confidence: Help-Desk]{Confidence of the mined specifications with respect to the threshold used for their discovery for the Help-Desk dataset~\cite{Polato2017helpdesk}.}
	\label{fig:evaluation-miners-trend-comparison-helpdesk}
\end{figure}
Other datasets (see~\cref{fig:evaluation-miners-trend-comparison-bpic12,fig:evaluation-miners-trend-comparison-bpic13,fig:evaluation-miners-trend-comparison-rtfm,fig:evaluation-miners-trend-comparison-helpdesk}), exhibit a similar behavior for lower thresholds.

This issue may lead to sub-optimal results, akin to the multiple testing problem~\cite{Haemaelaeinen2019ATutorial} in statistical inference:
assuming the independence of each hypothesis, thus not considering their inter-dependency, leads to erroneous outcomes.
With our technique, this kind of issue becomes evident. Improving declarative process specification miners with the integration of our technique paves the path for interesting future endeavors.
However, we remark that this analysis focuses on the quantitative matching of the specification to the logs, regardless of the semantics of the discovered specifications.
Next, we broaden the perspective of interestingness from the sole Confidence to a larger set of measures.

\subsection{Differing measures}
\label{sec:evaluation-absolute-match-different-measures}
In this part, we show how different specifications, though never violated in the log, may exhibit different characteristics through the inspection of multiple measures.

In this experiment, we make use of the Sepsis event log~\cite{Mannhardt2016Sepsis}.
The dataset refers to a real-world process in the healthcare domain and contains records of patient visits with sepsis symptoms at a Dutch hospital. 
The dataset exhibits high variability: \num{75}\% of the traces are unique, which makes it a good candidate to evaluate partial specifications.
To retrieve multiple distinct specifications, we first mined a specification with the miner set to discard partly violated rules. Any miner would be fit for the task, thus without loss of generality we employed Janus with a Confidence threshold set to \num{1} for all rules. The resulting specification consisted of \num{238} rules, which we partitioned randomly into \num{5} subsets of \glspl{rf}.
Finally, we computed interestingness measures for each of the sub-specifications.
An excerpt of the results is reported in~\cref{tab:evaluation-conf1-different-measures}. \begin{table}[tb]
	\begin{adjustbox}{width=.9\textwidth,center=\textwidth}
		\begin{tabular}{ccccccc}
\toprule
\textbf{Measure}     & \textbf{Original $\mathcal{S}$} & \textbf{$\mathcal{S}_1$} & \textbf{$\mathcal{S}_2$} & \textbf{$\mathcal{S}_3$} & \textbf{$\mathcal{S}_4$} & \textbf{$\mathcal{S}_5$} \\ 
\midrule
	Confidence	&	1.000	&	1.000	&	1.000	&	1.000	&	1.000	&	1.000	\\
	Support		&	1.000	&	0.003	&	0.011	&	0.074	&	0.522	&	0.683	\\
	Recall		&	1.000	&	0.011	&	1.000	&	1.000	&	1.000	&	1.000	\\
	Specificity	&	NaN		&	0.741	&	1.000	&	1.000	&	1.000	&	1.000	\\
	Lift		&	1.000	&	3.831	&	89.691	&	13.512	&	1.916	&	1.465	\\
\bottomrule
\end{tabular}
 	\end{adjustbox}
	\caption[Measurements of sub-specifications]{Measurements of sub-specifications}
	\label{tab:evaluation-conf1-different-measures}
\end{table}

The results show that despite a Confidence level of \num{1} for every specification (i.e., there are no violated rules), other measures can still detect differences. \todo{Updated result discussion with the new numbers from the software update}
For example, 
Support shows that $ \DeclaModel_1 $,$ \DeclaModel_2 $, and $ \DeclaModel_3 $ are applied to a very small portion of events of the log, while $ \DeclaModel_4 $ and $ \DeclaModel_5 $ are activated by more than half of it. 
Specificity and Recall signal that for every specification the respective targets and activators always occur together, except for $ \DeclaModel_1 $: 
its target occurs without the activator the vast majority of the time.
The division by zero for the original $ \DeclaModel $ Specificity suggests that the specification was activated in every event of the log. 
Also, Lift shows that, in all the models, the joint probability of activator and target is higher than their individual one, yet with different strengths, e.g., $ \DeclaModel_2 $ towering any other one. 
The input specifications and output measurements are available for reproducibility purposes at \url{https://oneiroe.github.io/DeclarativeSpecificationMeasurements-Journal-static/}.

To conclude, this experiment demonstrates the advantages of the availability of a full set of measures, going beyond the mere satisfaction of specifications. Indeed, the choice of the most appropriate measure will depend on the specific application of use, as we will exemplify next.

\subsection{Impact of process drift on specification measures}
\label{sec:use-case-process-drift}

Throughout this section, we analyze the impact that drifts~\cite{DBLP:journals/tkde/MaaradjiDRO17} in process specifications may have on the various interestingness measures. Our objective is twofold.
On the one hand, we aim to undertake a preliminary assessment of the capability of measures to reflect modifications in the process behavior.
On the other hand, we want to observe the difference in sensitivity to changes exhibited by the measures originally suggested in the literature for association rule mining~\cite{Tan2004Selecting,Lenca2008OnSelecting} and adopted in the context of declarative process mining in~\cite{Cecconi2022Measurement}. Specifically, we perform the analysis for entire specifications and not only for individual rules.

We used the Visual Drift Detection (VDD) tool~\cite{DBLP:conf/icpm/YeshchenkoMCP20} as a benchmark to detect the points in which the process is subject to variations over time. VDD divides event logs into sub-logs of consecutive traces, measures the Confidence of all discoverable {\Declare} constraints in each sub-log, and analyzes the time sequences generated by the Confidence measurements while grouping together constraints that exhibit similar trends over time. It automatically detects change points within groups based on the oscillation of values in the sequence. Then, it automatically classifies those change points as either process drifts or evidence of erratic behavior and outliers~\cite{DBLP:conf/er/YeshchenkoCMP19a}. 

\subsubsection{Experimental setting}
Taking the results illustrated in~\cite{DBLP:conf/er/YeshchenkoCMP19a,DBLP:journals/tvcg/YeshchenkoCMP22} as a reference, we conducted our experiments with the following two real-world event logs:
\begin{iiilist}
	\item the aforementioned Sepsis log~\cite{Mannhardt2016Sepsis}, which reports on the tasks executed from the registration to the discharge of patients affected by sepsis, and
	\item the Help-Desk log~\cite{Polato2017helpdesk}, recording the activities carried out within a management process of the Help-desk of an Italian software company.
\end{iiilist}
The scripts we use for our tests can be found alongside the analysis results at the following address: 
\url{https://github.com/l2brb/Measurement-change-point-evaluation}.

To carry out our experiments, we went through the following steps.
We first mined a process specification consisting of rules with high Confidence (rarely violated whenever triggered) from the event logs taken as a whole. To this end, we used Janus\footnote{\url{https://github.com/Oneiroe/Janus/}}~\cite{Cecconi2018Interestingness} with the Confidence threshold set to \num{0.95}.
Then, we sliced the event log into consecutive sub-logs with a tumbling window approach, namely extracting sequences of \num{50} consecutive,  non-overlapping trace sets.
We resorted to the dedicated pre-processing feature of MINERful\footnote{\url{https://github.com/cdc08x/MINERful/}}~\cite{DiCiccio2015OnTheDiscovery} to this extent.
We fed the sub-logs into VDD for the detection of change points, setting the parameters as follows: 
\begin{iiilist}
	\item \num{50} for the window size, \item \num{50} for slide size, and \item \num{300} for the cut threshold. \end{iiilist}
Finally, we made our tool compute the value of the measures of the initial specification on every sub-log in a sequence.
Thereupon, we analyzed the trend of the specification measurements of the specification in correspondence with the change points detected by VDD, to observe if, and in how far, they exhibit variations.
In the following, we discuss the results of our analysis.

\subsubsection{Experimental results}
\Cref{fig:italian_sublog_confidence,fig:sepsis_sublog_confidence} illustrate the trends of Confidence (on the $y$-axis) for the mined specifications on the sub-logs of the Help-Desk and the Sepsis datasets, respectively.
The points on the $x$-axis represent the timestamp of the first event in every sub-log.
We add colored bars in correspondence with the change points detected by VDD.
As Yeshchenko et al.\ explain in~\cite{DBLP:conf/er/YeshchenkoCMP19a}, two drifts occur in the course of the process executions recorded in the Help-Desk log. They are located in the first and the last quarter of the diagram.

In \cref{fig:italian_sublog_confidence}, we highlight the corresponding points with a dark green bar.
The lighter bars refer to the erratic behavior of the process in correspondence with other changing points of lower impact.
We observe that all changing points correspond to upward or downward peaks in the poly-line plotting the Confidence values.
However, they are not the sole steep slopes that can be noticed in the diagram.
The reason is, VDD considers the trend of Confidence values for \emph{all} discoverable constraints.
In our case, instead, we focus on a selection of constraints, namely those that were discovered by Janus from the whole log setting a minimum threshold of \SI{95}{\percent} for Confidence.
Therefore, variations in the measure can be more visible considering the derived specification, while the oscillation may be mitigated by taking into account other constraints that are violated more often. Also, notice that VDD takes every constraint in isolation, while our holistic approach aims at measuring the interestingness of a specification as a whole.

\begin{figure*}[bt]
	\centering
	\includegraphics[width=0.66\linewidth]{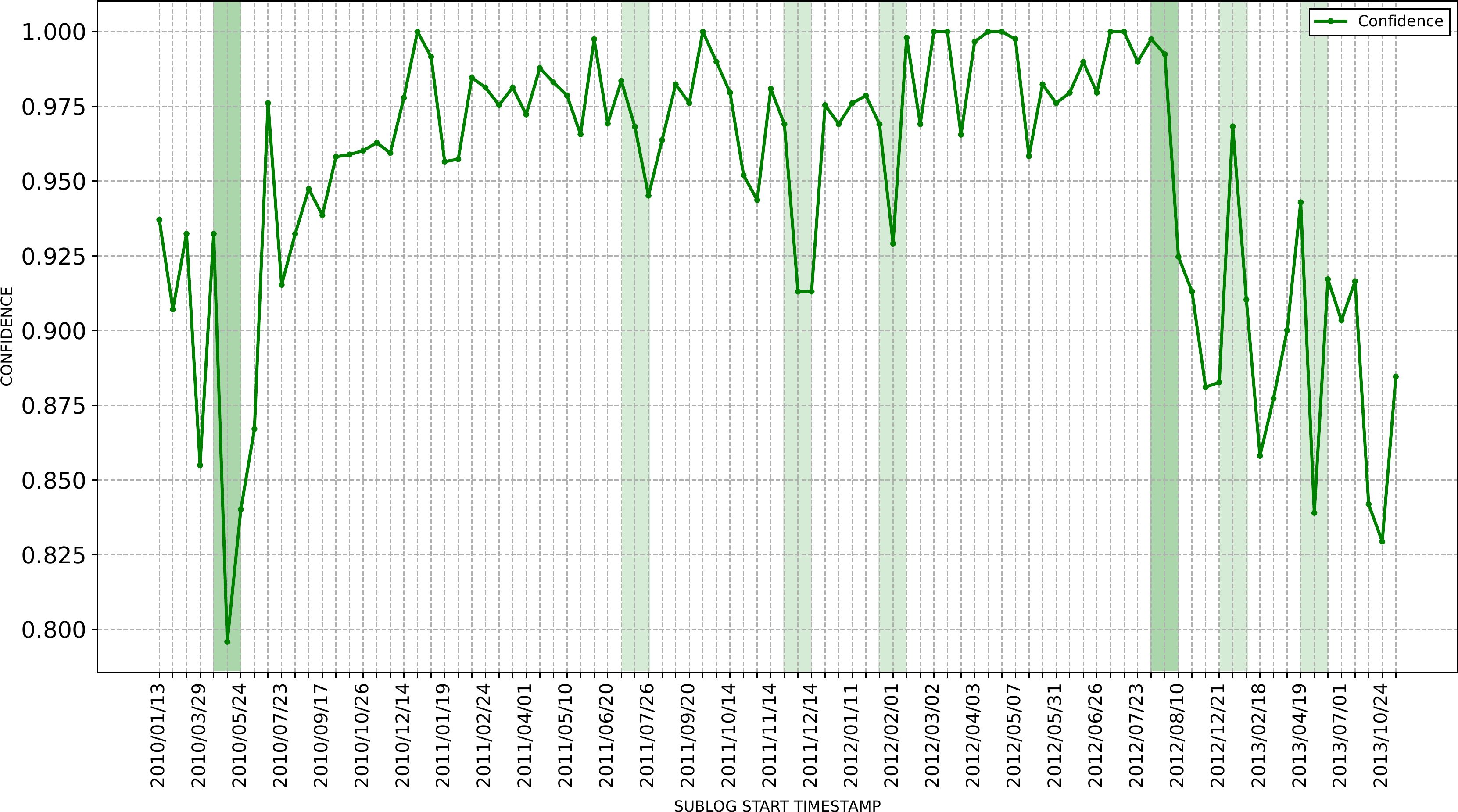}
	\caption[Help-Desk sub-log confidence values]{Help-Desk sub-log confidence values. The plot is visually edited with the addition of vertical bands in the areas where drifts (darker) and erratic behavior (lighter) are detected by VDD as in~\cite{DBLP:conf/er/YeshchenkoCMP19a}.}
	\label{fig:italian_sublog_confidence}
\end{figure*}
\begin{figure*}[tb]
	\centering
	\includegraphics[width=0.66\linewidth]{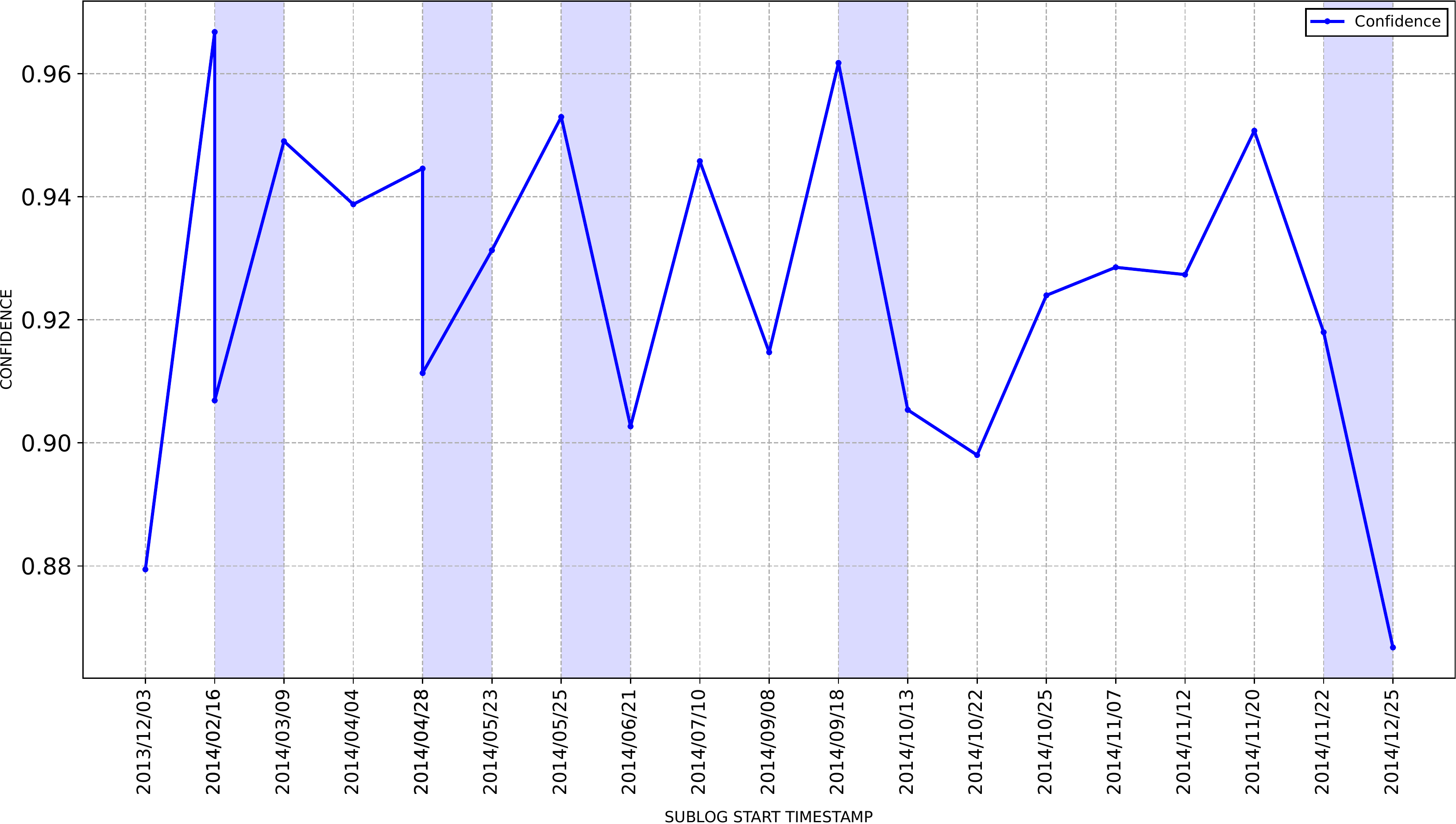}
	\caption[Sepsis sub-log confidence values]{Sepsis sub-log confidence values. The plot is visually edited with the addition of vertical bands in the areas where erratic behavior is detected by VDD as in~\cite{DBLP:journals/tvcg/YeshchenkoCMP22}.}
	\label{fig:sepsis_sublog_confidence}
\end{figure*}

Similar results are visible in \cref{fig:sepsis_sublog_confidence}. The bars in the diagram highlight the change points identified by VDD. As discussed in~\cite{DBLP:journals/tvcg/YeshchenkoCMP22}, none of the detected change points corresponds to actual process drifts.
Instead, these oscillations are evidence of the erratic behavior characterizing the event log, most likely due to the reportedly flexible nature of the healthcare process involved~\cite{Mannhardt2016Sepsis}; moreover, the latest oscillation is an outlier that occurs towards the end of the event log.
We observe that this characteristic is also evidenced by the fact that seasonal oscillations characterize the entire diagram, though their amplitude is limited compared to the Help-Desk case.

\medskip
\def\PA{\ensuremath{\DeclaModel_\RfActivation}}
\def\PB{\ensuremath{\DeclaModel_\RfObject}}
\def\PAB{{\PA \cap \PB}}
\def\PAnB{{\PA \cap \lnot \PB}}
\def\PnAB{{\lnot \PA \cap \PB}}
\def\PnAnB{{\lnot \PA \cap \lnot \PB}}
\newcommand{\Pof}[1]{P(#1)}
\begin{table}[tb]
	\caption[Quality measures employed in the sensitivity experiment]{Quality measures employed in the sensitivity experiment. For the sake of readability, we omit parameter $L$ (the event log) from the formulae.}
	\label{tab:interestingness:measures}
	\begin{adjustbox}{width=0.95\columnwidth,center=\columnwidth}\renewcommand{\arraystretch}{2}
\begin{tabular}{r l l}
	\toprule
	\textbf{Measure} & 
	\textbf{Formula} &
	\textbf{Range}
	\\
	
	\midrule
	Support &
	$ \Pof{\PAB} $ &
	$ [0,1] $
	\\
	
	\makecell[r]{Confidence/\\Precision}&
	$ \Pof{\PB|\PA} $ &
	$ [0,1] $
	\\

	Recall&
	$ \Pof{\PA|\PB} $ &
	$ [0,1] $
	\\
Accuracy
	&
	$ \Pof{\PAB} + \Pof{\PnAnB} $ &
	$ [0,1] $
	\\
	\makecell[r]{Lift/\\Interest}
	&
	$ \dfrac{\Pof{\PAB}}{\Pof{\PA}\Pof{\PB}} $ &
	$ [0,+\infty) $
	\\
	Leverage
	&
	$ \Pof{\PB|\PA} - \Pof{\PA}\Pof{\PB} $ &
	$ [-1,1] $
	\\
Added Value
	&
	$ \Pof{\PB|\PA} - \Pof{\PB} $ &
	$ [-1, 1] $
	\\
Jaccard
	&
	$ \dfrac{\Pof{\PAB}}{\Pof{\PA}+\Pof{\PB}-\Pof{\PAB}} $ &
	$ [0,1] $
	\\
	Certainty factor
	&
	$ \dfrac{\Pof{\PB|\PA}-\Pof{\PB}}{1-\Pof{\PB}} $ &
	$ [-1,1] $
	\\
Klosgen
	&
	\makecell[l]{$ \sqrt{\Pof{\PAB}} $ \\ 
		$ \quad \times \max\left({\Pof{\PB|\PA} - \Pof{\PB}},{\Pof{\PA | \PB} - \Pof{\PA}}\right) $} &
$ [-1,1] $
	\\
	Conviction
	&
	$ \dfrac{\Pof{\PA}\Pof{\lnot \PB}}{\Pof{\PAnB}} $ &
	$ [0,+\infty) $
	\\
J-Measure
	&
$ \Pof{\PAB} \log\dfrac{\Pof{\PB|\PA}}{\Pof{\PB}} + \Pof{\PAnB} \log\dfrac{\Pof{\lnot \PB|\PA}}{\Pof{\lnot \PB}} $ & $ (-\infty,+\infty) $
	\\
	One-Way Support
	&
	$ \Pof{\PB|\PA}\log_2\dfrac{\Pof{\PAB}}{\Pof{\PA}\Pof{\PB}} $ &
	$ (-\infty,+\infty) $
	\\
	Two-Way Support
	&
	$ \Pof{\PAB}\log_2\dfrac{\Pof{\PAB}}{\Pof{\PA}\Pof{\PB}} $ &
	$ (-\infty,+\infty) $
	\\
Piatetsky-Shapiro
	&
	$ \Pof{\PAB}-\Pof{\PA}\Pof{\PB} $ &
	$ [-1,1] $
	\\
	Cosine
	&
	$ \dfrac{\Pof{\PAB}}{\sqrt{\Pof{\PA}\Pof{\PB}}} $ &
	$ [0,+\infty) $
	\\
	Loevinger
	&
	$ 1-\dfrac{\Pof{\PA}\Pof{\lnot \PB}}{\Pof{\PAnB}} $ &
	$ (-\infty,1] $
	\\
	Information Gain
	&
	$ \log\dfrac{\Pof{\PAB}}{\Pof{\PA}\Pof{\PB}} $ &
	$ (-\infty,+\infty) $
	\\
	Sebag-Schoenauer
	&
	$ \dfrac{\Pof{\PAB}}{\Pof{\PAnB}} $ &
	$ [0,+\infty) $
	\\
	Least Contradiction
	&
	$ \dfrac{\Pof{\PAB}-\Pof{\PAnB}}{\Pof{\PB}} $ &
	$ (-\infty,+\infty) $
	\\
	Odd Multiplier
	&
	$ \dfrac{\Pof{\PAB}\Pof{\lnot \PB}}{\Pof{\PB}\Pof{\PAnB}} $ &
	$ [0,+\infty) $
	\\
	\makecell[r]{Example and\\Counterexample\\Rate}
	&
	$ 1- \dfrac{\Pof{\PAnB}}{\Pof{\PAB}} $ &
	$ (-\infty,1] $
	\\
	Zhang
	&
	$ \dfrac{\Pof{\PAB}-\Pof{\PA}\Pof{\PB}}{\max\left(\Pof{\PAB}\Pof{\lnot \PB},\Pof{\PB}\Pof{\PAnB}\right)} $ &
	$ (-\infty,+\infty) $
	\\
	\bottomrule
\end{tabular} \end{adjustbox}
\end{table}
Our approach allows one to gauge the interestingness of a specification for additional measures beyond Confidence. Specifically,
\cref{tab:interestingness:measures} lists a comprehensive set of the measures that we used in our experiments. For every measure, we report the formula to compute it and its range.

The Sebag-Schoenauer measure of a specification $\DeclaModel$ given an event log $\EvtLog$, e.g., is computed as $$1-\dfrac{\Pof{\PA,\EvtLog}\Pof{\lnot \PB,\EvtLog}}{\Pof{\PAnB,\EvtLog}},$$ and its values range from $0$ (included) to $+\infty$.
Least Contradiction is defined as $$\dfrac{\Pof{\PAB,\EvtLog}-\Pof{\PAnB,\EvtLog}}{\Pof{\PB,\EvtLog}}, $$ on the range $(-\infty,+\infty)$.
Equipped with this set of measures, we conducted a comparative analysis to evaluate their ability to signal change points.

To allow for a comparison between measures defined on different ranges (e.g., $[0,+\infty)$ for the Sebag-Schoenauer measure, and $(-\infty,+\infty)$ for Least Contradiction), we normalize the values applying the method used in~\cite{Cecconi2022Measurement}, which projects all values on a $[0,1]$ interval (i.e., the range of Confidence, among others).

\begin{table*}[tb]
	\caption{Variability of specification measures in the context of drift detection.}
	\centering
	\subfloat[Sepsis~\cite{Mannhardt2016Sepsis}]{\label{tab:measure:variability:sepsis}\begin{adjustbox}{height=0.15\textheight}
			\begin{tabular}{l
		S[table-format = 1.12e+2,round-precision = 12,round-mode=places]
		S[table-format = 1.12e+2,round-precision = 12,round-mode=places]
		S[table-format = 1.12e+2,round-precision = 12,round-mode=places]}
\toprule
\textbf{Measure}                     & \textbf{Mean}          & \textbf{Std. Dev.}     & \textbf{CV}                             \\ \cmidrule(r){1-1}\cmidrule{2-4} 
Sebag-Schoenauer                     & 0.861505846319816     & 0.0020530501798698100  & 0.00238309488976777   \\
Compliance                           & 0.9249894805320810    & 0.000696413128719769   & 0.000752887620212904  \\
Jaccard                              & 0.9249894942804780    & 0.0006964131032955830  & 0.000752887581536591  \\
Accuracy                             & 0.9249894941500840    & 0.000696412944925081   & 0.000752887410429426  \\
Support                              & 0.9249894941500840    & 0.000696412944925081   & 0.000752887410429426  \\
Confidence                           & 0.9249894873573770    & 0.000696412869824218   & 0.000752887334767247  \\
Cosine                               & 0.9624220816079950    & 0.00017594325508933600 & 0.000182812986580039  \\
Example and Counterexample Rate      & 0.9698384055573610    & 0.00015956323713049200 & 0.000164525591290429  \\
Least Contradiction                  & 0.7273788053688400    & 8.9754288789846E-05    & 1.23394149138472E-04  \\
Certainty factor                     & 0.49999987553067000   & 3.62778454402209E-13   & 7.25557089423628E-13  \\
Zhang                                & 0.49999992959575500   & 1.14387060744897E-13   & 2.28774153703137E-13  \\
Leverage                             & 0.49999997161684400   & 1.83387616775375E-14   & 3.66775254371229E-14  \\
Added Value                          & 0.4999999858084200    & 4.584691433807E-15     & 9.16938312787009E-15  \\
Piatetsky-Shapiro                    & 0.4999999858084200    & 4.584691433807E-15     & 9.16938312787009E-15  \\
Klosgen                              & 0.4999999867513850    & 3.99210886256133E-15   & 7.98421793668232E-15  \\
Two-way Support                      & 0.49999999492097600   & 2.81299658592334E-15   & 5.62599322899579E-15  \\
One-way Support                      & 0.4999999949209810    & 2.81299623864483E-15   & 5.62599253443871E-15  \\
Lovinger                             & 0.6666666587678120    & 2.57148744891604E-15   & 3.85723121907562E-15  \\
Information Gain                     & 0.49999999574252400   & 1.80173298621423E-15   & 3.6034660031118E-15   \\
J Measure                            & 0.4999999960421020    & 1.75051397629489E-15   & 3.5010279803032E-15   \\
Lift                                 & 0.500000002619047     & 3.69047648344369E-16   & 7.38095292822525E-16  \\
Conviction                           & 0.5000000066666660    & 2.58333335466866E-16   & 5.16666664044844E-16  \\
Odd Multiplier                       & 0.5000000050000000    & 5.17689969051289E-32   & 1.03537992774878E-31  \\
Recall                               & 1.0                   & 0.0                    & 0.0                   \\
\bottomrule
\end{tabular} 		\end{adjustbox}
	}
	\hfill
	\subfloat[Help-Desk~\cite{Polato2017helpdesk}]{\label{tab:measure:variability:italian}\begin{adjustbox}{height=0.15\textheight}
			\begin{tabular}{@{}
		l
		S[table-format = 1.12e+2,round-precision = 12,round-mode=places]
        S[table-format = 1.12e+2,round-precision = 12,round-mode=places]
        S[table-format = 1.12e+2,round-precision = 12,round-mode=places]}
   		\toprule
		\textbf{Measure}                     & \textbf{Mean}          & \textbf{Std. Dev.}     & \textbf{CV}                             \\ \cmidrule(r){1-1}\cmidrule{2-4} 
Sebag-Schoenauer                     & 0.9008192372816360   & 0.006297508246123700   & 0.00699086785172064  \\
Accuracy                             & 0.9506220743728800   & 0.0021794735617332800  & 0.00229268141408463  \\
Support                              & 0.9506220743728800   & 0.0021794735617332800  & 0.00229268141408463  \\
Compliance                           & 0.9506220711662860   & 0.0021794734722218400  & 0.00229268132765729  \\
Confidence                           & 0.9506220761053870   & 0.0021794733337249500  & 0.00229268117005451  \\
Jaccard                              & 0.9506220811878140   & 0.002179473204845890   & 0.00229268102222348  \\
Cosine                               & 0.9752398032617170   & 0.0005515822391238080  & 0.000565586266351133 \\
Example and Counterexample Rate      & 0.9794139246013880   & 0.0005243043514953000  & 0.000535324583738879 \\
Least Contradiction                  & 0.7345604451566600   & 0.00029492119236132600 & 0.000401493429582133 \\
Certainty factor                     & 0.5000000479861520   & 3.02112660488172E-13   & 6.04225262987453E-13 \\
Zhang                                & 0.5000000268109960   & 8.43416317016721E-14   & 1.68683254358212E-13 \\
Leverage                             & 0.5000000090702720   & 4.13323212570877E-15   & 8.26646410145938E-15 \\
One-way Support                      & 0.5000000087120500   & 1.03728706039543E-15   & 2.07457408464327E-15 \\
Two-way Support                      & 0.5000000087120490   & 1.03728700045095E-15   & 2.07457396475432E-15 \\
Added Value                          & 0.500000004535136    & 1.03330821258011E-15   & 2.06661640641545E-15 \\
Piatetsky-Shapiro                    & 0.500000004535136    & 1.03330821258011E-15   & 2.06661640641545E-15 \\
Klosgen                              & 0.5000000042737360   & 9.20537665386252E-16   & 1.84107531503596E-15 \\
J Measure                            & 0.5000000073832810   & 6.86343306927391E-16   & 1.37268659358492E-15 \\
Information Gain                     & 0.5000000068027020   & 6.26668457801818E-16   & 1.25333689855148E-15 \\
Lovinger                             & 0.6666666696480230   & 8.2867194558003E-16    & 1.24300791281127E-15 \\
Lift                                 & 0.5000000078260860   & 1.11156231710444E-16   & 2.22312459941215E-16 \\
Conviction                           & 0.5000000030952380   & 8.84968477979402E-17   & 1.76993694500205E-16 \\
Odd Multiplier                       & 0.5000000050000000   & 1.99591313369172E-31   & 3.99182622746518E-31 \\
Recall                               & 1.0                  & 0.0                    & 0.0                  \\
\bottomrule
\end{tabular} 		\end{adjustbox}
	}
	\label{tab:measure:variability}
\end{table*}
\Cref{tab:measure:variability:sepsis,tab:measure:variability:italian} show the result of our analysis on both the Sepsis and Help-Desk logs, respectively. We report only measures with non-null values. For every measure, we indicate the mean value, the standard deviation, and the Coefficient of Variation (CV). CV, expressed as the ratio of the standard deviation to the mean, shows the extent of variability in relation to the average value along the sequence of collected values. We sort the measures in descending order by CV, standard deviation and mean.
Measures that exhibit relatively more ample oscillations (hence, most sensitive to changes) are thus at the top.
Both with the Sepsis and Help-Desk logs, the Sebag-Schoenauer measure ranks first, followed by a group of measures including Support and Confidence with very close values.
From the Least Contradiction on, all measures follow with a CV that is lower by orders of magnitude (from $10^{-4}$ down to $10^{-13}$ and below). The Recall measure closes the list with a steady mean of \num{1.0}, as both CV and standard deviation equate to \num{0}.
Notice that, besides minimal variations in the ranking, the aforementioned groups are composed by the same measures both in \cref{tab:measure:variability:sepsis,tab:measure:variability:italian}, thus regardless of the event log under examination.

\begin{figure*}[htb]
	\centering
	\includegraphics[width=0.66\linewidth]{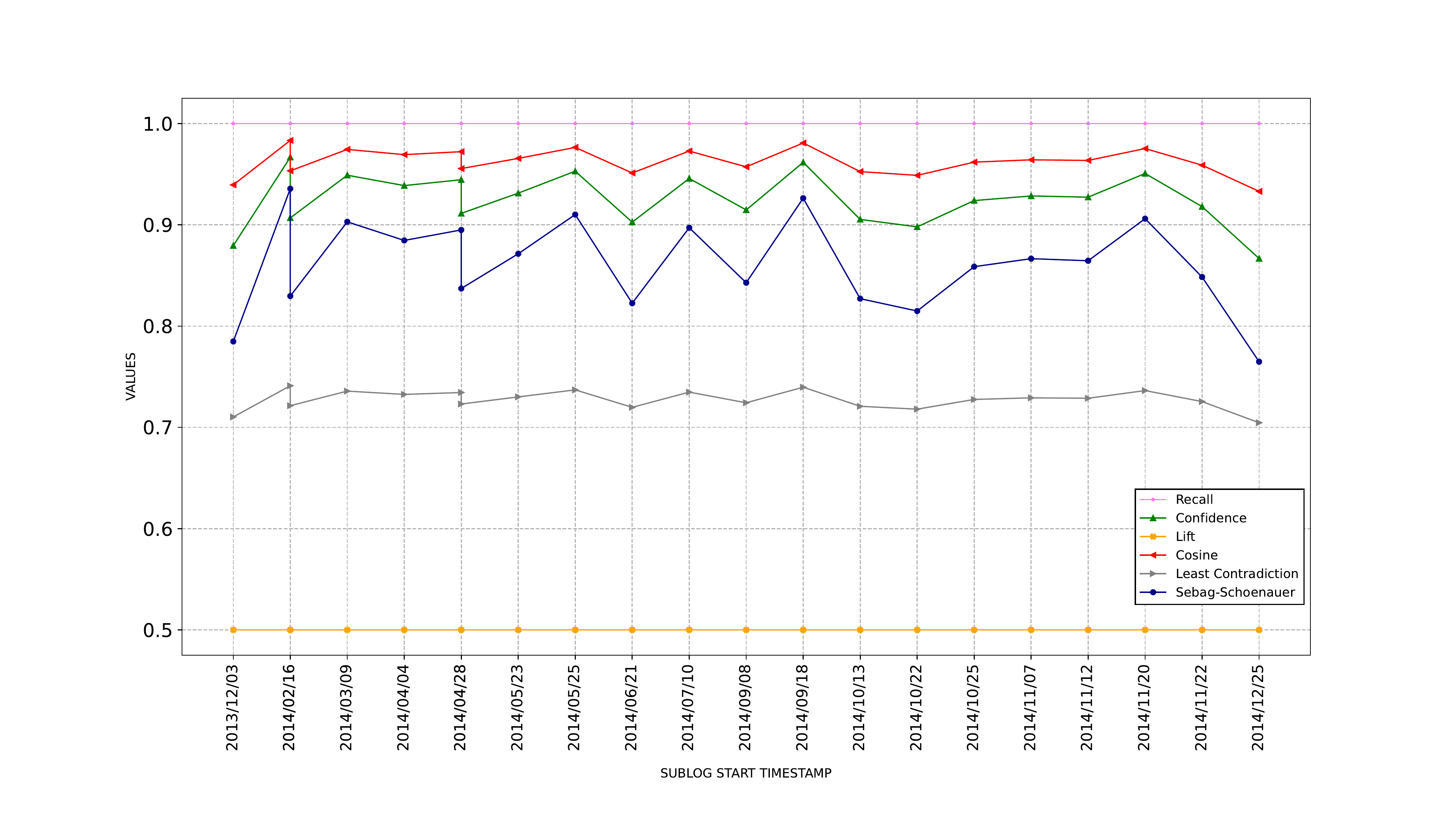}
	\caption[Specification measure oscillations on the Sepsis event log]{Specification measure oscillations on the Sepsis event log}
	\label{fig:sepsis_measures_trends}
\end{figure*}
\begin{figure*}[htb]
	\centering
	\includegraphics[width=0.66\linewidth]{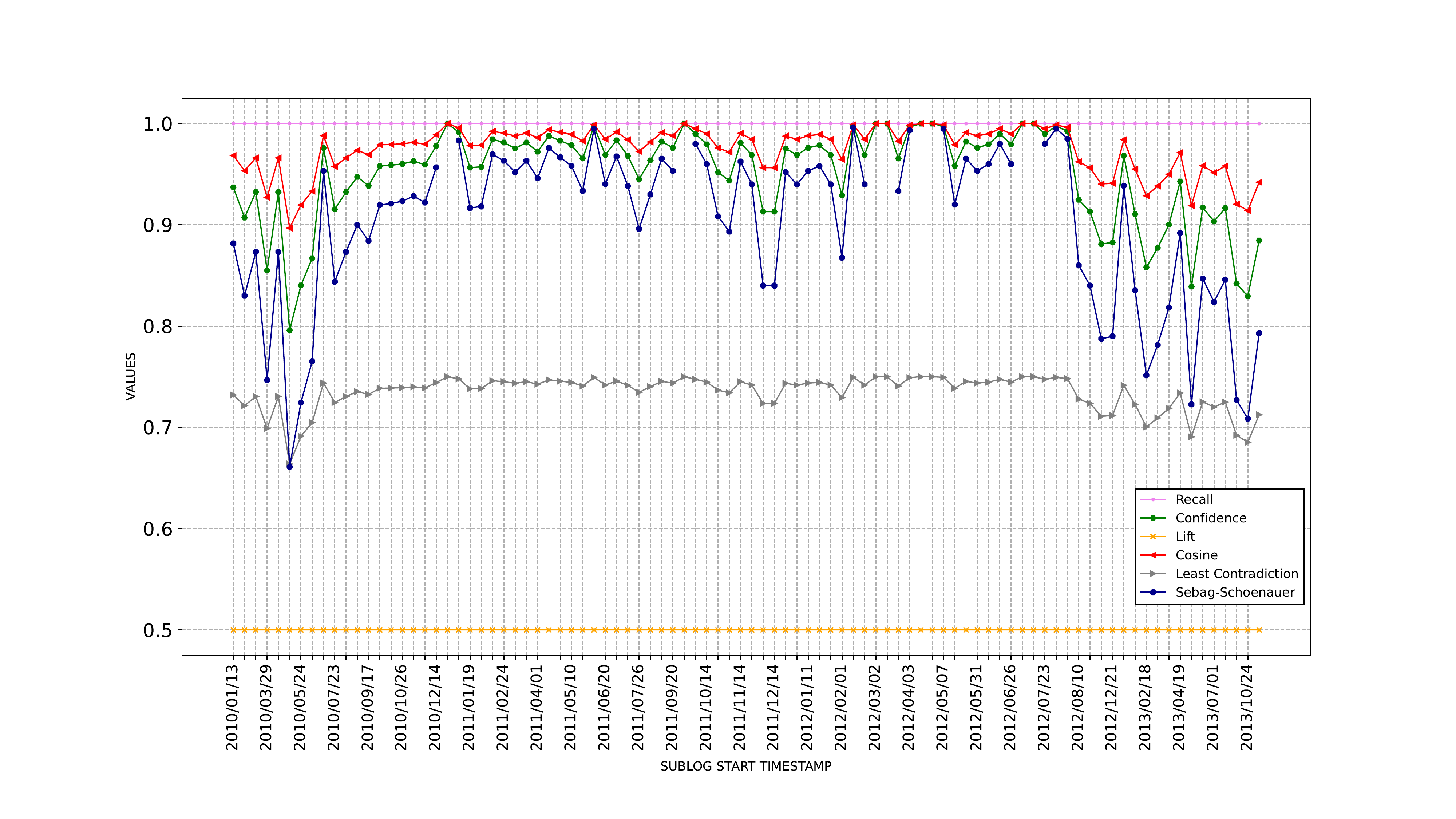}
	\caption[Specification measure oscillations on the help desk event log]{Specification measure oscillations on the Help-Desk event log}
	\label{fig:italian_measures_trends}
\end{figure*}
\Cref{fig:sepsis_measures_trends,fig:italian_measures_trends} plot the trends of those measures to visually compare their variations with the sub-logs extracted from the Sepsis and Help-Desk datasets, respectively.
In both cases, we observe that the amplest oscillations are exhibited by the Sebag-Schoenauer measure, as expected in light of the previous discussion.
Similar trends characterize the polylines associated with measures representing the group at the top of \cref{tab:measure:variability:sepsis,tab:measure:variability:italian}, such as Confidence and Cosine.
Least Contradiction lies in the middle of the diagrams and, though subject to variations, shows a lower sensitivity to changes.
Lift
is a representative of a large group of measures showing only imperceptible fluctuations.
Finally, as anticipated above, Recall remains steadily equal to \num{1.0}.

\bigskip
\todo{@Claudio: The new part on partial subspec analysis here.}
\begin{table}[tb]
	\caption[Aggregate confidence in the Help-Desk sub-logs]{Aggregate confidence in the Help-Desk sub-logs~\cite{Polato2017helpdesk}  for the  constraint clusters exhibiting the most erratic behavior~\cite{DBLP:conf/er/YeshchenkoCMP19a}}
	\label{tab:italian:clusters}
	\begin{tabular}{l 
			S[table-format = 1.4,round-precision = 4,round-mode=places] 
			S[table-format = 1.4,round-precision = 4,round-mode=places] 
			S[table-format = 1.4,round-precision = 4,round-mode=places] 
			S[table-format = 1.4,round-precision = 4,round-mode=places] 
			S[table-format = 1.4,round-precision = 4,round-mode=places]}
		\toprule
		\textbf{Cluster}    & \textbf{Mean}        & \textbf{Std.\ Dev}   & \textbf{CV}          & \textbf{Max}        & \textbf{Min}         \\
		\midrule
		Cluster 4           & 0.633592362 & 0.008314476 & 0.013122753 & 0.89115238 & 0.512873275 \\
		Cluster 9           & 0.316430342 & 0.211047113 & 0.666962314 & 1          & 0           \\
		Cluster 11          & 0.054347826 & 0.051958911 & 0.956043956 & 1          & 0           \\
		\midrule
		Joint specification & 0.634607513 & 0.00815469  & 0.012849974 & 0.89176961 & 0.512873275 \\
		\bottomrule
	\end{tabular}
\end{table}

\begin{table}[tb]
	\caption[Aggregate confidence in the Help-Desk sub-logs]{Aggregate confidence in the Sepsis sub-logs~\cite{Mannhardt2016Sepsis} for the  constraint clusters exhibiting the most erratic behavior~\cite{DBLP:journals/tvcg/YeshchenkoCMP22}}
	\label{tab:sepsis:clusters}
	\begin{tabular}{l 
			S[table-format = 1.4,round-precision = 4,round-mode=places] 
			S[table-format = 1.4,round-precision = 4,round-mode=places] 
			S[table-format = 1.4,round-precision = 4,round-mode=places] 
			S[table-format = 1.4,round-precision = 4,round-mode=places] 
			S[table-format = 1.4,round-precision = 4,round-mode=places]}
		\toprule
		\textbf{Cluster}    & \textbf{Mean} & \textbf{Std.\ Dev} & \textbf{CV} & \textbf{Max} & \textbf{Min} \\ \midrule
		Cluster 8           & 0.497715576   & 0.240440682        & 0.483088523 & 1            & 0            \\
		Cluster 12          & 0.302526847   & 0.177461662        & 0.586598062  & 1            & 0            \\ \midrule
		Joint specification & 0.55178841     & 0.148071382         & 0.26834812 & 1            & 0            \\ \bottomrule
	\end{tabular}
\end{table}

The trends exhibited in \cref{fig:italian_sublog_confidence,fig:sepsis_sublog_confidence,fig:italian_measures_trends,fig:sepsis_measures_trends} pertain to full specifications mined from the original event logs.
Our approach, however, can gauge the interestingness of any sets of constraints, including sub-specifications (as illustrated in the experiment on synthetic data in \cref{sec:evaluation-discovery-vs-specification-measurements}).
Here, we observe the variations that Confidence undergoes for some sub-specifications in particular.
To have a better understanding of the rules that have led to the drifts, we form these sub-specifications by joining the constraints in the clusters that exhibit the most erratic trends in the Help-Desk log~\cite{Polato2017helpdesk} and the Sepsis log~\cite{Polato2017helpdesk} according to~\cite{DBLP:conf/er/YeshchenkoCMP19a} and \cite{DBLP:journals/tvcg/YeshchenkoCMP22}, respectively.
The authors of~\cite{DBLP:conf/er/YeshchenkoCMP19a,DBLP:journals/tvcg/YeshchenkoCMP22} indicate that those clusters are the ones that mainly signal change points in the process.
Here, we leverage the capability of our approach to assess the behavior of those constraints taken together as sub-specifications, rather than individually as in~\cite{DBLP:conf/er/YeshchenkoCMP19a,DBLP:journals/tvcg/YeshchenkoCMP22}.
Thereafter, we also create specifications that join those sub-specifications.
Finally, we measure Confidence for every sub-specification and specification on the respective sub-logs, using the same tumbling window approach as in the experiments above.

\Cref{tab:italian:clusters,tab:sepsis:clusters} report on the minimum, maximum, and average values of Confidence together with the standard deviation and coefficient of variation for the clusters and the specifications derived therefrom on the Help-Desk and Sepsis event logs, respectively. 
We can observe that the erraticity of the clusters' behavior is confirmed by the CV, which is orders of magnitude higher than the values reported in \cref{tab:measure:variability:sepsis,tab:measure:variability:italian}.
\Cref{fig:cluster:italian,fig:cluster:sepsis} illustrate the trend of Confidence over the subsequent sub-logs from the Help-Desk and Sepsis datasets, respectively.
Both figures consider the individual clusters alone (\cref{fig:cluster:italian:9,fig:cluster:italian:11,fig:cluster:italian:4,fig:cluster:sepsis:8,fig:cluster:sepsis:12}) and the joined specifications (\cref{fig:cluster:italian:all,fig:cluster:sepsis:all}).
The plots in~\cref{fig:cluster:italian:9,fig:cluster:italian:11,fig:cluster:italian:4,fig:cluster:sepsis:8,fig:cluster:sepsis:12}, pertaining to the individual sub-specifications, closely resemble the trends illustrated in~\cite{DBLP:conf/er/YeshchenkoCMP19a,DBLP:journals/tvcg/YeshchenkoCMP22}.
However, we can also observe particular phenomena due to the definition of Confidence itself (see \cref{tab:interestingness:measures}):
\begin{iiilist}
	\item The Confidence of joined specifications tend to take the non-zero values from the sub-specifications, on one hand;
	\item On the other hand, the height of the peaks stemming from some sub-specifications is reduced whenever in other sub-specifications a corresponding lower Confidence is reported.
\end{iiilist}

The former effect is particularly visible comparing \cref{fig:cluster:sepsis:8,fig:cluster:sepsis:12} with \cref{fig:cluster:sepsis:all}: the plateaus lying at $0$ occur in \cref{fig:cluster:sepsis:all} only in correspondence of $0$ values for the Confidence of \emph{both} the clusters the specification is composed of.
Such a plateau, e.g., occurs between the marks ``2014/03/09'' and ``2014/12/25'' in \cref{fig:cluster:sepsis:12} (cluster 12), but not in \cref{fig:cluster:sepsis:8} (cluster 8). Then, the plateau does not occur in the Confidence plot for the whole specification in \cref{fig:cluster:sepsis:all}.
We see plateaus of Confidence equal to $0$ in \cref{fig:cluster:sepsis:all} right before the ``2014/03/09'' and ``2014/11/07'', instead, as the Confidence of neither of the sub-specifications goes beyond the minimum.

The conjunction of the two aforementioned effects is instead noticeable from \cref{fig:cluster:italian:4,fig:cluster:italian:all}. Confidence for cluster 4 (\cref{fig:cluster:italian:4}) neither drops to $0$ nor raises to $1$ with any sub-log, as opposed to clusters 9 and 11 (\cref{fig:cluster:italian:9,fig:cluster:italian:11}; see also the minimum and maximum values reported in \cref{tab:measure:variability:italian}). Then, the plot in \cref{fig:cluster:italian:all} tends to almost completely overlap with that of \cref{fig:cluster:italian:4}. We remark that this effect differs from what we would have obtained by merely averaging the Confidence values of the clusters' sub-specifications to assign the Confidence values of the joined specification.
 
\begin{figure*}[tbp]
	\centering
	\begin{subfigure}{.49\textwidth}
		\centering
		\includegraphics[width=1.0\linewidth]{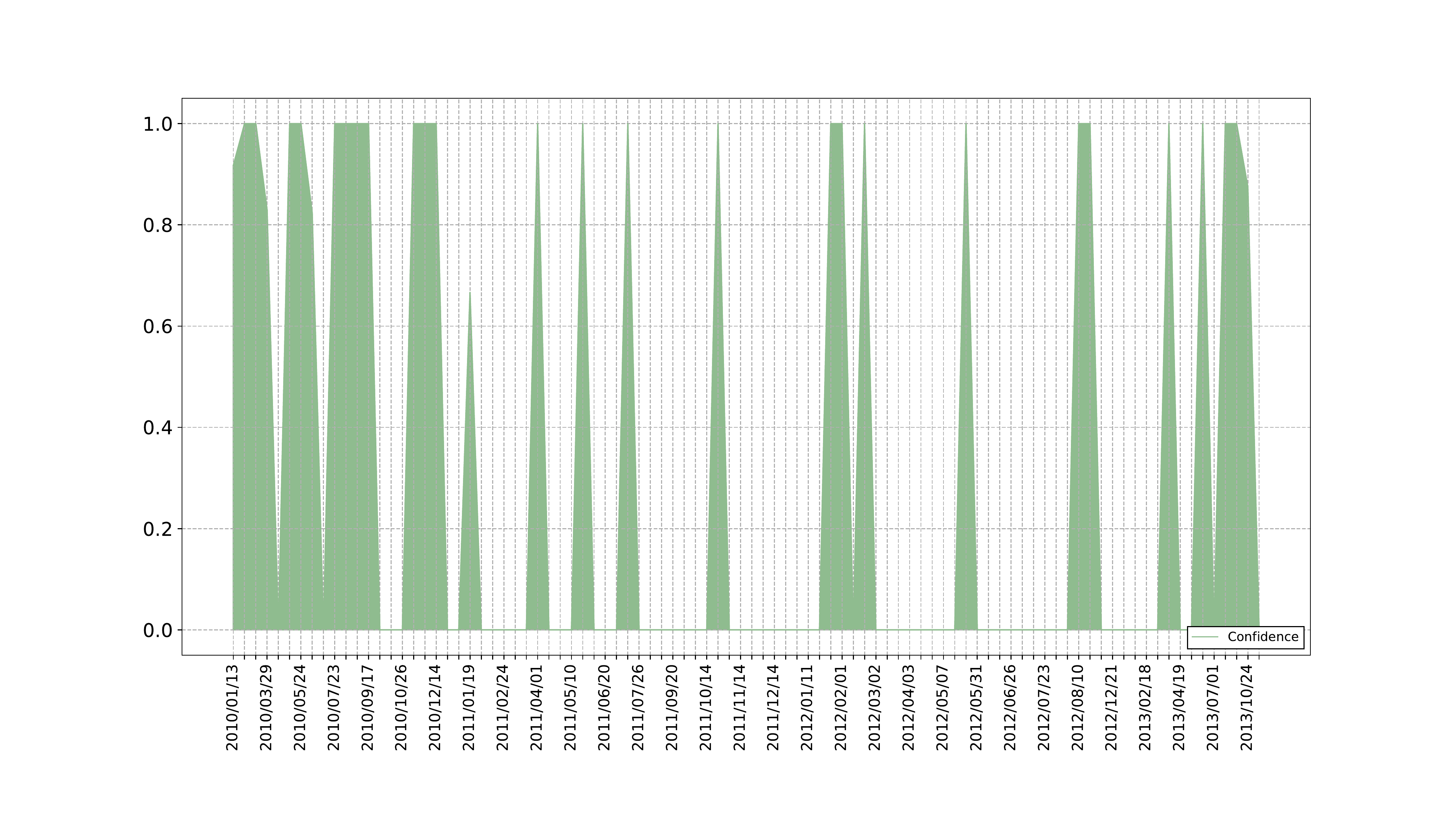}
		\caption{Cluster 9}
		\label{fig:cluster:italian:9}
	\end{subfigure}\hfill
	\begin{subfigure}{.49\textwidth}
		\centering
		\includegraphics[width=1.0\linewidth]{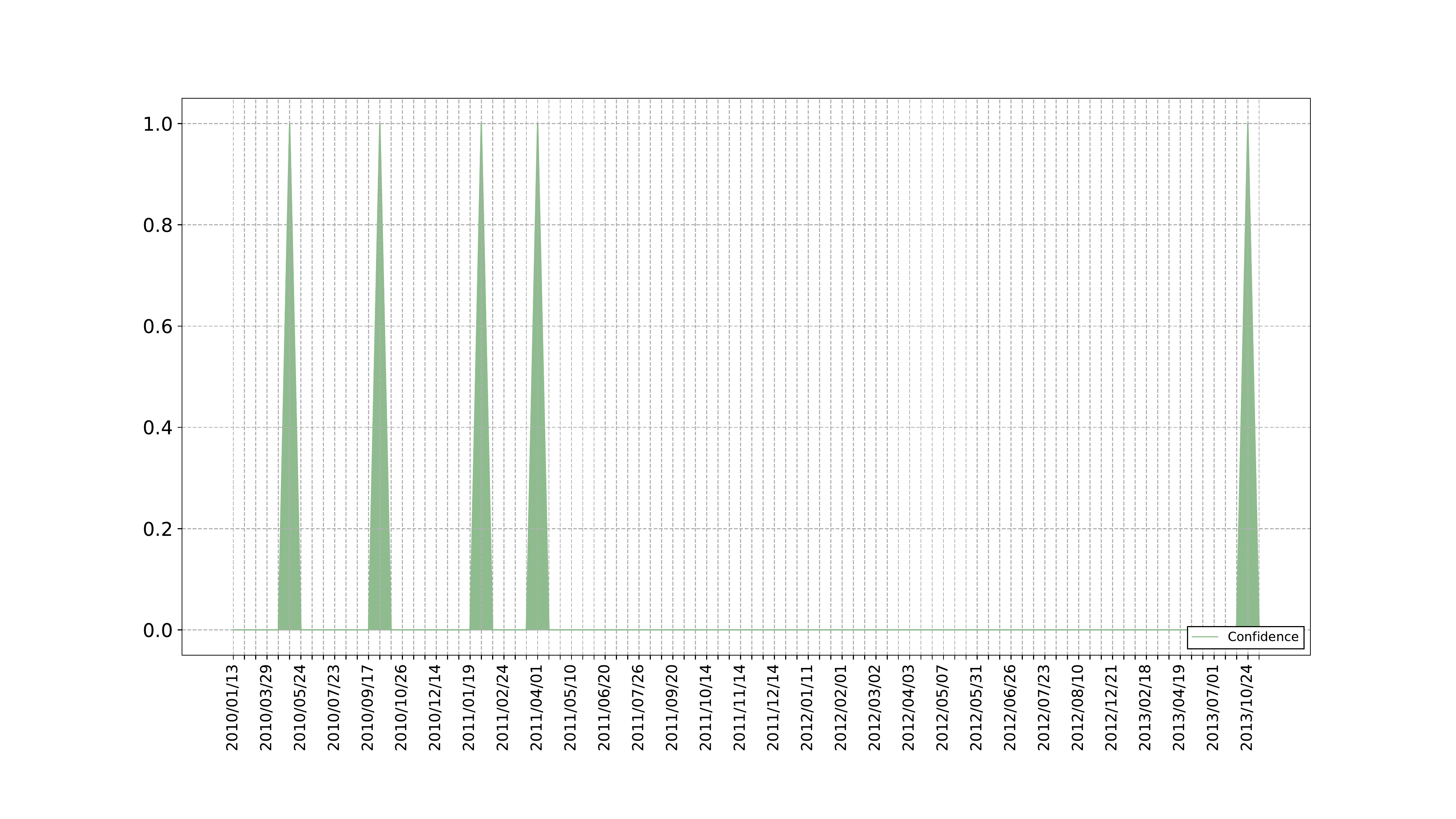}
		\caption{Cluster 11}
		\label{fig:cluster:italian:11}
	\end{subfigure}\\
	\begin{subfigure}{.49\textwidth}
		\centering
		\includegraphics[width=1.0\linewidth]{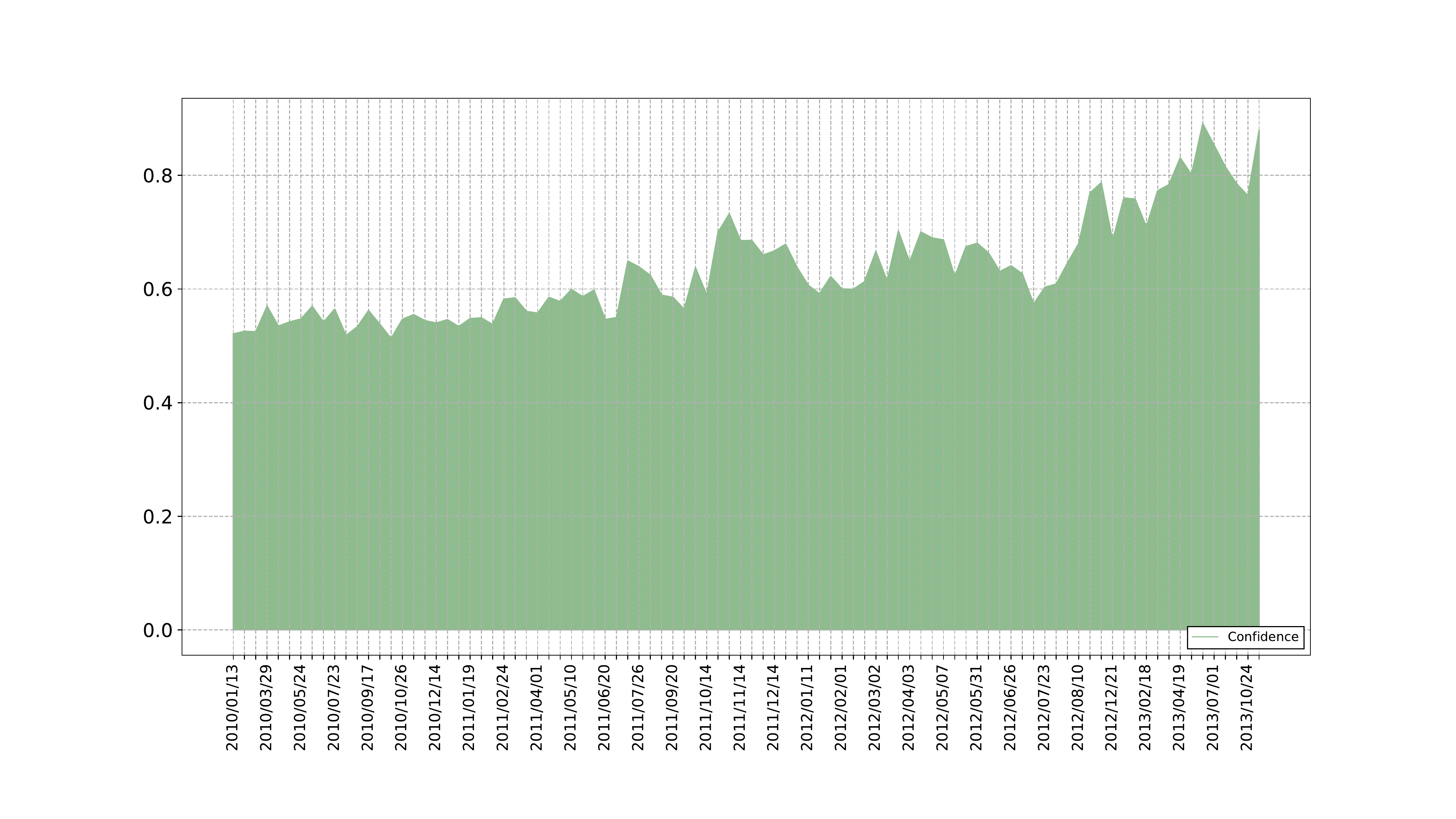}
		\caption{Cluster 4}
		\label{fig:cluster:italian:4}
	\end{subfigure}\hfill
	\begin{subfigure}{.49\textwidth}
		\centering
		\includegraphics[width=1.0\linewidth]{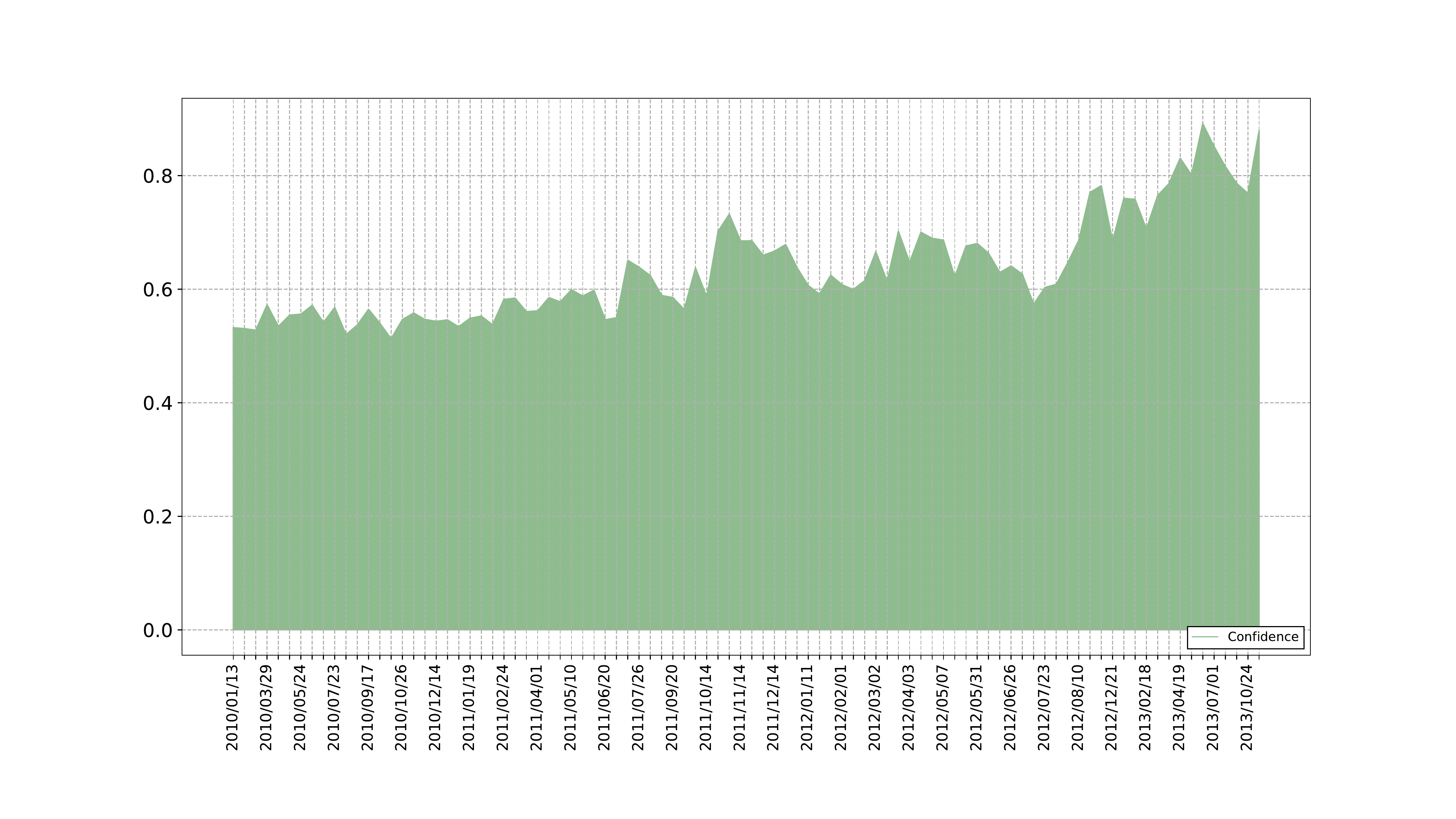}
		\caption{Specification consisting of the constraints in the clusters}
		\label{fig:cluster:italian:all}
	\end{subfigure}
	\caption[Confidence of the erratic clusters with the Help-Desk log]{Confidence of the constraint clusters with the most erratic behavior with the Help-Desk log, and of the specification stemming from their union}
	\label{fig:cluster:italian}
\end{figure*}

\begin{figure*}[tbp]
	\centering
	\begin{subfigure}{.49\textwidth}
		\centering
		\includegraphics[width=1.0\linewidth]{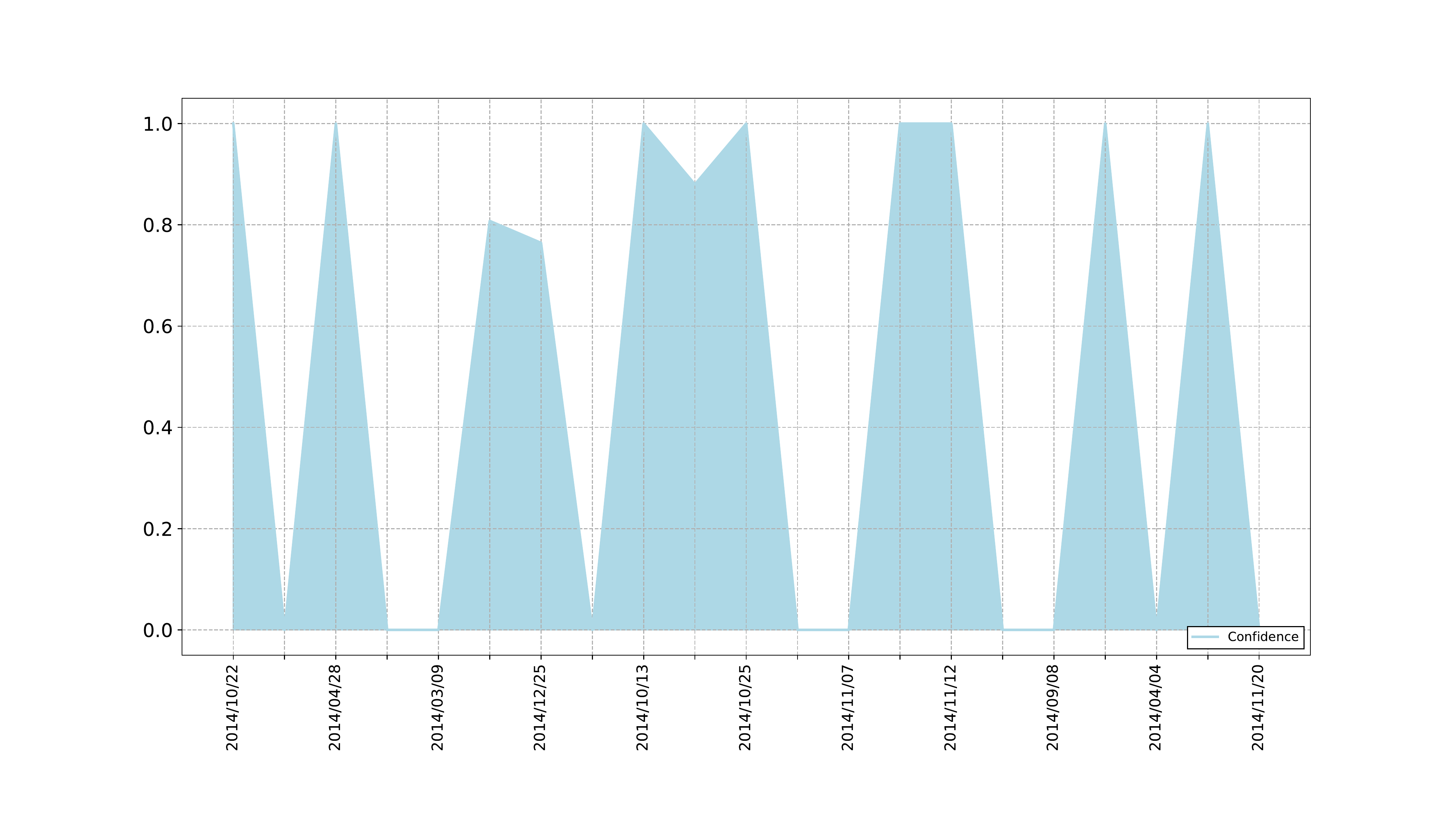}
		\caption{Cluster 8}
		\label{fig:cluster:sepsis:8}
	\end{subfigure}\hfill
	\begin{subfigure}{.49\textwidth}
		\centering
		\includegraphics[width=1.0\linewidth]{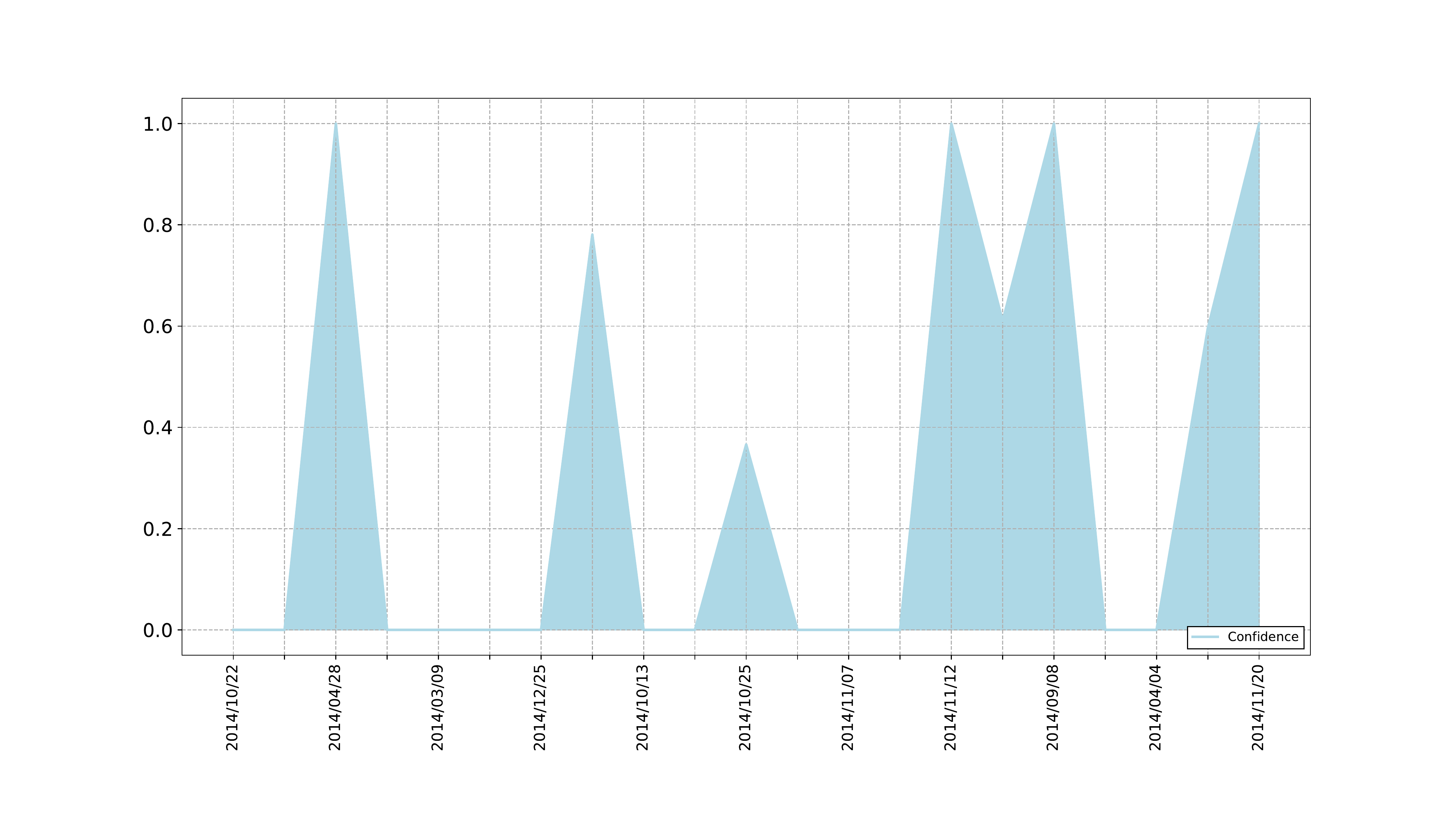}
		\caption{Cluster 12}
		\label{fig:cluster:sepsis:12}
	\end{subfigure}\\
	\begin{subfigure}{.49\textwidth}
		\centering
		\includegraphics[width=1.0\linewidth]{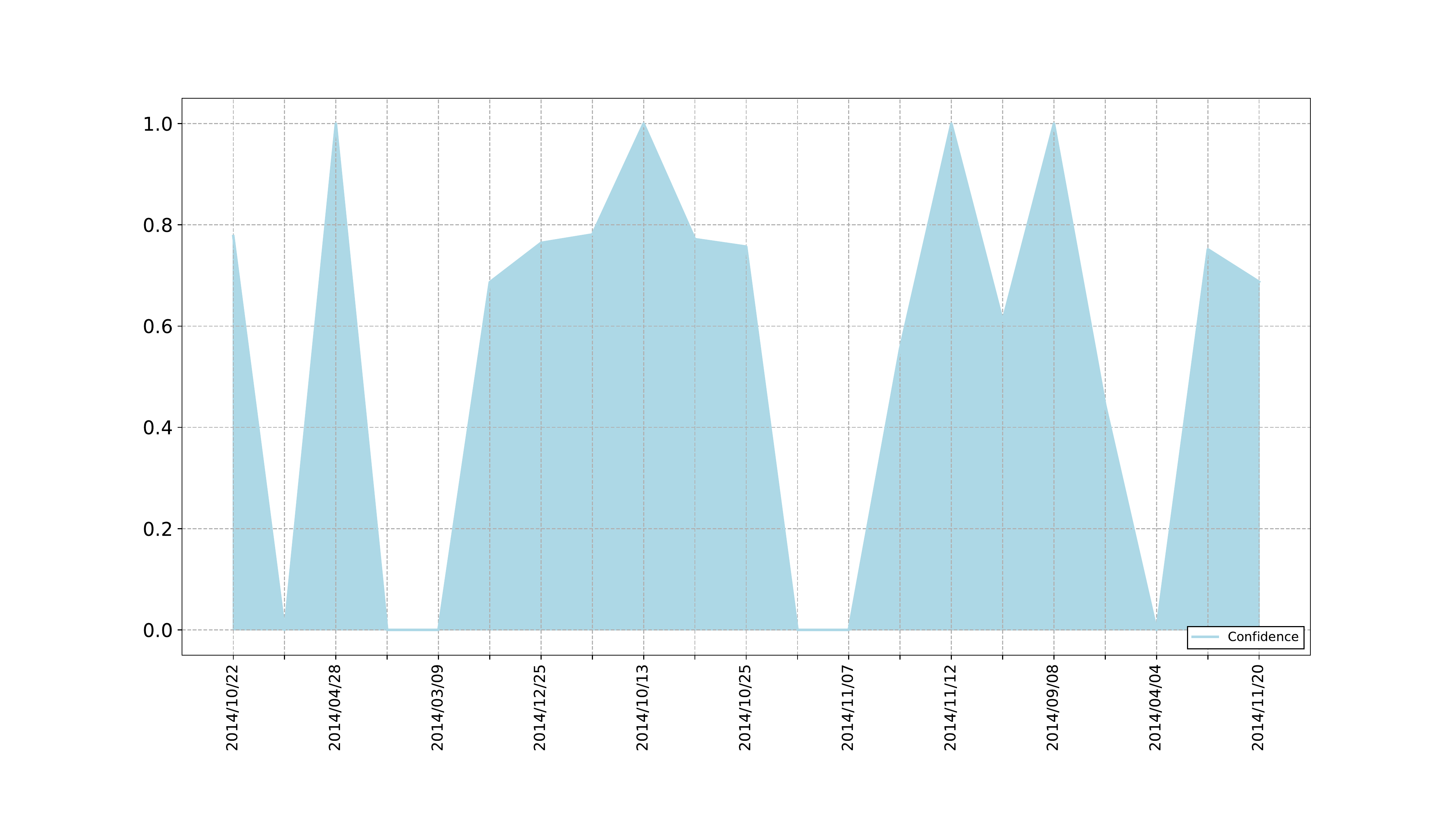}
		\caption{Specification consisting of the constraints in the clusters}
		\label{fig:cluster:sepsis:all}
	\end{subfigure}\caption[Confidence of the erratic clusters with the Sepsis log]{Confidence of the constraint clusters with the most erratic behavior with the Sepsis log, and of the specification stemming from their union}
	\label{fig:cluster:sepsis}
\end{figure*}

\subsection{Discussion}
\label{sec:exp:discussion}
Thus far, we have investigated various aspects to which the application of our measurement framework contributes with novel insights. We provide their summary below.

First of all, we observed in \cref{sec:evaluation-discovery-vs-specification-measurements} that the overall compliance of an event log to a whole specification tends to be lower than its compliance to every rule taken individually. This aspect is of considerable relevance in the context of process and specification mining, as current works in the literature still tend to resort to an analysis centered around individual rules, thus neglecting their effect on the expected behavior in combination~\cite{Yang2006Perracotta,Lemieux2015General,Le2015Beyond,DiCiccio2015OnTheDiscovery,Alman.etal/BPM2021:RuMDeclarativeProcessMiningDistilled,DBLP:journals/sttt/BackSHM22}.

Also, the opportunity to compute a large array of measures, including but not limited to those that derive from association rule mining, sheds light on interesting facts from different perspectives. 
The empirical piece of evidence presented in \cref{sec:evaluation-absolute-match-different-measures} confirms that a single measure is not sufficient to provide a full account of the degree of interestingness of a mined process specification.
The use case illustrated in \cref{sec:use-case-process-drift} shows a possible application of our approach to assessing the variability of the degree of compliance of process executions with the overall specification over time. Our tool is not meant to be a drift identification technique. Other approaches, such as the aforementioned VDD~\cite{DBLP:journals/tvcg/YeshchenkoCMP22}, are specifically designed to achieve this goal. However, our method can be used in practice to observe the effect that drifts and other change points in the process may have on an entire specification~\cite{DBLP:conf/bpmds/SchutzenmeierCD23}. The most suitable measure to consider for this kind of analysis appears to be the Least Contradiction.
Also, the oscillations of measures in the sequence highlight that process executions are subject to variations over time, while the mining of a whole event log cannot represent these changes in behavior (either temporary or not).
Furthermore, we have observed that the mere aggregation of measures stemming from sub-specifications does not properly account for the measurement of a conjoint specification.

In the next section, we provide an overview of the related literature and position our contribution against the existing background.

\section{Related work}
\label{sec:related-works}
Different contributions in the literature aim at quantitative extensions of \acrshort{ltl}/\gls{ltlf} enriching the languages with quantitative operators.
\cite{Almagor2016Formally,Lahijanian2015ThisTime,Piribauer2021Quantified} all proposed the addition of quantitative operators into the logic.
\acrshort{ltl}$[\mathcal{F}]$~\cite{Almagor2016Formally} introduces quality operators quantifying over distinct satisfactions of a formula. Quantified-\acrshort{ltl}~\cite{Piribauer2021Quantified} uses quantifiers over its propositional variables, also in probabilistic systems such as Markov chains. 

In~\cite{Lahijanian2015ThisTime}, the quantification of satisfaction, applied in the context of planning, 
is based on associating costs to specification violations based on user ranking of tasks priority. Differently from these methods, we do not extend the syntax and semantics of \gls{ltlf} with new operators, as we quantify the satisfaction of formulae based on standard \gls{ltlf}. 

As for the interplay of temporal logic and probabilities, statistical model checking techniques~\cite{Legay2019Statistical} retrieve the probability for a formula to be satisfied in a probabilistic environment as Markov chains. Their goal is to predict the likelihood of a formula for any possible execution (a probabilistic relaxation of traditional model checking), while we study only already executed traces.
The method proposed in \cite{Maggi2020Temporal} is close to our investigation, as it resorts to the association of a probability threshold to each rule. The threshold is used to perform relaxed conformance checking: each rule should hold in at least a portion of the log that is greater than that value.
However, only single rules are analyzed and the trace satisfaction is not quantified but considered as boolean, whereas we can assess the partial satisfaction of specifications, also on single traces. 

In process mining, the compliance of process models to the data is usually gauged with four scores: Fitness, Precision, Generalization, and Simplicity~\cite{buijs2014quality}.
In~\cite{DeLeoni2015AnAlignment} Fitness, Precision, and Generalization are devised for {\Declare} models through alignments, while in~\cite{Polyvyanyy2020Monotone} Fitness and Precision are computed for any regular language through entropy.
Our framework focuses on a different set of measures, inspired by association rule mining. The comparison and integration of the four measures above paves the path for future research endeavors.
Notably, the novel measure of informativeness is proposed in~\cite{Burattin2019Fifty} to understand the differences between compliant traces. We showed in \cref{sec:evaluation} how separate measures can spot differences in compliant specifications, thus a deeper analysis in this direction is an interesting research outlook.

\section{Conclusion}
\label{sec:conclusions}
In this paper, we presented a tractable approach to quantify the satisfaction degree of rule-based \gls{ltlf} specifications on bags of traces. Our approach is grounded in probabilistic models with which we have derived maximum-likelihood estimators for \gls{ltlf} formulae, \glspl{rf}, and process specifications. We apply our prototype to real-world data, showing its broad range of employment. We provide experimental evidence that the Confidence of a mined specification is often below the minimum levels set for the discovery of its rules. Also, we observe that a single measure does not give a full picture of the level of interestingness of a specification for an event log. Furthermore, we describe the effect of drifts and anomalies on the specification measurements, with insights into the measures that are more suitable for the signaling of behavioral changes over time.

The advantage of analyzing processes at the higher level of a specification is complementary to the details provided by its single rules. The specification measurement gives a holistic view of the status of the process, which could not be achieved by the sum of its single parts. This information can guide subsequent detailed analysis, e.g., highlighting the overall mismatch between the specification and the data, whenever it diverges significantly enough to call for an in-depth analysis.

Looking onward, 
our result for {\ltlf} can be easily extended to \acrlong{ldlf}~\cite{DeGiacomo2013Linear}, which has the expressive power of Monadic Second Order Logic, but with the same computational cost of \gls{ltlf}. Moreover, we aim to explore the enrichment of log labeling with additional contextual data (such as patient diagnoses) akin to \cite{DBLP:conf/icsoc/SchonigCMM16} and
construct estimators that take this information into account. Specifically, a possible extension would be to model the event of satisfying a formula conditional on context via logistic regression.
Also, a relevant direction for practical implementations is the assessment of measures based on the most common specification and process mining tasks. To this end, the definition of desirable propositions for the interestingness measures and the properties they should guarantee is crucial, similarly to what has been done in~\cite{Syring2019Evaluating} for conformance measures in process mining.\todo{added reference suggested by rev.1}
Another interesting outlook is the employment of 
specifications measures as data features for machine learning applications, 
e.g., trace clustering~\cite{Weerdt2019Trace}.
Similarly, as hinted in the evaluation, our measurement framework can be used to support process mining applications, such as drift detection (with statistically significant identification of change points and pinpointing of the sub-specifications exhibiting the most erratic behavior), or post-processing filtering for declarative process discovery techniques.
A highly stimulating research endeavor can be spurred by the application of our technique on the field, in the context of highly flexible, dynamic system and process execution scenarios such as healthcare with the checking and monitoring of rules defined in clinical pathways~\cite{DBLP:conf/pakdd/WeerdtCVB12}, with extensions aimed at weighting differently constraints according to their compulsory or best-practice nature~\cite{DBLP:journals/flap/AmanteaRSGB22}.

\section*{Acknowledgments}
L.~Barbaro and C.~Di~Ciccio were partly supported by the Italian Ministry of University and Research (MUR) under PRIN grant B87G22000450001 (PINPOINT).
L.~Barbaro received funding from the Latium Region under the PO FSE+ grant B83C22004050009 (``Predictive process monitoring for production planning'').
C.~Di~Ciccio was also supported by project SERICS (PE00000014) under the NRRP MUR program funded by the EU-NGEU.
The authors declare that they have no known competing financial interests or personal relationships that could have appeared to influence the work reported in this paper.

\printcredits


\appendix

\section{Basic Probability Notation and Results}
\label{appendix:probability}
The theory of probability is a fundamental mathematical concept that underpins a wide range of fields, including statistics, economics, and physics~\cite{feller1957introduction}. The key concepts in probability theory include the sample space $\Omega$, events $A, B, C, \dots$, probability functions $P(\cdot)$, and conditional probabilities $P(A \mid B)$.
The sample space $\Omega$ represents the set of all possible outcomes of an experiment. For example, when flipping a coin, the sample space is $\{\textrm{H}, \textrm{T}\}$, where $\textrm{H}$ represents heads and $\textrm{T}$ represents tails.

Events $A, B, C, \dots$ in probability theory are subsets of the sample space ($A \cup B \cup C \ldots \subseteq \Omega $) representing specific outcomes of interest.
For example, $A$ could represent the event of getting heads when flipping a coin. Probability functions $P(\cdot)$ assign probabilities to events: $P(A)$, e.g., represents the probability of event $A$ occurring, and it is a number between \num{0} and \num{1}, inclusive. $P(A \cap B)$ represents the probability of both $A$ and $B$ occurring, also known as the intersection of $A$ and $B$. $P(A \cup B)$ represents the probability of either $A$ or $B$ occurring, also known as the union of $A$ and $B$.
Notice that the notion of event in probability slightly differs from that of event in the context of process and specification mining (see \cref{def:eventlog}). $P(A)$ represents the probability that a record in the trace satisfies $A$. In this setting, then, $A$ is a statement that can hold true, or not, in the elements that a trace consists of.

\subsection{Conditional Probability}
Conditional probability is the probability of an event $A$ given that another event $B$ has occurred. We write it as $P(A \mid B)$ and define it using Kolmogorov's analysis:
\begin{equation}
P(A \mid B) = \frac{P(A \cap B)}{P(B)}.
\end{equation} 
This formula allows us to update our beliefs about the occurrence of an event based on new information. The following properties of conditional probability hold. Firstly, $P(A \mid B) \geq 0$ for all events $A$ and $B$. This means that the conditional probability of $A$ given $B$ is always non-negative.

Secondly, $P(\Omega \mid B) = 1$ for any event $B$. This means that the probability of the entire sample space given that $B$ has occurred is equal to $1$. Thirdly, if $A$ and $B$ are mutually exclusive events (i.e., $A \cap B = \emptyset$), then $P(A \mid B) = 0$. This is because if $A$ and $B$ cannot occur simultaneously, then the occurrence of $B$ rules out the possibility of $A$ occurring.

The law of total probability (LTP) can be used to compute conditional probabilities. If $A_1, A_2, \dots, A_n$ are mutually exclusive events that partition the sample space, then 
\begin{equation}
	P(B) = \sum_{i=1}^{n} P(B \mid A_i) P(A_i).
\end{equation} 
This formula can be rearranged to compute conditional probabilities as follows (for any $1 \leq j \leq n$):
\begin{equation}
	P(A_j \mid B) = \frac{P(B \mid A_j) P(A_j)}{\sum_{i=1}^{n} P(B \mid A_i) P(A_i)}.
\end{equation} 

\subsection{Discrete Random Variables}
A discrete random variable is a random variable that takes on a countable number of values, such as the number of heads obtained in a series of coin flips, or the number of defects in a batch of products. We denote a discrete random variable as $X$ and its possible values as $x_1, x_2, \dots, x_n$. The probability distribution of $X$ specifies the probabilities $P(X = x_k)$ for each possible value $x_k$ with $1 \leq k \leq n$. For random variables, the conditional probability is defined as
\begin{equation}
	P(X=x|Y=y) = \frac{P(X=x \cap Y=y)}{P(Y=y)}
\end{equation} 
where $X$ and $Y$ are discrete random variables, and $x$ and $y$ are possible values that $X$ and $Y$ can take, respectively. $P(X=x \cap Y=y)$ is the probability that $X=x$ and $Y=y$ occur simultaneously, and $P(Y=y)$ is the probability that $Y=y$ occurs.

The law of total probability for discrete random variables can be written as follows. Let $X$ be a discrete random variable and let $Y_1, Y_2, \dots, Y_m$ be a partition of the sample space $\Omega$, i.e., $\Omega = Y_1 \cup Y_2 \cup \dots \cup Y_m$ and $Y_i \cap Y_j = \emptyset$ for $1\leq i\leq m$, $1\leq j \leq m$, and $i \neq j$. Then, for any event $A$, the law of total probability states:
\begin{equation} 
 P(A) = \sum_{i=1}^{m} P(Y_i) P(A|Y_i)
\end{equation} 
where the sum is taken over all possible values of $i$.
 
\section{Maximum Likelihood Estimation}
\label{appendix:stats}
Maximum Likelihood Estimation (MLE) is a popular statistical method for estimating the parameters of a statistical model. MLE is widely used in many fields, including econometrics, biostatistics, and engineering, due to its desirable properties (see~\cite[Chapter 7]{casella2021statistical}). One of the key strengths of MLE is its asymptotic efficiency, which means that as the sample size increases, the MLE estimates converge to the true parameter values at the fastest possible rate among all consistent estimators. This means that MLE produces the most precise estimates possible, given the available data.

Furthermore, MLE is a consistent estimator, meaning that as the sample size increases, the MLE estimates converge to the true parameter values. As a result, the accuracy of the MLE estimates increases with more data becoming available. Under certain conditions, MLE is also an unbiased estimator, meaning that the expected value of the estimates is equal to the true parameter value. This requires that the model is correctly specified and the sample size is sufficiently large.

Finally, MLE has solid theoretical foundations based on sound statistical theory and has been extensively studied in the literature, providing a large body of knowledge for understanding its properties and effective use. Additionally, the computation of MLE is widely supported by several existing software packages encoded in R and Python, among others. 
 
\section{Linear Temporal Logic on Finite Traces}
\label{appendix:ltlf}
\acrfull{ltlf}~\cite{DeGiacomo2013Linear} expresses propositions over linear discrete-time structures of finite length -- namely, traces.
{\ltlf} has the same syntax of \gls{ltl}~\cite{Pnueli/FOCS1977:LTL}, but is interpreted on finite traces.
In this paper, in particular, we consider the \gls{ltl} dialect including past modalities~\cite{DBLP:conf/lop/LichtensteinPZ85}.

\Gls{ltlf} formulae are built from an alphabet $\LogAlph \supseteq \{ \lettera \}$ of propositional symbols, auxiliary symbols `$($' and `$)$', propositional constants $\RNFtrue$ and $\RNFfalse$, the logical connectives $\lnot$ (\emph{negation}, unary) and $\land$ (\emph{conjunction}, binary),
the unary temporal operators {\RNFnext} (\emph{next}) and {\RNFprev} (\emph{yesterday}), and the binary temporal operators \RNFuntil (\emph{until}) and \RNFsince (\emph{since}).
Typically, their syntax is enriched with additional binary logical connectives, namely
$ \lor $ (\emph{disjunction}) and $ \to $ (\emph{implication}),
and unary temporal operators 
$\RNFfut$ (\emph{eventually}), $\RNFpast$  (\emph{once}), $\RNFglobFut$ (\emph{always}), and $\RNFglobPast$ (\emph{historically}).
Finally, we introduce two additional temporal constants,
$\RNFstart$ and $\RNFend$, intuitively denoting the initial and final event of a trace.

Although they do not add expressive power, they contribute to more succinct formulations.
The following grammar provides the syntax rules to build a well-formed {\ltlf} formula:
\begin{flalign*}
	\LTLfpFormula \Coloneqq \: &
	\RNFtrue \mid \RNFfalse \mid \RNFstart \mid \RNFend \mid \lettera \mid \\ &
	(\lnot\LTLfpFormula) \mid (\LTLfpFormula_1 \land \LTLfpFormula_2)  \mid (\LTLfpFormula_1 \lor \LTLfpFormula_2)  \mid (\LTLfpFormula_1 \to \LTLfpFormula_2) \mid \\ &
	(\RNFnext \LTLfpFormula) \mid (\LTLfpFormula_1 \RNFuntil \LTLfpFormula_2) \mid (\RNFfut \LTLfpFormula) \mid (\RNFglobFut \LTLfpFormula) \mid \\ &
	(\RNFprev \LTLfpFormula) \mid (\LTLfpFormula_1 \RNFsince \LTLfpFormula_2) \mid (\RNFpast \LTLfpFormula) \mid (\RNFglobPast \LTLfpFormula) \\
\end{flalign*}
We may omit parentheses when the operator precedence intuitively follows from the expression.
Given $\{\taska, \taskb\} \subseteq \LogAlph$, e.g., the following is a well-formed \gls{ltlf} formula:
${(\RNFnext\lnot\taska)\RNFuntil\taskb}$. 

Semantics of \gls{ltlf} is given in terms of finite \glspl{evttrace}, i.e., finite words over the alphabet $2^\LogAlph$.
We name the index of the element in the trace as \emph{instant}.
Intuitively, 
$\RNFtrue$ and $\RNFfalse$ denote truth and falsity,
$\lnot\LTLfpFormula$ means that $\LTLfpFormula$ does not hold true,
$\LTLfpFormula_1 \land \LTLfpFormula_2$ signifies that both $\LTLfpFormula_1$ and $\LTLfpFormula_2$ hold true,
$\LTLfpFormula_1 \lor \LTLfpFormula_2$ indicates that $\LTLfpFormula_1$ or $\LTLfpFormula_2$ (or both) hold true, and
$\LTLfpFormula_1 \to \LTLfpFormula_2$ states that if $\LTLfpFormula_1$ holds true then $\LTLfpFormula_2$ must be verified (whereas if $\LTLfpFormula_1$ does \emph{not} hold true, no condition is exerted on $\LTLfpFormula_2$).
Formulae consisting of propositional symbols, constants and logical connectives are verified in a specific point in time (an \textit{instant}).
Temporal operators and constants require an evaluation of the formula on a trace of subsequent instants.
$\RNFnext \LTLfpFormula$ and $\RNFprev \LTLfpFormula$ indicate that $\LTLfpFormula$ holds true at the next and previous instant, respectively;
$\LTLfpFormula_1 \RNFuntil \LTLfpFormula_2$ states that $\LTLfpFormula_2$ will eventually hold and, until then, $\LTLfpFormula_1$ holds too;
dually, $\LTLfpFormula_1 \RNFsince \LTLfpFormula_2$ signifies that $\LTLfpFormula_2$ holds at some point and, from that instant on, $\LTLfpFormula_1$ holds too.
$\RNFfut \LTLfpFormula$ and $\RNFpast \LTLfpFormula$ mean that $\LTLfpFormula$ holds true eventually in the future, or at some instant in the past.
Finally, $\RNFglobFut \LTLfpFormula$ and $\RNFglobPast \LTLfpFormula$ express the truthness of $\LTLfpFormula$ in every instant from the current one on, and in every instant from the current one back in the trace, respectively.
We formalize the above as follows.

\medskip
Given a finite \gls{evttrace} {\EvtTrace} of length $n \in \mathbb{N}$, we write $ (\EvtTrace, \instant) \models \LTLfpFormula $ to denote that an \gls{ltlf} formula {\LTLfpFormula} is satisfied at a given instant $\instant \in \mathbb{N}$, with $ 1\leq\instant\leq\EvtTraceLength $, by induction of the following:

\begin{itemize}
	\item[$ (\EvtTrace, \instant) $] $  \models \RNFtrue $; 
$ (\EvtTrace, \instant) \nvDash \RNFfalse $; \item[$ (\EvtTrace, \instant) $] $   \models \lettera $ iff $ \lettera $ is $ \RNFtrue $ in $ \EvtTrace$ at instant $\instant$; 
	\item[$ (\EvtTrace, \instant) $] $  \models \lnot\LTLfpFormula $ iff $ (\EvtTrace, \instant) \nvDash \LTLfpFormula $; \item[$ (\EvtTrace, \instant) $] $   \models \LTLfpFormula_1\land\LTLfpFormula_2 $ iff $ (\EvtTrace, \instant) \models \LTLfpFormula_1 $ and $ (\EvtTrace, \instant) \models \LTLfpFormula_2 $; \item[$ (\EvtTrace, \instant) $] $   \models \LTLfpFormula_1\lor\LTLfpFormula_2 $ iff $ (\EvtTrace, \instant) \models \LTLfpFormula_1 $ or $ (\EvtTrace, \instant) \models \LTLfpFormula_2 $; \item[$ (\EvtTrace, \instant) $] $   \models \LTLfpFormula_1\to\LTLfpFormula_2 $ iff $ (\EvtTrace, \instant) \nvDash \LTLfpFormula_1 $ or $ (\EvtTrace, \instant) \models \LTLfpFormula_2 $; \item[$ (\EvtTrace, \instant) $] $  \models \RNFnext\LTLfpFormula $ iff $ i<\EvtTraceLength $ and $ (\EvtTrace, \instant+1) \models \LTLfpFormula $; \item[$ (\EvtTrace, \instant) $] $  \models \RNFprev\LTLfpFormula $ iff $ \instant>1 $ and $ (\EvtTrace, \instant-1) \models \LTLfpFormula $; \item[$ (\EvtTrace, \instant) $] $  \models \LTLfpFormula_1\RNFuntil\LTLfpFormula_2 $ iff there exists a $j \in \mathbb{N}$ with $ \instant\leq j\leq\EvtTraceLength $ s.t.\ $ (\EvtTrace, j) \models\LTLfpFormula_2 $ and $ (\EvtTrace, k) \models \LTLfpFormula_1 $ for every $k \in \mathbb{N}$ s.t.\ $ {\instant\leq k<j} $; \item[$ (\EvtTrace, \instant) $] $  \models \LTLfpFormula_1\RNFsince\LTLfpFormula_2 $ iff there exists a $j \in \mathbb{N}$ with $ 1\leq j\leq \instant $ s.t.\ $ (\EvtTrace, j) \models\LTLfpFormula_2 $ and $ (\EvtTrace, k) \models \LTLfpFormula_1 $ for every $k \in \mathbb{N}$ s.t.\ $ {j<k\leq \instant} $. \item[$ (\EvtTrace, \instant) $] $  \models \RNFfut\LTLfpFormula $ iff there exists a $j \in \mathbb{N}$ with $ \instant\leq j\leq\EvtTraceLength $ s.t.\ $ (\EvtTrace, j) \models\LTLfpFormula $; \item[$ (\EvtTrace, \instant) $] $  \models \RNFpast\LTLfpFormula $ iff there exists a $j \in \mathbb{N}$ with $ 1\leq j\leq \instant $ s.t.\ $ (\EvtTrace, j) \models\LTLfpFormula $; \item[$ (\EvtTrace, \instant) $] $  \models \RNFglobFut\LTLfpFormula $ iff $ (\EvtTrace, k) \models\LTLfpFormula $ for every $k \in \mathbb{N}$ s.t.\ $ {\instant \leq k \leq n} $; \item[$ (\EvtTrace, \instant) $] $  \models \RNFglobPast\LTLfpFormula $ iff $ (\EvtTrace, k) \models\LTLfpFormula $ for every $k \in \mathbb{N}$ s.t.\ $ {1 \leq k \leq \instant} $; \item[$ (\EvtTrace, \instant) $] $  \models \RNFstart $ iff $ (\EvtTrace, 1) \models\LTLfpFormula $; \item[$ (\EvtTrace, \instant) $] $  \models \RNFend $ iff $ (\EvtTrace, n) \models\LTLfpFormula $. \end{itemize}

\noindent
For example, $ (\EvtTrace, \instant) \models \taska\wedge\RNFfut\taskb$ \, (i.e., $\taska\wedge\RNFfut\taskb$ is satisfied in a trace $\EvtTrace$ of length $n$ at instant $\instant$) when the propositional atom $\taska$ holds true in $\EvtTrace$ at $\instant \leq n$ and $\taskb$ holds true at a later instant $j$ with $\instant \leq j \leq n$ in the same trace $\EvtTrace$.

\smallskip
\noindent
From the above operators, we observe the following logical equivalences:
\begin{itemize}
	\item $ \LTLfpFormula_1 \lor \LTLfpFormula_2 \equiv \lnot(\lnot \LTLfpFormula_1 \land \lnot \LTLfpFormula_2) $;
	\item $ \LTLfpFormula_1 \to \LTLfpFormula_2 \equiv \lnot \LTLfpFormula_1 \lor \LTLfpFormula_2 $;
	\item $\RNFend \equiv \lnot(\RNFnext\RNFtrue) $;
	\item $\RNFstart \equiv \lnot(\RNFprev\RNFtrue) $;
	\item $ \RNFfut\LTLfpFormula \equiv \RNFtrue\RNFuntil\LTLfpFormula $;
\item $ \RNFpast\LTLfpFormula \equiv \RNFtrue\RNFsince\LTLfpFormula $;
	\item $ \RNFglobFut\LTLfpFormula \equiv \lnot\RNFfut\lnot\LTLfpFormula $; \item $ \RNFglobPast\LTLfpFormula \equiv \lnot\RNFpast\lnot\LTLfpFormula $. \end{itemize}

\end{document}